\documentclass{article}

\usepackage{hyperref}

\usepackage{savesym,multirow,url,xspace,graphicx,hyperref,enumitem,color}
\usepackage[titletoc]{appendix}
\usepackage{subcaption}
\usepackage{dsfont}
\usepackage{multicol}

\usepackage{etoolbox}

\usepackage{tikz}
\usetikzlibrary{automata, positioning, arrows}
\tikzset{
	->, 
	every state/.style={thick, fill=gray!10}, 
	initial text=$ $, 
}
\usetikzlibrary{arrows.meta}
\usepackage{subcaption}
\usepackage{pgfplots}
\pgfplotsset{width=10cm,compat=1.9}
\usepackage{diagbox}
\usepackage{caption}

\makeatletter
\newif\if@restonecol
\makeatother

\usepackage{amsmath,amssymb,xfrac,amsthm}
\setitemize{noitemsep,topsep=1pt,parsep=1pt,partopsep=1pt, leftmargin=12pt}
\newtheorem{lemma}{Lemma}
\newtheorem{theorem}{Theorem}
\newtheorem*{theorem*}{Theorem}

\newtheorem{corollary}{Corollary}
\newtheorem{proposition}{Proposition}

\newtheorem{remark}{Remark}
\makeatletter

\newcommand{\Rmnum}[1]{\expandafter\@slowromancap\romannumeral #1@}
\makeatother
\usepackage{caption}
\usepackage{wrapfig}

\usepackage{enumitem}
\usepackage{wasysym}

\DeclareMathOperator*{\argmax}{arg\,max}

\newcommand{\R}{\mathbb{R}}

\usepackage{bm}
\usepackage{bbm}

\newcommand{\norm}[1]{\left\lVert#1\right\rVert}

\newcommand{\vecdot}[2]{\left<#1, #2\right>}

\newcommand{\expct}[1]{\mathbb{E}\left[#1\right]}

\newcommand{\ind}[1]{\mathds{1}\left[#1\right]}

\newcommand{\attack}{{\mathcal A}}
\newcommand{\defense}{{\mathcal D}}

\newcommand{\epsA}{{\epsilon_{\dagger}}}
\newcommand{\epsD}{{\epsilon_{\mathcal D}}}

\newcommand{\defensepi}{{\pi_{\mathcal D}}}
\newcommand{\influence}{{\Delta}}

\newcommand{\tightset}{{\Theta^{\epsilon}}}
\newcommand{\tightsetprime}{{\Theta^{\epsilon'}}}

\newcommand{\tightsettarget}{{\Theta^{\epsilon_{\dagger}}}}
\newcommand{\tightsetstate}{{\Theta^{\epsilon}_s}}
\newcommand{\tightsetstatetarget}{{\Theta^{\epsA}_s}}

\newcommand{\targetpi}{\ensuremath{{\pi_{\dagger}}}}
\newcommand{\neighbor}[3]{#1\{#2;#3\}}

\newcommand{\score}{{\rho}}
\newcommand{\hatscore}{\widehat{\score}}
\newcommand{\barscore}{\overline{\score}}

\newcommand{\occupancy}{{\psi}}
\newcommand{\Occupancy}{{\Psi}}

\newcommand{\occstate}{{\mu}}
\newcommand{\occdiffmatrix}{{\mathbf \Phi}}

\newcommand{\optpi}{\pi^{*}}
\newcommand{\alignment}{{\Gamma^{\{s;a\}}}}

\usepackage[preprint,nonatbib]{neurips_2021}

\usepackage[utf8]{inputenc} 
\usepackage[T1]{fontenc}    
\usepackage{hyperref}       
\usepackage{url}            
\usepackage{booktabs}       
\usepackage{amsfonts}       
\usepackage{nicefrac}       
\usepackage{microtype}      

\title{Defense Against Reward Poisoning Attacks in Reinforcement Learning}

\author{%
	Kiarash Banihashem\\
	MPI-SWS\\
	\texttt{kbanihas@mpi-sws.org} \\
	\And
	Adish Singla\\
	MPI-SWS\\
	\texttt{adishs@mpi-sws.org} \\
	\And
	Goran Radanovic\\
	MPI-SWS\\
	\texttt{gradanovic@mpi-sws.org} \\
	\\
}

\begin{document}
	
	\maketitle
	
	\newtoggle{longversion}
	\settoggle{longversion}{true}
\begin{abstract}
We study defense strategies against reward poisoning attacks in reinforcement learning. As a threat model, we consider attacks that minimally alter rewards to make the attacker's target policy uniquely optimal under the poisoned rewards, with the optimality gap specified by an attack parameter. Our goal is to design agents that are robust against such attacks in terms of  the worst-case utility w.r.t. the true, unpoisoned, rewards while computing their policies under the poisoned rewards. We propose an optimization framework for deriving optimal defense policies, both when the attack parameter is known and unknown. Moreover, we show that defense policies that are solutions to the proposed optimization problems have provable performance guarantees. In particular, we provide the following bounds with respect to the true, unpoisoned, rewards: a) lower bounds on the expected return of the defense policies, and b) upper bounds on how suboptimal these defense policies are compared to the attacker's target policy. Using simulation-based experiments, we demonstrate the effectiveness of our defense approach.
\end{abstract}

\section{Introduction}\label{sec.introduction}

One of the key challenges in designing trustworthy AI systems is ensuring that they are technically robust and resilient to security threats \cite{ethicsEU}. Amongst many 
requirements that are important to satisfy in order for an AI system to be deemed trustworthy is robustness to adversarial attacks \cite{hamon2020robustness}. 
~\\
\looseness-1Standard approaches to reinforcement learning (RL) \cite{sutton2018reinforcement} have shown to be susceptible to adversarial attacks which manipulate the feedback that an agent receives from its environment, i.e., its input data. These attacks broadly fall under two categories: 
 a) {\em test-time} attacks, which manipulate an agent's input data at test-time without changing the agent's policy \cite{huang2017adversarial,DBLP:conf/ijcai/LinHLSLS17,tretschk2018sequential}, and b) {\em training-time} attacks that manipulate an agent's input data at training-time, 
 thereby influencing the agent's learned policy
 \cite{DBLP:conf/aaai/ZhangP08,ma2019policy,DBLP:conf/gamesec/HuangZ19a,rakhsha2020policy,rakhsha2020policy-jmlr,xuezhou2020adaptive,sun2020vulnerability}. 
 In this paper, we focus on training-time attacks which specifically modify rewards (aka {\em reward poisoning}) to force the agent into  adopting a  {\em target} policy \cite{ma2019policy,rakhsha2020policy-jmlr}.
 ~\\
Prior work on reward poisoning attacks on RL primarily focuses on designing optimal attacks. 
In this paper, we take a different perspective on targeted reward poisoning attacks, and focus on designing {\em defense strategies} that are effective against such attacks. This is challenging, given that the attacker is typically unconstrained in poisoning the rewards to force the target policy, while the agent's performance is measured under the true reward function, which is unknown. The key idea that we exploit in our work is that the poisoning attacks have an underlying \emph{structure} arising from the attacker's objective to minimize the cost of the attack needed to force the target policy. We therefore ask the following question: {\em Can we design an effective defense strategy against reward poisoning attacks by exploiting the underlying structure of these attacks?}
\begin{figure*}[ht]
	\centering
	\begin{subfigure}{.45\textwidth}
		\resizebox{0.9\linewidth}{!}{\begin{tikzpicture}

    	\node[state] (s0) {$s_0$};
	\node[state, right=0.6 of s0] (s1) {$s_1$};
	\node[state, right=0.6 of s1] (s2) {$s_2$};
	\node[state, right=0.6 of s2] (s3) {$s_3$};
        \draw[-{Triangle[width=8pt,length=3pt]}, line width=4pt, opacity=0.10](-0.10,-1.40) -- (-0.65, -1.40);
        \draw[-{Triangle[width=8pt,length=3pt]}, line width=4pt, opacity=0.85](0.00,-1.40) -- (0.55, -1.40);
        \node[text width=0.2 cm] at (-0.30, -1.10) (t0) {\scalebox{0.8}{$0$}};
        \node[text width=0.2 cm] at (0.30, -1.10) (t1) {\scalebox{0.8}{$1$}};
        
\draw[-{Triangle[width=8pt,length=3pt]}, line width=4pt, opacity=0.10](1.45,-1.40) -- (0.90, -1.40);
        \draw[-{Triangle[width=8pt,length=3pt]}, line width=4pt, opacity=0.85](1.55,-1.40) -- (2.10, -1.40);
        \node[text width=0.2 cm] at (1.25, -1.10) (t0) {\scalebox{0.8}{$0$}};
        \node[text width=0.2 cm] at (1.85, -1.10) (t1) {\scalebox{0.8}{$1$}};
        
\draw[-{Triangle[width=8pt,length=3pt]}, line width=4pt, opacity=0.85](3.00,-1.40) -- (2.45, -1.40);
        \draw[-{Triangle[width=8pt,length=3pt]}, line width=4pt, opacity=0.10](3.10,-1.40) -- (3.65, -1.40);
        \node[text width=0.2 cm] at (2.80, -1.10) (t0) {\scalebox{0.8}{$1$}};
        \node[text width=0.2 cm] at (3.40, -1.10) (t1) {\scalebox{0.8}{$0$}};
        
\draw[-{Triangle[width=8pt,length=3pt]}, line width=4pt, opacity=0.85](4.55,-1.40) -- (4.00, -1.40);
        \draw[-{Triangle[width=8pt,length=3pt]}, line width=4pt, opacity=0.10](4.65,-1.40) -- (5.20, -1.40);
        \node[text width=0.2 cm] at (4.35, -1.10) (t0) {\scalebox{0.8}{$1$}};
        \node[text width=0.2 cm] at (4.95, -1.10) (t1) {\scalebox{0.8}{$0$}};
        
\draw (s0) edge[loop left, below, pos=0.2] node{\scalebox{0.8}{-2.50}} (s0);
\draw (s0) edge[bend left, above] node{\scalebox{0.8}{-2.50}} (s1);
\draw (s1) edge[bend left, below] node{\scalebox{0.8}{0.50}} (s0);
\draw (s1) edge[bend left, above] node{\scalebox{0.8}{0.50}} (s2);
\draw (s2) edge[bend left, below] node{\scalebox{0.8}{0.50}} (s1);
\draw (s2) edge[bend left, above] node{\scalebox{0.8}{0.50}} (s3);
\draw (s3) edge[bend left, below] node{\scalebox{0.8}{-0.50}} (s2);
\draw (s3) edge[loop right, above, pos=0.2] node{\scalebox{0.8}{-0.50}} (s3);
\end{tikzpicture}}
		\caption{$\overline{R}, \optpi$}\label{fig.example.a}
	\end{subfigure}
	\quad
	\begin{subfigure}{0.45\textwidth}
		\resizebox{0.9\linewidth}{!}{\begin{tikzpicture}

    	\node[state] (s0) {$s_0$};
	\node[state, right=0.6 of s0] (s1) {$s_1$};
	\node[state, right=0.6 of s1] (s2) {$s_2$};
	\node[state, right=0.6 of s2] (s3) {$s_3$};
        \draw[-{Triangle[width=8pt,length=3pt]}, line width=4pt, opacity=0.10](-0.10,-1.40) -- (-0.65, -1.40);
        \draw[-{Triangle[width=8pt,length=3pt]}, line width=4pt, opacity=0.85](0.00,-1.40) -- (0.55, -1.40);
        \node[text width=0.2 cm] at (-0.30, -1.10) (t0) {\scalebox{0.8}{$0$}};
        \node[text width=0.2 cm] at (0.30, -1.10) (t1) {\scalebox{0.8}{$1$}};
        
\draw[-{Triangle[width=8pt,length=3pt]}, line width=4pt, opacity=0.10](1.45,-1.40) -- (0.90, -1.40);
        \draw[-{Triangle[width=8pt,length=3pt]}, line width=4pt, opacity=0.85](1.55,-1.40) -- (2.10, -1.40);
        \node[text width=0.2 cm] at (1.25, -1.10) (t0) {\scalebox{0.8}{$0$}};
        \node[text width=0.2 cm] at (1.85, -1.10) (t1) {\scalebox{0.8}{$1$}};
        
\draw[-{Triangle[width=8pt,length=3pt]}, line width=4pt, opacity=0.10](3.00,-1.40) -- (2.45, -1.40);
        \draw[-{Triangle[width=8pt,length=3pt]}, line width=4pt, opacity=0.85](3.10,-1.40) -- (3.65, -1.40);
        \node[text width=0.2 cm] at (2.80, -1.10) (t0) {\scalebox{0.8}{$0$}};
        \node[text width=0.2 cm] at (3.40, -1.10) (t1) {\scalebox{0.8}{$1$}};
        
\draw[-{Triangle[width=8pt,length=3pt]}, line width=4pt, opacity=0.10](4.55,-1.40) -- (4.00, -1.40);
        \draw[-{Triangle[width=8pt,length=3pt]}, line width=4pt, opacity=0.85](4.65,-1.40) -- (5.20, -1.40);
        \node[text width=0.2 cm] at (4.35, -1.10) (t0) {\scalebox{0.8}{$0$}};
        \node[text width=0.2 cm] at (4.95, -1.10) (t1) {\scalebox{0.8}{$1$}};
        
\draw (s0) edge[loop left, below, pos=0.2] node{\scalebox{0.8}{-2.50}} (s0);
\draw (s0) edge[bend left, above] node{\scalebox{0.8}{-2.50}} (s1);
\draw (s1) edge[bend left, below] node{\scalebox{0.8}{0.50}} (s0);
\draw (s1) edge[bend left, above] node{\scalebox{0.8}{0.13}} (s2);
\draw (s2) edge[bend left, below] node{\scalebox{0.8}{0.07}} (s1);
\draw (s2) edge[bend left, above] node{\scalebox{0.8}{0.61}} (s3);
\draw (s3) edge[bend left, below] node{\scalebox{0.8}{-0.50}} (s2);
\draw (s3) edge[loop right, above, pos=0.2] node{\scalebox{0.8}{0.19}} (s3);
\end{tikzpicture}}
		\caption{$\widehat{R}, \targetpi$}\label{fig.example.b}
	\end{subfigure}
	\quad
	\begin{subfigure}{.45\textwidth}
		\resizebox{0.9\linewidth}{!}{\begin{tikzpicture}

    	\node[state] (s0) {$s_0$};
	\node[state, right=0.6 of s0] (s1) {$s_1$};
	\node[state, right=0.6 of s1] (s2) {$s_2$};
	\node[state, right=0.6 of s2] (s3) {$s_3$};
        \draw[-{Triangle[width=8pt,length=3pt]}, line width=4pt, opacity=0.10](-0.10,-1.40) -- (-0.65, -1.40);
        \draw[-{Triangle[width=8pt,length=3pt]}, line width=4pt, opacity=0.85](0.00,-1.40) -- (0.55, -1.40);
        \node[text width=0.2 cm] at (-0.30, -1.10) (t0) {\scalebox{0.8}{$0$}};
        \node[text width=0.2 cm] at (0.30, -1.10) (t1) {\scalebox{0.8}{$1$}};
        
\draw[-{Triangle[width=8pt,length=3pt]}, line width=4pt, opacity=0.10](1.45,-1.40) -- (0.90, -1.40);
        \draw[-{Triangle[width=8pt,length=3pt]}, line width=4pt, opacity=0.85](1.55,-1.40) -- (2.10, -1.40);
        \node[text width=0.2 cm] at (1.25, -1.10) (t0) {\scalebox{0.8}{$0$}};
        \node[text width=0.2 cm] at (1.85, -1.10) (t1) {\scalebox{0.8}{$1$}};
        
\draw[-{Triangle[width=8pt,length=3pt]}, line width=4pt, opacity=0.80](3.00,-1.40) -- (2.45, -1.40);
        \draw[-{Triangle[width=8pt,length=3pt]}, line width=4pt, opacity=0.15](3.10,-1.40) -- (3.65, -1.40);
        \node[text width=0.2 cm] at (2.60, -1.10) (t0) {\scalebox{0.8}{$0.94$}};
        \node[text width=0.2 cm] at (3.20, -1.10) (t1) {\scalebox{0.8}{$0.06$}};
        
\draw[-{Triangle[width=8pt,length=3pt]}, line width=4pt, opacity=0.10](4.55,-1.40) -- (4.00, -1.40);
        \draw[-{Triangle[width=8pt,length=3pt]}, line width=4pt, opacity=0.85](4.65,-1.40) -- (5.20, -1.40);
        \node[text width=0.2 cm] at (4.35, -1.10) (t0) {\scalebox{0.8}{$0$}};
        \node[text width=0.2 cm] at (4.95, -1.10) (t1) {\scalebox{0.8}{$1$}};
        
\draw (s0) edge[loop left, below, pos=0.2] node{\scalebox{0.8}{-2.50}} (s0);
\draw (s0) edge[bend left, above] node{\scalebox{0.8}{-2.50}} (s1);
\draw (s1) edge[bend left, below] node{\scalebox{0.8}{0.50}} (s0);
\draw (s1) edge[bend left, above] node{\scalebox{0.8}{0.13}} (s2);
\draw (s2) edge[bend left, below] node{\scalebox{0.8}{0.07}} (s1);
\draw (s2) edge[bend left, above] node{\scalebox{0.8}{0.61}} (s3);
\draw (s3) edge[bend left, below] node{\scalebox{0.8}{-0.50}} (s2);
\draw (s3) edge[loop right, above, pos=0.2] node{\scalebox{0.8}{0.19}} (s3);
\end{tikzpicture}}
		\caption{$\widehat{R}, \defensepi$}\label{fig.example.c}
	\end{subfigure}
	\quad
	\begin{minipage}[][][b]{0.45\textwidth}
	\resizebox{0.9\linewidth}{!}{
    	\begin{subfigure}{1\textwidth}
    		\centering
    		{\renewcommand{\arraystretch}{1.5}
    		\begin{tabular}{c*{3}{|c}}
    			\diagbox{MDP}{Policy} & $\optpi$ & $\targetpi$& $\defensepi$\\
    			\hline
    			\input{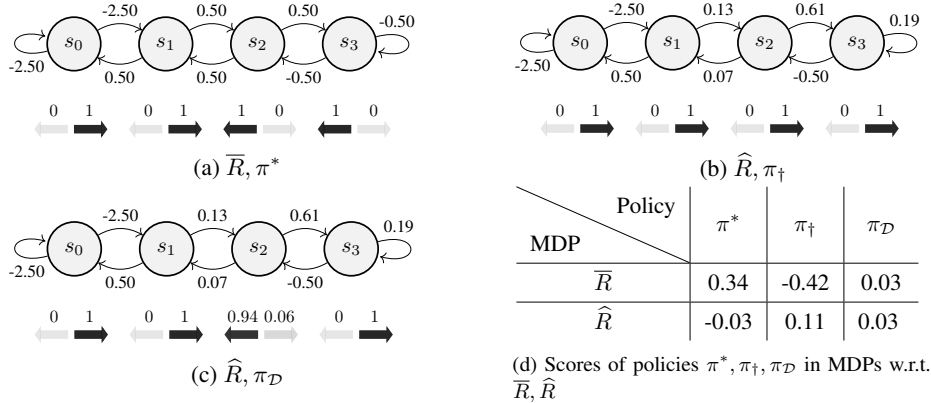}
    		\end{tabular}	
    	}
    	
    	\caption{Scores of policies $\optpi, \targetpi, \defensepi$ in MDPs w.r.t. $\overline{R}, \widehat{R}$}\label{fig.example.d}
    	\end{subfigure}
	}
    \end{minipage}
	\caption{A simple chain environment with 4 states and two possible actions: {\em left} and {\em right}. $s_0$ is the initial state. The agent goes in the direction of its action with
	probability $90\%$, 
	and otherwise the next state is selected uniformly at random from the other 3 states. Weights on edges indicate rewards for the action taken. For example in Fig. \ref{fig.example.a},
	if the agent takes {\em left} in state $s_1$, it receives $0.5$.
	We denote the true rewards by $\overline{R}$, the poisoned rewards by $\widehat{R}$, the optimal policy under $\overline{R}$ by $\optpi$, the target policy 
	(which is uniquely optimal under $\widehat{R}$) 
	by $\targetpi$, and the defense policy (which is derived from our framework) by $\defensepi$. %
 	\textbf{(a)} shows $\overline{R}$ and $\optpi$. 
 		In particular, the numbers above the arrows and the different shades of gray  show the probabilities of taking actions {\em left} and {\em right} under $\optpi$.
	 \textbf{(b)} shows $\widehat{R}$ and $\targetpi$.
\textbf{(c)} shows $\defensepi$ that our optimization framework derived from $\widehat{R}$, and by reasoning about the goal of the attack ($\targetpi$).
In particular, 
our optimization framework maximizes the worst-case performance under $\overline{R}$: while the optimization procedure does not know $\overline{R}$, it can constrain the set of plausible candidates for $\overline{R}$ using $\widehat{R}$.
	\textbf{(d)}  Table. \ref{fig.example.d}): Each entry in the table indicates the score of a (policy, reward function) pair, where the score is a scaled version of the total discounted return (see Section \ref{sec.setting}).
    For example, the score of policy $\targetpi$ equals $-0.42$ and $0.11$ under $\widehat{R}$ and $\overline{R}$ respectively.
    Our defense policy significantly improves upon this and achieves a score of $0.03$. For comparison, the score of $\optpi$
    equals $0.34$.
	Moreover, unlike for the target policy
	$\targetpi$, the score of our defense policy $\defensepi$
	under $\overline{R}$ is always at least as high as its score under $\widehat{R}$, as predicted by our results (see Theorem
	\ref{thm.defense_optimization.known_epsilon}). The results are obtained with parameters $\epsA = 0.1$,  $\epsD = 0.2$ and $\gamma=0.99$ (see Section \ref{sec.setting}).
	}\label{fig.example}
    \vspace{-5mm}
\end{figure*}
~\\
In this paper, we answer this question affirmatively. While an agent only has access to the poisoned rewards, it can still infer some information about the true reward function, using the fact that the attack exhibits some structure. By maximizing the worst-case utility over the set of plausible candidates for the true reward function, the agent can substantially limit the influence of the attack. The approach we take can be understood from Figure \ref{fig.example} which demonstrates our defense on the chain environment from \cite{rakhsha2020policy-jmlr}.

\textbf{Contributions.} 
We formalize this reasoning, and characterize the utility of our novel framework for designing defense policies. In summary, the key contributions include:
\begin{itemize}[leftmargin=*,labelindent=-2pt]
\vspace{-6pt}
\item \looseness-1We formalize the problem of finding defense policies that are effective against reward poisoning attacks that minimally modify the original reward function to achieve their goal (force a target policy).
\item \looseness-1We introduce a novel optimization framework for designing defense policies against reward poisoning attacks---this framework focuses on optimizing the agent's worst-case utility among the set of reward functions that are plausible candidates of the true reward function.
\item \looseness-1We provide characterization results that establish lower bounds on the performance of defense policies derived from our optimization framework, and upper bounds on the suboptimality of these defense policies compared to the target policy.
\item \looseness-1We empirically demonstrate the effectiveness of our approach using numerical simulations.
\end{itemize}
\vspace{-1mm}
To our knowledge, this is the first framework for studying this type of defenses against reward poisoning attacks that try to force a target policy at a minimal cost.

\vspace{2mm}
\section{Related Work} \label{sec.relatedwork}
While this paper is broadly related to the literature on adversarial machine learning (e.g., \cite{huang2011adversarial}), we recognize four themes in supervised learning (SL) and reinforcement learning (RL)  that closely connect to our work.  
~\\
\textbf{Poisoning attacks in SL and RL.} 
This paper is closely related to data poisoning attacks, first introduced and extensively studied in the context of supervised learning \cite{DBLP:conf/icml/BiggioNL12,xiao2012adversarial,mei2015using,DBLP:conf/icml/XiaoBBFER15,li2016data,koh2017understanding,biggio2018wild}. These attacks are also called {\em training-time attacks}, and unlike {\em test time attacks}~\cite{szegedy2014intriguing, pinto2017robust,behzadan2017whatever,zhang2020robust,moosavi2016deepfool, nguyen2015deep,madry2017towards}, 
which attack an already trained agent, 
they change data points during the training phase, which in turn affects the parameters of the learned model. More recently, data poisoning attacks have been studied in the bandits literature \cite{DBLP:conf/nips/Jun0MZ18,DBLP:conf/gamesec/MaJ0018,DBLP:conf/icml/LiuS19a}, and as we already mentioned, in RL.
~\\
\textbf{Defenses against poisoning attacks in SL.}
In supervised learning, defenses against data poisoning attacks are often based on data sanitization that
removes outliers from the training set
\cite{cretu2008casting,paudice2018detection}, trusted data points that support robust learning \cite{nelson2008exploiting,zhang2018training}, or robust estimation \cite{charikar2017learning,diakonikolas2019sever}. While such defenses can mitigate some attack strategies, they are in general susceptible to data poisoning attacks \cite{steinhardt2017certified,DBLP:journals/corr/abs-1811-00741}.
~\\
\textbf{Robustness to model uncertainty.}
There is a rich literature that studies robustness to uncertainty in MDP models, both in the context of uncertain reward functions \cite{mcmahan2003planning, regan2010robust}, and uncertain transition models \cite{nilim2005robust,iyengar2005robust, bagnell2001solving}.
Typically, these works consider settings in which instead of knowing the exact parameters of the MDP, the agent has access to a set of possible parameters (uncertainty set). These works design policies that perform well in the worst case.
More recent works have proposed ways to scale up these approaches via function approximation \cite{tamar2014scaling}, as well as utilize them in online settings \cite{lim2013reinforcement}. While our work uses the same principles of robust optimization, we  do not assume that the uncertainty set, i.e., the set of all possible rewards, is directly given. Instead, we show how to derive it from the poisoned reward function.
~\\
\textbf{Robustness to corrupted episodes.}
Another important line of work is the literature on robust learners that receive corrupted input during their training phase.
Such learners have recently been designed for bandits and experts settings \cite{lykouris2018stochastic,gupta2019better,bogunovic2020stochastic,amir2020prediction}, and episodic reinforcement learning \cite{lykouris2019corruption, zhang2021robustICML}.
Typically, these works consider an attack model in which the adversary can arbitrarily corrupt a limited number of episodes. As we  operate in the non-episodic setting and do not assume a limit in the attacker's poisoning budget, these works are orthogonal to the aspects we study in this paper. Instead, we utilize the {\em structure} of the attack in order to design a defense algorithm.
	
\section{Formal Setting}\label{sec.setting}
In this section, we describe our formal setting, and identify relevant background details on reward poisoning attacks, as well as our problem statement. 
The problem formulation specifies our objectives that we establish and formally analyze in the next sections.
\subsection{Preliminaries}\label{sec.setting.preliminaries} 
We consider a standard reinforcement learning setting in which the environment is described by a discrete-time discounted Markov Decision Processes (MDP) \cite{Puterman1994}, defined as $M = (S, A, R, P, \gamma, \sigma)$, where: $S$ is the state space, $A$ is the action space, $R:S \times A \rightarrow \mathds R$ is the reward function, 
$P: S \times A \times S \rightarrow [0, 1]$ is the transition model with $P(s, a,  s')$ defining the probability of transitioning to state $s'$ by taking action $a$ in state $s$, 
$\gamma \in [0, 1)$ is the discount factor, and $\sigma$ is the initial state distribution. We consider state and action spaces, i.e., $S$ and $A$, that are finite and discrete, and due to this we can adopt a vector notation for quantities dependent on states or state-action pairs. W.l.o.g., we assume that $|A| \ge 2$.
~\\
A generic (stochastic) policy is denoted by $\pi$, and it is a mapping $\pi:S \rightarrow \mathcal P(A)$, where $\mathcal P(A)$ is the probability simplex over action space $A$. We use $\pi(a|s)$ to denote the probability of taking action $a$ in state $s$. While deterministic policies are a special case of stochastic policies, when explicitly stating that a policy $\pi$ is deterministic, we assume that it is a mapping from states to actions, i.e., $\pi:S\rightarrow A$. We denote the set of all policies by $\Pi$ and the set of all deterministic policies by $\Pi^{\textnormal{det}}$.
For policy $\pi$, we define its {\em score}, $\score^{\pi}$, as
$\expct{ 
      (1-\gamma)\sum_{t=1}^{\infty} \gamma^{t-1} R(s_t, a_t) | \pi, \sigma}$,
where state $s_1$ is sampled from the initial state distribution $\sigma$, and then subsequent states $s_t$ are obtained by executing policy $\pi$ in the MDP. The score of a policy is therefore its total expected return scaled by a factor of $1-\gamma$. 
~\\
Finally, we consider occupancy measures. We denote 
the state-action occupancy measure in the Markov chain induced by policy $\pi$ 
by $\occupancy^{\pi}(s, a)=\expct{(1-\gamma) \sum_{t=1}^{\infty} \gamma^{t-1} \ind{s_t = s, a_t = a} | \pi, \sigma}$.
Given the MDP $M$, the set of realizable state-action occupancy measures under any (stochastic) policy $\pi \in \Pi$ is  denoted by $\Occupancy$. 
Score $\score^{\pi}$ and 
$\occupancy^{\pi}$ satisfy   $\score^{\pi} = \vecdot{\occupancy^{\pi}}{R}$, 
where $\vecdot{.}{.}$ computes the dot product between two vectors of sizes $|S| \cdot |A|$.
 We denote by $\occstate^{\pi}(s)=\expct{ (1-\gamma)\sum_{t=1}^{\infty} \gamma^{t-1}  \ind{s_t = s} | \pi, \sigma}$ the state occupancy measure in the Markov chain induced by policy $\pi \in \Pi$.
State-action occupancy measure $\occupancy^{\pi}(s,a)$ and state occupancy measure $\occstate^{\pi}(s)$  satisfy $\occupancy^{\pi}(s,a) = \occstate^{\pi}(s) \cdot \pi(a|s)$. We focus on {\em ergodic} MDPs, which in turn implies that $\occstate^\pi(s) > 0$ for all $\pi$ and $s$ \cite{Puterman1994}. This is a standard assumption in this line of work (e.g, see \cite{rakhsha2020policy-jmlr}) and is used to ensure the feasibility of the attacker's optimization problem.
\subsection{Reward Poisoning Attacks}\label{sec.setting.attack_model}
We consider reward poisoning attacks on an offline learning agent that optimally change the original reward function with the goal of deceiving the agent to adopt a deterministic policy $\targetpi \in \Pi^{\textnormal{det}}$, called {\em target policy}. %
This type of attack has been extensively studied in the literature, and here we utilize the attack formulation based on the works of \cite{ma2019policy,rakhsha2020policy,rakhsha2020policy-jmlr,xuezhou2020adaptive}. 
In the following, we introduce the necessary notation, the attacker's model, and the agent's model (without defense).
~\\
\textbf{Notation.} We use $\overline{M}$ to denote the {\em true} or {\em original} MDP with true, unpoisoned, reward function $\overline{R}$, i.e., $\overline{M} = (S, A, \overline R, P, \gamma, \sigma)$. We use $\widehat{M}$ to denote the {\em modified} or {\em poisoned} MDP with poisoned reward function $\widehat R$, i.e., $\widehat{M} = (S, A, \widehat R, P, \gamma, \sigma)$. Note that only the reward function $R$ changes across these MDPs. Quantities that depend on reward functions 
have analogous notation. For example, the score of policy $\pi$ under $\overline{R}$ is denoted by $\overline{\score}^{\pi}$, whereas its score under $\widehat{R}$ is denoted by $\widehat{\score}^{\pi}$. We denote an optimal policy under $\overline{R}$ by $\optpi$, i.e., $\optpi \in \argmax_{\pi \in \Pi} \overline \score^{\pi}$.
~\\
\textbf{Attack model.} The attacker we consider in this paper has full knowledge of $\overline{M}$. It can be modeled by a function $\attack(R', \targetpi, \epsA)$ that returns a poisoned reward function 
for a given reward function $R'$, target policy $\targetpi$, and a desired attack parameter $\epsA$.
In particular, the attacker solves the following optimization problem.
\begin{align}
	\label{prob.attack}
	\tag{P1}
	&\quad \min_{R} \norm{R - R'}_{2}
	\notag
	\quad\quad\quad \mbox{ s.t. } \quad \score^{\targetpi} \ge \score^{\pi} + \epsA \quad \forall \pi \in \Pi^{\text{det}} \backslash \{ \pi^{\dagger} \}. 	
	\\
	\intertext{
	As shown by \cite{rakhsha2020policy-jmlr}, this problem is feasible for ergodic MDPs and has a unique optimal solution. Furthermore, instead of considering all deterministic policies, it is sufficient to consider policies that differ from $\targetpi$ in a single action.
	 Using $\neighbor{\targetpi}{s}{a}$ to denote a policy that follows $a\ne\targetpi(s)$ in state $s$ and $\targetpi(\tilde s)$ in states $\tilde s\ne s$, \eqref{prob.attack} can be rewritten as follows.}
	\label{prob.attack.neighbor}
	\tag{P1'}
	&\quad \min_{R} \norm{R - R'}_{2}
	\notag
	\quad\quad\quad \mbox{ s.t. } \quad \score^{\targetpi} \ge \score^{\neighbor{\targetpi}{s}{a}} + \epsA \quad \forall s, a\ne \targetpi(s).
	\end{align}
By solving this problem, i.e., 
setting $\widehat R = \attack(\overline R, \targetpi, \epsA)$, 
the attacker 
finds the closest reward function to $\overline{R}$ (in Euclidean distance)
for which $\targetpi$ is a uniquely optimal policy (with attack parameter $\epsA$). 
~\\
\textbf{Agent without defense:} The agent receives the poisoned MDP $\widehat{M} := (S, A, \widehat{R}, P, \gamma, \sigma)$ where the underlying true reward function $\overline{R}$ (unknown to the agent) has been poisoned to $\widehat{R}$. 
In the existing works on reward poisoning attacks, an agent naively optimizes score $\widehat \rho$ (score w.r.t. $\widehat R$). Because of this, the agent ends up adopting policy $\targetpi$.
\subsection{Problem Statement}
Perhaps unsurprisingly, the agent without defense, could perform arbitrarily badly under the true reward function $\overline{R}$
(see Figure \ref{fig.example}). Our goal is to design a robust agent that has provable worst-case guarantees w.r.t. $\overline{R}$. This agent has access to the poisoned reward vector $\widehat{R}=\attack(\overline{R}, \targetpi, \epsA)$, but $\overline{R}$, $\targetpi$, and $\epsA$ are not given to the agent. Notice that $\targetpi$ is obtainable by solving the optimization problem $\argmax_{\pi} \widehat \score^{\pi}$ as $\targetpi$ is uniquely optimal in $\widehat{M}$. On the other hand, $\overline{R}$ is unknown to the agent. In terms of $\epsA$, we will focus on two cases, the case when $\epsA$ is known to the agent, and the case when it is not.
In the first case, we can formulate the following optimization problem of maximizing the worst case performance of the agent, given that $\overline{R}$ is unknown: 
\begin{align}
	\label{prob.defense_a}
	\tag{P2a}
	&\quad \max_{\pi} \min_{R} \score^{\pi}
	\notag
	\quad\quad\quad
	\mbox{ s.t. }  \quad  \widehat R = \attack(R, \targetpi, \epsA).\\
\intertext{We study this optimization problem in more detail in Section \ref{sec.defense_characterization_general_MDP.known_epsilon}. 
For the case when the agent does not know $\epsA$, 
we use the following optimization problem:}
	\label{prob.defense_b}
	\tag{P2b}
	&\quad \max_{\pi} \min_{R, \epsilon} \score^{\pi}
	\notag
	\quad\quad\quad
	\mbox{ s.t. }  \quad  \widehat R = \attack(R, \targetpi, \epsilon), \quad
 \quad 0 < \epsilon \le \epsD,
\end{align}
where the agent uses $\epsD$ as an upper bound on $\epsA$.
We study this optimization problem in more detail in Section \ref{sec.defense_characterization_general_MDP.unknown_epsilon}.
We denote solutions to the optimization problems \eqref{prob.defense_a} and $\eqref{prob.defense_b}$ by $\defensepi$, and it will be clear from the context which optimization problem we are referring to with $\defensepi$.
	
	
\section{Known Parameter Setting}\label{sec.defense_characterization_general_MDP.known_epsilon}

In this section, we provide characterization results for the case when the attack parameter $\epsA$ is known to the agent. The proofs of our theoretical results can be found in the Appendix.

\subsection{Optimal Defense Policy}\label{sec.defense_characterization_general_MDP.known_epsilon.optimal_solution}

We begin by analyzing the optimization problem \eqref{prob.defense_a}.
Denote  by $\tightset$ state-action pairs $(s, a)$ for which the difference between $\widehat \score^{\targetpi}$ and $\widehat \score^{\neighbor{\targetpi}{s}{a}}$ is equal to $\epsilon$, i.e., $ \tightset = \left \{ (s , a) : \widehat \score^{\neighbor{\targetpi}{s}{a}} - \widehat \score^{\targetpi} =  -\epsilon \right \}$
\footnote{
In practice,
$\tightset$ should be calculated with some tolerance due to numerical imprecision
(See Section \ref{sec.experiments}).
}.
For the results of this section, $\tightset$ with $\epsilon = \epsA$ plays a critical role.  As shown by the following lemma, it characterizes the feasible set of
\eqref{prob.defense_a}.
\begin{lemma}\label{lm.tightset}
Reward function $R$ satisfies $\widehat{R} = \attack(R, \targetpi, \epsA)$ if and only if there exists some $\alpha_{s, a} \ge 0$ such that
\begin{align*}
    R = \widehat{R} + \sum_{(s, a) \in \tightsettarget} \alpha_{s,a} \cdot \left (\occupancy^{\neighbor{\targetpi}{s}{a}} - \occupancy^{\targetpi} \right).
\end{align*}
\end{lemma}
 To see the importance of this result, let us
instantiate it with $\overline{R}$ and use it to calculate $\overline\score^{\pi} = \vecdot{\occupancy^{\pi}}{\overline R}$:
 \begin{align*}
     \overline\rho^{\pi} =
     \vecdot{\occupancy^{\pi}}{\widehat{R}} 
     + \sum_{(s, a) \in \tightsettarget} \alpha_{s,a} \cdot \alignment(\pi),
 \end{align*}
 where we introduced $\alignment(\pi) = \vecdot{\occupancy^{\neighbor{\targetpi}{s}{a}}  - \occupancy^{\targetpi}}{\occupancy^{\pi}}$.
 Given this equation, we can expect that aligning the occupancy measure of $\defensepi$ with directions $\occupancy^{\neighbor{\targetpi}{s}{a}} - \occupancy^{\targetpi}$ will yield some guarantees on the performance of $\defensepi$ under the original reward function $\overline{R}$. 
 This insight is formalized by the following theorem, which also describes a way to solve the optimization problem \eqref{prob.defense_a}.
 
\begin{theorem}\label{thm.defense_optimization.known_epsilon} 
Consider the following optimization problem parameterized by $\epsilon$:
    \begin{align}
    &\max_{\occupancy \in \Occupancy} \vecdot{\occupancy}{\widehat R}
    \label{prob.defense_optimization.known_epsilon}
	\tag{P3}
	\notag
    \quad\quad\quad
	\mbox{ s.t. } \vecdot{\occupancy^{\neighbor{\targetpi}{s}{a}} - \occupancy^{\targetpi}}{\occupancy} \ge 0 \quad \forall s, a \in \tightset.
\end{align}
For $\epsilon = \epsA$, this optimization problem
is always feasible, and its optimal solution $\occupancy_{\max}$ specifies an optimal solution to the optimization problem \eqref{prob.defense_a} with
\begin{align}\label{eq.defense_optimization.known_epsilon}
    \defensepi(a|s) = \frac{\occupancy_{\max}(s, a)}{\sum_{a'} \occupancy_{\max}(s, a')}.
\end{align}
The score of $\defensepi(a|s)$ is lower bounded by $\overline \score^{\defensepi} \ge \widehat \score^{\defensepi}$. Furthermore, $\alignment(\defensepi)$ is non-negative, i.e., $\alignment(\defensepi) \ge 0$ for all $(s,a) \in \tightsettarget$
\end{theorem}
As we discuss in the Appendix, the set of valid occupancy measures, $\Occupancy$, is a subset of $\mathds R^{|S|\cdot |A|}$ 
defined by a set of linear constraints. 
Therefore, since
occupancy measures $\occupancy^{\neighbor{\targetpi}{s}{a}}$ and $\occupancy^{\targetpi}$ can be precomputed,
the optimization problem
\eqref{prob.defense_optimization.known_epsilon} can be efficiently solved.
Theorem \ref{thm.defense_optimization.known_epsilon} also provides a performance guarantee of the defense policy w.r.t. the true reward function, i.e.,  $\overline \score^{\defensepi} \ge \widehat \score^{\defensepi}$. 
Such a bound is important in practice since it provides a certificate of the worst-case performance under the true reward function $\overline{R}$, even though the agent can only optimize over $\widehat{R}$.
\subsection{Attack Influence}\label{sec.defense_characterization_general_MDP.known_epsilon.attack_influence}

While informative, the guarantee of Theorem \ref{thm.defense_optimization.known_epsilon} does not tell us how well this solution fares compared to other policies, and in particular, the attacker's target policy $\targetpi$.
To provide a relative comparison, we turn to the measure of {\em attack influence} $\influence$, which for policy $\pi$, we define as $\influence^{\pi} = \overline \score^{\optpi} - \overline\score^{\pi}$.
Without any defense, the attack influence is equal to $\influence^{\targetpi}$, whereas the attack influence when we do have defense is $\influence^{\defensepi}$.
 In this section, we establish formal results that compare $\influence^{\defensepi}$ to $\influence^{\targetpi}$. 
 
 As we will see in our results, the following condition plays a critical role in comparing $\influence^{\defensepi}$ to $\influence^{\targetpi}$:
 \begin{align}\label{eq.attack_influence_condition}
     \alignment(\defensepi) \ge \alignment(\targetpi),\quad\quad \forall (s, a) \in \tightsettarget.
 \end{align}
 For settings where this condition holds, we derive upper bounds on $\influence^{\defensepi}$ in terms of $\influence^{\targetpi}$ (Theorem \ref{thm.attack_influence_a} and Theorem \ref{cor.attack_influence_a}).
 As for the settings in which this condition does not hold, we show that such bounds cannot be obtained and $\influence^{\defensepi} > \influence^{\targetpi}$ in the worst-case scenario. We start with Theorem \ref{thm.attack_influence_a}.
 
 \begin{theorem}\label{thm.attack_influence_a}
 Let $\defensepi$ be the defense policy obtained from the optimization problem \eqref{prob.defense_optimization.known_epsilon} and Equation \eqref{eq.defense_optimization.known_epsilon} with $\epsilon = \epsA$ and let $\widehat \influence = \widehat \rho^{\targetpi} - \widehat \rho^{\defensepi}$. Furthermore, let us assume that
 the condition in Equation \eqref{eq.attack_influence_condition} holds.
 Then the attack influence $\influence^{\defensepi}$ is bounded by
\begin{align}\label{eq.thm_attack_influence}
     \influence^{\defensepi} \le \max \{ \widehat \influence, \frac{\zeta}{1+\zeta} \cdot [\influence^{\targetpi} + \epsA]  + [\widehat \influence - \epsA] \},
\end{align}
where $\zeta = 0$ {\em if} $\tightsettarget = \emptyset$, and $\zeta = \max_{(s,a) \in \tightsettarget, \pi \in \Pi^{\text{det}}} \frac{\alignment(\pi)-\alignment(\defensepi)}{\alignment(\defensepi)-\alignment(\targetpi)}$ {\em if} $\tightsettarget \ne \emptyset$.
\end{theorem}

 Theorem \ref{thm.attack_influence_a} shows that $\influence^{\defensepi}$ can be lower than $\influence^{\targetpi}$ by factor $\frac{\zeta}{\zeta + 1}$ provided that $\defensepi$ and $\targetpi$ have similar scores under the poisoned reward (i.e., when $\widehat \influence$ is small). This factor, $\frac{\zeta}{\zeta + 1}$, is dependent on the occupancy measures induced by policy $\targetpi$ and its neighbour policies $\neighbor{\targetpi}{s}{a}$ through $\alignment$.
 
 Our next result, Theorem \ref{cor.attack_influence_a}, expresses $\frac{\zeta}{\zeta + 1}$ in terms of the quantity $\beta^{\occstate} = \max_{s,a} \norm{\occstate^{\targetpi} - \occstate^{\neighbor{\targetpi}{s}{a}}}_{\infty}$. 
 This quantity essentially captures how different actions affect transitions to next states. 
 \begin{theorem}\label{cor.attack_influence_a}
 Let $\occstate_{\min} = \min_{\pi, s},  \occstate^{\pi}(s)$, $\widehat \influence = \widehat \rho^{\targetpi} - \widehat \rho^{\defensepi}$, 
 and assume that $\beta^{\mu} \le \occstate_{\min}^2$. Then, the condition in Equation \eqref{eq.attack_influence_condition} holds and the attack influence $\influence^{\defensepi}$ is bounded by 
 \begin{align*}
     \influence^{\defensepi} \le \max \left \{\widehat \influence,
     \quad
     \frac{1 + 2 \cdot \frac{\beta^{\occstate}}{\occstate_{\min}^2}}{2+\frac{\beta^{\occstate}}{\occstate_{\min}^2}}\cdot [\influence^{\targetpi} +  \epsA] + [\widehat \influence - \epsA]\right \}.
 \end{align*}
 \end{theorem}
 The bounds in Theorem \ref{thm.attack_influence_a} and 
 Theorem
 \ref{cor.attack_influence_a} have dependency on $\widehat \influence = \widehat \rho^{\targetpi} - \widehat \rho^{\defensepi}$. This quantity is analogous to the notion of influence $\influence$ defined on $\overline{R}$---it measures how suboptimal $\defensepi$ is under $\widehat R$ in terms of score $\widehat \rho$. 
 Moreover, note that when $\beta^{\occstate} = 0$, the factor that multiplies influence $\influence^{\targetpi}$ is equal to $\frac{1}{2}$. We further discuss
 this special case
 and the tightness of this bound in
 the Appendix.

As mentioned earlier, the bounds in Theorem \ref{thm.attack_influence_a} and
 Theorem 
\ref{cor.attack_influence_a} require the condition in Equation \eqref{eq.attack_influence_condition} to hold. 
The next theorem shows that this condition is indeed necessary for establishing these bounds.
\begin{theorem}\label{prop.influence_impossibility}
Fix the poisoned reward function $\widehat{R}$, and assume that there exists state-action pair $(s, a) \in \tightsettarget$ such that $\alignment(\defensepi) < \alignment(\targetpi)$. Then, for any $\delta > 0$, there exists a reward function $\overline{R}$ such that $\widehat R = \attack(\overline R, \targetpi,\epsA)$ and $\influence^{\defensepi} \ge  \influence^{\targetpi} + \delta$.
\end{theorem}
We conclude this section by noting that the attack influence analysis (the bounds on $\influence^{\defensepi}$ in Theorems~\ref{thm.attack_influence_a} and \ref{cor.attack_influence_a}) is somewhat orthogonal to the worst-case score analysis (the bound on $\overline\score^{\defensepi}$ in Theorem \ref{thm.defense_optimization.known_epsilon}).
Importantly, $\targetpi$ can be much worse than 
$\defensepi$ in terms of the worst-case guarantees on score $\overline \score$.
In fact, while we certify that
$\overline{\score}^{\defensepi} \ge \widehat{\score}^{\defensepi}$, 
in general the worst-case value of
$\overline{\score}^{\targetpi}$ can be arbitrarily low.
	

\section{Unknown Parameter Setting}\label{sec.defense_characterization_general_MDP.unknown_epsilon}
In this subsection, we focus on the optimization problem \eqref{prob.defense_b}. 
First, note the structural difference between 
\eqref{prob.defense_a} and \eqref{prob.defense_b}. In the former case, 
$\epsA$ is given, and hence, the defense can infer possible values of $\overline R$ by solving an inverse problem to the attack problem $\eqref{prob.attack}$. In particular, we know that the original reward function $\overline{R}$ has to be in the set $\{R: \widehat R = \attack(R, \targetpi, \epsA) \}$. 
In the latter case, $\epsA$ is not known, and instead we use parameter $\epsD$ as an upper bound on $\epsA$. We distinguish two cases:
\begin{itemize}
    \item {\em Overestimating Attack Parameter}: If $\epsA \le \epsD$, then we know that $\overline{R}$ is in the set $\{R: \widehat R = \attack(R, \targetpi, \epsilon) \text{ s.t. } 0 < \epsilon \le \epsD \}$. Note that this set is a super-set of $\{R: \widehat R = \attack(R, \targetpi, \epsA) \}$, which means that it is less informative about $\overline{R}$. 
    \item {\em Underestimating Attack Parameter}: If $\epsA > \epsD$, then the set $\{R: \widehat R = \attack(R, \targetpi, \epsilon) \text{ s.t. } 0 < \epsilon \le \epsD\}$ will have only a single element, i.e., $\widehat R$. In other words, this set typically contains no information about $\overline R$.   
\end{itemize}
We analyze these two cases separately, first focusing on the former one. The proofs of our theoretical results can be found in the Appendix.


\subsection{Overestimating Attack Parameter }\label{sec.defense_characterization_general_MDP.unknown_epsilon.overestimating}
When $\epsD \ge \epsA$, our formal analysis builds on the one presented in Section \ref{sec.defense_characterization_general_MDP.known_epsilon}, and we highlight the main differences. 
Given that $\epsA$ is not exactly known, we cannot directly operate on the set $\tightsettarget$.
However, since $\epsD$ upper bounds $\epsA$, the defense can utilize the procedure from the previous section (Theorem \ref{thm.defense_optimization.known_epsilon}) with appropriately chosen $\epsilon$ to solve \eqref{prob.defense_b}
as we show in the following theorem.
\begin{theorem}\label{thm.overestimate_attack_param}
Assume that $\epsD \ge \epsA$, and define $\widehat \epsilon = \min_{s, a\ne \targetpi(s)} \left [ \widehat \score^{\targetpi}-\widehat \score^{\neighbor{\targetpi}{s}{a}} \right]$. 
Then, the optimization problem \eqref{prob.defense_optimization.known_epsilon} with $\epsilon = \min \{ \epsD, \widehat \epsilon \}$ is feasible and its optimal solution $\occupancy_{\max}$ identifies an optimal policy $\defensepi$ for the optimization problem \eqref{prob.defense_b} via Equation
\eqref{eq.defense_optimization.known_epsilon}. 
This policy $\defensepi$ satisfies $\overline \score^{\defensepi} \ge \widehat \score^{\defensepi}$. Furthermore, if the condition in Equation \eqref{eq.attack_influence_condition} holds, 
the attack influence of policy $\defensepi$ is bounded as in Equation \eqref{eq.thm_attack_influence}.
\end{theorem}
To interpret the bounds, let us consider three cases: 
\begin{itemize}
    \item \looseness-1$\overline{R} \ne \widehat{R}$: If the attack indeed poisoned $\overline{R}$, then the smallest $\epsilon' \in (0, \epsD]$ such that $\tightsetprime \ne \emptyset$ corresponds to $\epsA$. In this case, it turns out that 
    $\epsA = \widehat \epsilon$, and somewhat surprisingly, the defense policies of \eqref{prob.defense_a} and \eqref{prob.defense_b} coincide. (Note that this analysis assumes that $\epsD \ge \epsA$.)
    \item $\overline{R} = \widehat{R}$ and $\epsD < \widehat \epsilon$: This corresponds to the case when the attack did not poison $\overline{R}$ and there is no $\epsilon' \in (0, \epsD]$ such that $\tightsetprime \ne \emptyset$. In this case, it turns out that the optimal solution to the optimization problem \eqref{prob.defense_b} is $\defensepi = \targetpi$ (indeed $\targetpi$ is uniquely optimal under $\overline{R}$).
    \item $\overline{R} = \widehat{R}$ and $\epsD \ge \widehat \epsilon$: This corresponds to the case when the attack did not poison $\overline{R}$ and there is $\epsilon' \in (0, \epsD]$ such that $\tightsetprime \ne \emptyset$. In fact, $\widehat \epsilon$ is the smallest such $\epsilon'$. In this case, it turns out that, in general, the optimal solution to the optimization problem \eqref{prob.defense_b} is $\defensepi \ne \targetpi$, even though $\targetpi$ is uniquely optimal under $\overline{R}$. 
\end{itemize}    
These three cases also showcase the importance of choosing $\epsD$ that is a good upper bound on $\epsA$. When $\overline{R} = \widehat{R}$, the agent should select $\epsD$ that is strictly smaller than $\widehat \epsilon$. On the other hand, when $\overline{R} \ne \widehat{R}$, the agent should select $\epsD \ge \widehat \epsilon$, as it will be apparent from the result of the next subsection (Theorem \ref{thm.underestimate_attack_param}). While the agent knows $\widehat \epsilon$, it does not know if $\overline{R} = \widehat{R}$ or $\overline{R} \ne \widehat{R}$.  

\subsection{Underestimating Attack Parameter }\label{sec.defense_characterization_general_MDP.unknown_epsilon.underestimating} 

In this subsection, we analyze the case when $\epsD < \epsA$. We first state our result, and then discuss its implications. 

\begin{theorem}\label{thm.underestimate_attack_param}
If $\epsA > \epsD$, then
$\targetpi$ is the unique solution of
the optimization problem \eqref{prob.defense_b}.
Therefore
$\defensepi=\targetpi$ and
$\influence^{\targetpi} = \influence^{\defensepi}$.
\end{theorem}

Therefore, together with Theorem \ref{thm.overestimate_attack_param}, Theorem \ref{thm.underestimate_attack_param} is showing the importance of having a good prior knowledge about the attack parameter $\epsA$. In particular:
\begin{itemize}
    \item When the attack did not poison the reward function (i.e., $\widehat{R} = \overline{R}$), overestimating $\epsA$ implies that $\defensepi$ might not be equal to $\targetpi$ for larger values of $\epsD$, even though  $\targetpi$ is uniquely optimal under $\overline{R}$. This can have a detrimental effect in terms of attack influence as  $\influence^{\defensepi} > \influence^{\targetpi} = 0$.
    \item When the attack did poison the reward function $\overline{R}$ (i.e., $\widehat{R} \ne \overline{R}$), underestimating $\epsA$ implies $\defensepi = \targetpi$, but $\targetpi$ might be suboptimal. In this case, the defense policy does not limit the influence of the attack at all, i.e., $\influence^{\defensepi} = \influence^{\targetpi} \ge 0$. 
\end{itemize}
We further discuss nuances to selecting $\epsD$ in Section~\ref{sec.conclusion}.
	
	\section{Experimental Evaluation}\label{sec.experiments}
While Figure \ref{fig.example} shows our policy in a simple setting, in this section we evaluate our defense strategy on additional environments in order to better understand its efficacy and robustness.
In the experiments, 
due to limited numerical precision, $\tightset$ is calculated with a tolerance parameter,
which we set to $10^{-4}$ by default.\footnote{
The value was chosen because the CVXPY solver (\cite{diamond2016cvxpy, agrawal2018rewriting}) uses a precision of $10^{-5}$.
}. In other words, $\tightset = \{(s, a): 
|\hatscore^{\targetpi} - \hatscore^{\neighbor{\targetpi}{s}{a}} - \epsilon| \le 10^{-4}\}$.
~\\
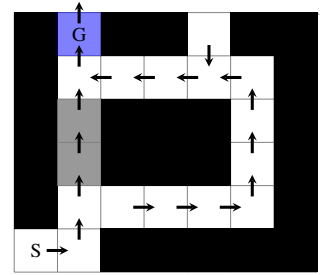
\begin{wrapfigure}[23]{r}{0.35\textwidth}
\centering
\begin{subfigure}{\linewidth}
\centering
    \resizebox{\textwidth}{!}{%
		\begin{tikzpicture}

\node[state, minimum size=1.5cm] (s0) {\scalebox{2.3}{$s_0$}};
\node[state, right=0.6 of s0, minimum size=1.5cm] (s1) {\scalebox{2.3}{$s_1$}};
\node[state, right=0.6 of s1, minimum size=1.5cm] (s2) {\scalebox{2.3}{$s_2$}};
\node[state, right=0.6 of s2, minimum size=1.5cm] (s3) {\scalebox{2.3}{$s_3$}};
\node[state, right=1 of s3, minimum size=1.5cm] (s4) {\scalebox{2.3}{$s_4$}};
\node[state, above right=0.6 of s4, minimum size=1.5cm] (s5) {\scalebox{2.3}{$s_5$}};
\node[state, above right=0.6 of s5, minimum size=1.5cm] (s6) {\scalebox{2.3}{$s_6$}};
\node[state, below right=0.6 of s4, minimum size=1.5cm] (s7) {\scalebox{2.3}{$s_7$}};
\node[state, below right=0.6 of s7, minimum size=1.5cm] (s8) {\scalebox{2.3}{$s_8$}};

\draw[black, ->,>=stealth, line width=4pt](0.00,-1.20) -- (0.60, -1.20);
\draw[black, ->,>=stealth, line width=4pt](2.15,-1.20) -- (2.75, -1.20);
\draw[black, ->,>=stealth, line width=4pt](4.30,-1.20) -- (4.90, -1.20);
\draw[black, ->,>=stealth, line width=4pt](6.45,-1.20) -- (7.05, -1.20);
\draw[black, ->,>=stealth, line width=4pt](8.70,-1.20) -- (9.10, -1.60);
\draw[black, ->,>=stealth, line width=4pt](9.80,-2.30) -- (10.20, -2.70);
\draw[black, ->,>=stealth, line width=4pt](11.20,-3.70) -- (11.20, -4.30);
\draw[black, ->,>=stealth, line width=4pt](9.80, 2.35) -- (9.40, 1.95);
\draw[black, ->,>=stealth, line width=4pt](11.30, 4.00) -- (10.90, 3.60);

\draw (s0) edge[loop left, below, pos=0.2] node{\scalebox{0.8}{\Large }} (s0);
\draw (s0) edge[bend left, above] node{\scalebox{0.8}{\Large }} (s1);
\draw (s1) edge[bend left, below] node{\scalebox{0.8}{\Large }} (s0);
\draw (s1) edge[bend left, above] node{\scalebox{0.8}{\Large }} (s2);
\draw (s2) edge[bend left, below] node{\scalebox{0.8}{\Large }} (s1);
\draw (s2) edge[bend left, above] node{\scalebox{0.8}{\Large }} (s3);
\draw (s3) edge[bend left, below] node{\scalebox{0.8}{\Large }} (s2);
\draw (s3) edge[bend left, below] node{\scalebox{0.8}{\Large }} (s4);

\draw (s4) edge[bend left, left] node{\scalebox{0.8}{\Large }} (s5);
\draw (s4) edge[bend left, right] node{\scalebox{0.8}{\Large }} (s7);

\draw (s5) edge[bend left, left] node{\scalebox{0.8}{\Large }} (s6);
\draw (s5) edge[bend left, right] node{\scalebox{0.8}{\Large }} (s4);

\draw (s7) edge[bend left, left] node{\scalebox{0.8}{\Large }} (s4);
\draw (s7) edge[bend left, right] node{\scalebox{0.8}{\Large }} (s8);

\draw (s8) edge[bend left, left] node{\scalebox{0.8}{\Large }} (s7);
\draw (s6) edge[bend left, right] node{\scalebox{0.8}{\Large }} (s5);

\draw (s6) edge[loop above] node{\scalebox{0.8}{\Large }} (s6);
\draw (s8) edge[loop below] node{\scalebox{0.8}{\Large }} (s8);
\end{tikzpicture}
	}
	\caption{\label{fig.env.nav}The navigation environment}
\end{subfigure}
\begin{subfigure}{\linewidth}
\centering
	\resizebox{0.84\textwidth}{!}{%
		\begin{tikzpicture}
	\draw[step=2cm,black,thin,opacity=0.5] (0,0) grid (14,12);
	\fill[black] (0,2) rectangle (2,12);
	\fill[black] (4,0) rectangle (14,2);
	\fill[black] (4,4) rectangle (10,8);
	\fill[black] (4,10) rectangle (8,12);
	\fill[black] (12,0) rectangle (14,12);
	\fill[black] (10,10) rectangle (12,12);
	\fill[gray,opacity=0.8] (2,4) rectangle (4,8);
	\fill[blue,opacity=0.5] (2,10) rectangle (4,12);
	\node at (1,1) {\scalebox{1.5}{\LARGE S}};
	\node at (3,11) {\scalebox{1.5}{\LARGE G}};
	
	
%
	
	
	\draw [black,->,>=stealth,  line width=4pt] (1.5,1) -- (2.5,1);
	
	\draw [black,->,>=stealth, line width = 4pt] (3,1.5) -- (3,2.5);
	
	\draw [black,->,>=stealth, line width = 4pt] (3,3.5) -- (3,4.5);
	
		\draw [black,->,>=stealth, line width = 4pt] (5.5,3) -- (6.5,3);
	
		\draw [black,->,>=stealth, line width=4pt] (7.5,3) -- (8.5,3);
	
		\draw [black,->,>=stealth, line width=4pt] (9.5,3) -- (10.5,3);
	%
		\draw [black,->,>=stealth, line width=4pt] (11,3.5) -- (11,4.5);
	
		\draw [black,->,>=stealth, line width=4pt] (11,5.5) -- (11,6.5);
	%
	\draw [black,->,>=stealth, line width=4pt] (11,7.5) -- (11,8.5);
	%
		\draw [black,->,>=stealth, line width=4pt] (10.5,9) -- (9.5,9);
	
	\draw [black,->,>=stealth, line width=4pt] (8.5,9) -- (7.5,9);
	\draw [black,->,>=stealth, line width=4pt] (9,10.5) -- (9,9.5);
	
		\draw [black,->,>=stealth, line width=4pt] (6.5,9) -- (5.5,9);
	
	\draw [black,->,>=stealth, line width=4pt] (4.5,9) -- (3.5,9);
	
	\draw [black,->,>=stealth, line width=4pt] (3,9.5) -- (3,10.5);
	\draw [black,->,>=stealth, line width=4pt] (3,7.5) -- (3,8.5);
	
		\draw [black,->,>=stealth, line width=4pt] (3, 5.5) -- (3,6.5);
%
	
	\draw [black,->,>=stealth, line width=4pt] (3,11.5) -- (3,12.5);
	
\end{tikzpicture}
	}
    \caption{\label{fig.env.grid}The grid world environment}
\end{subfigure}
\caption{Environments}
\end{wrapfigure}
\textbf{Navigation environment}. 
Our first environment, shown in 
Figure \ref{fig.env.nav} is the Navigation environment taken from \cite{rakhsha2020policy-jmlr}. The environment has 9 states and 2 possible actions. The reward function is action independent and has the following values:
$\overline{R}(s_0, .)=\overline{R}(s_1, .)=\overline{R}(s_2, .)=\overline{R}(s_3, .)=-2.5$, $\overline{R}(s_4, .)=\overline{R}(s_5, .)=1$ and $\overline{R}(s_6, .)=\overline{R}(s_7, .)=\overline{R}(s_8, .)=0$.
When the agent takes an action, it will successfully navigate in the direction shown by the arrows with probability $0.9$; otherwise, the next state will be sampled uniformly at random. The bold arrows in the figure indicate the attacker's target policy. The initial state is $s_0$ and the discounting factor $\gamma$ equals $0.99$. 
~\\
\textbf{Grid world environment.}
For our second environment, shown in Figure \ref{fig.env.grid}, we use the grid world environment from \cite{ma2019policy} with slight modifications in order to ensure ergodicity --- we add a $10\%$ failure probability to each action, sampling the next state randomly in case of failure.
The environment has 18 states and 4 actions: \textit{up}, \textit{down}, \textit{right} and \textit{left}. The white, gray and blue cells in the figure represent the states and the black cells represent walls. In the white and gray states, the agent will attempt to go in the direction specified by its action if there is a neighboring state in that direction. If there is no such state, the agent will attempt to stay in its own place. 
In the blue state $G$, the agent will attempt to stay in its own place regardless of the action taken. In all states, each attempt will succeed with probability $0.9$; with probability $0.1$, the next state will be sampled uniformly at random.
	In the gray and white states, the agent's reward 
	is a function of the state it is attempting to visit. Attempting to visit a gray, white and blue state will yield a reward of $-10$, $-1$ and $2$ respectively. If the agent is in a blue state, it will always receive a reward of $0$. The bold arrows in the figure specify the attacker's target policy. The initial state is $S$ and the discounting factor $\gamma$ equals $0.9$.
\begin{figure*}[ht]
	\centering
    \begin{subfigure}[t]{.48\textwidth}
        \centering
		\resizebox{0.8\linewidth}{!}{\includegraphics{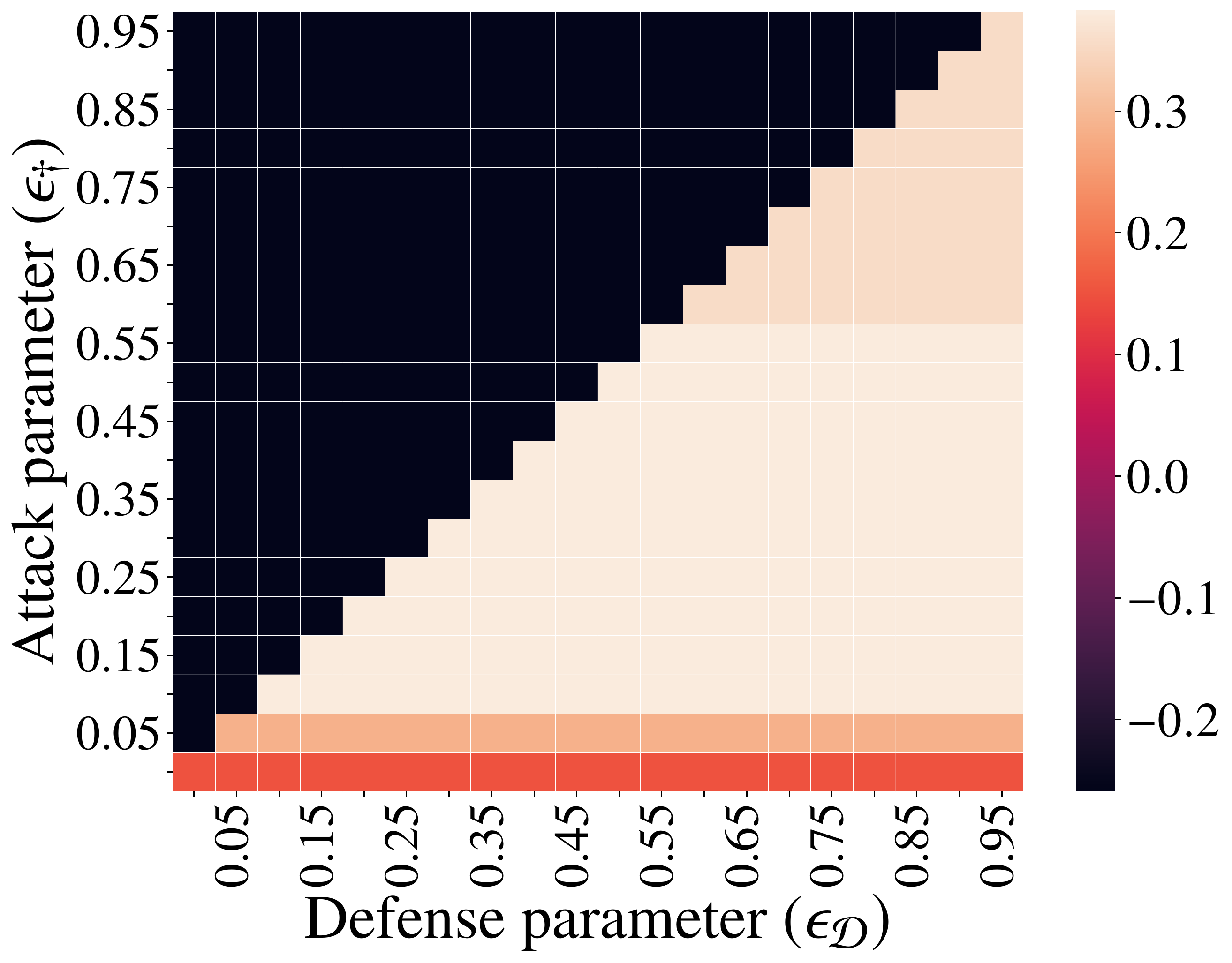}}
		\caption{$\barscore^{\defensepi}$  in the navigation environment.
		For comparison, $\barscore^{\targetpi}=-0.26$ and
		$\barscore^{\optpi}=0.45$.
		}\label{fig.heatmap.nav}
	\end{subfigure}
	\quad
	\begin{subfigure}[t]{.48\textwidth}
	    \centering
		\resizebox{0.8\linewidth}{!}{\includegraphics{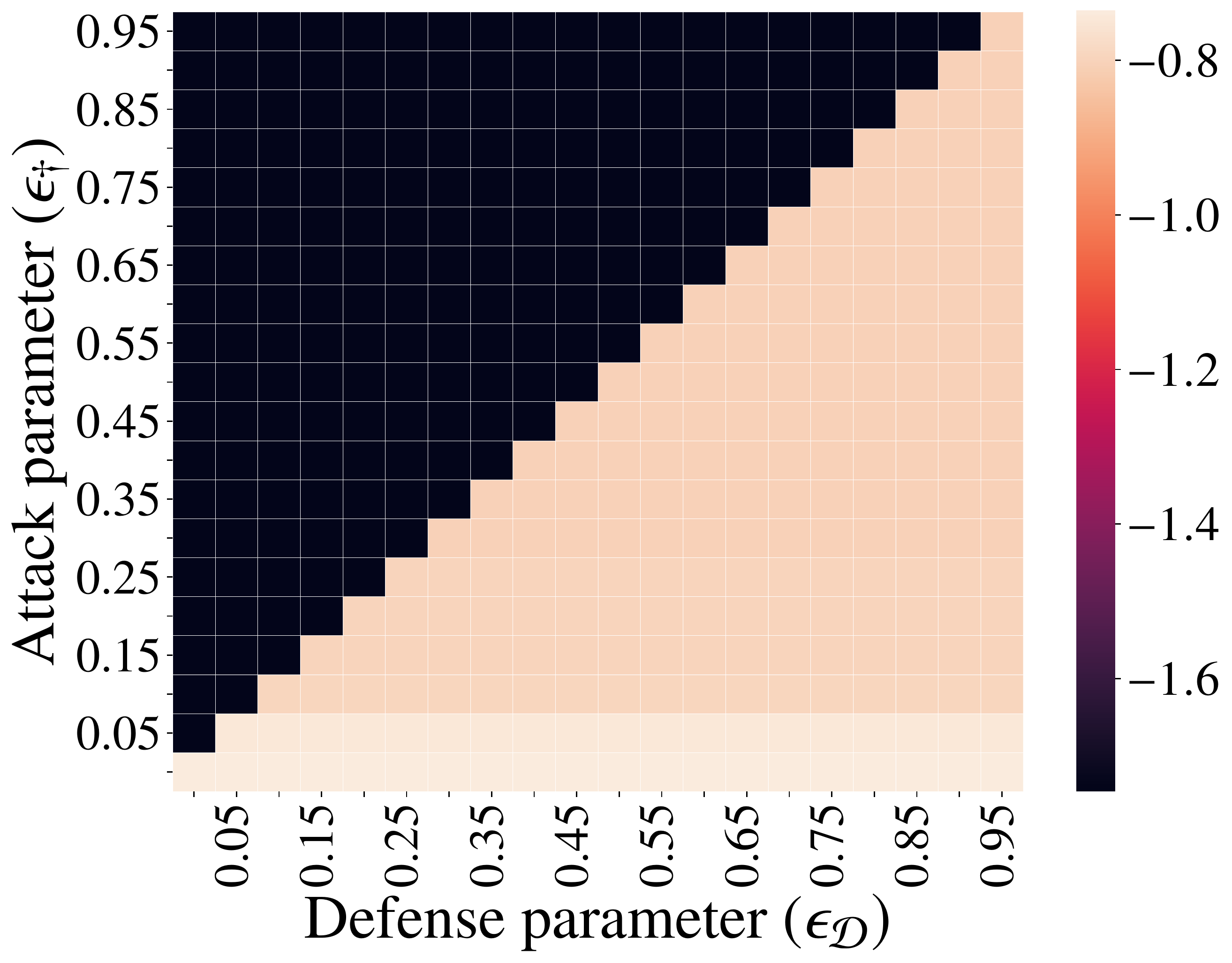}}
		\caption{$\barscore^{\defensepi}$  in the grid world environment.
		For comparison, $\barscore^{\targetpi}=-1.75$ and
		$\barscore^{\optpi}=-0.70$.
		}\label{fig.heatmap.grid}
	\end{subfigure}\\
	\caption{
	Score of the defense policy $\defensepi$ w.r.t $\overline{R}$ for different values of
	$\epsA$ and $\epsD$.
	}
	\vspace{-2mm}
	\label{fig.heatmap}
\end{figure*}
~\\
\textbf{Policy score for different values of parameters.} 
We first analyze the score of our defense policy
in both environments with different values of the attack parameter ($\epsA$) and defense parameter ($\epsD$). 
For comparison, we also report the scores of the target policy ($\targetpi$) and the optimal policy ($\optpi$).
The results are shown in Figures \ref{fig.heatmap}. As seen in the figures, as long as $\epsD \ge \epsA$, our defense policy significantly improves the agent's score compared to $\targetpi$. 
~\\
\textbf{Robustness to perturbations.} 
We now analyze our algorithm's robustness towards uncertainties in the reward functions used by the attacker and the defender. 
For our first experiment, which we call PreAttack, we add a randomly perturb 
the attacker's input.
 In particular,
 the input to the defender's optimization problem $\widehat{R}$ is sampled from 
$\attack(\overline{R} + \mathcal{N}(0, \sigma^2I), \targetpi, \epsA)$  where $I$
is the
identity
matrix, $\mathcal{N}$ denotes the multivariate normal distribution and $\sigma$ is the perturbation parameter varied in the experiment. For our second experiment, called PostAttack, we randomly perturb the reward vector after the attack, sampling the defender's input from 
$\attack(\overline{R}, \targetpi, \epsA) + \mathcal{N}(0, \sigma^2I)$. In both experiments we use $\epsA=0.1$ and $\epsD=\infty$. 
As explained below, when calculating $\tightset$, we also experiment with a larger tolerance parameter of $10^{-1}$, denoting the defense policy in this case with $\pi_{\mathcal{D}+}$.
~\\
The results can be seen in Figure \ref{fig.robustness}. As seen in the figures, our defense policy
$\defensepi$
consistently improves on the baseline obtained with no defense (i.e, $\targetpi$). It is also clear that the PostAttack perturbations have a greater negative impact on our defense strategy's score.
Results for $\pi_{\mathcal D+}$ indicate that this is due to random perturbations prohibiting our algorithm from identifying all of the elements in $\tightset$.
While having a higher tolerance parameter helps with robustness, it can also lead to a lower performance when there is no noise, as $\tightset$ would falsely include additional elements.
We leave choosing the tolerance parameter in a more systematic way for future work.
\begin{figure*}[ht]
	\centering
	\begin{subfigure}{0.24\textwidth}
    	\centering
		\resizebox{0.98\linewidth}{!}{\includegraphics{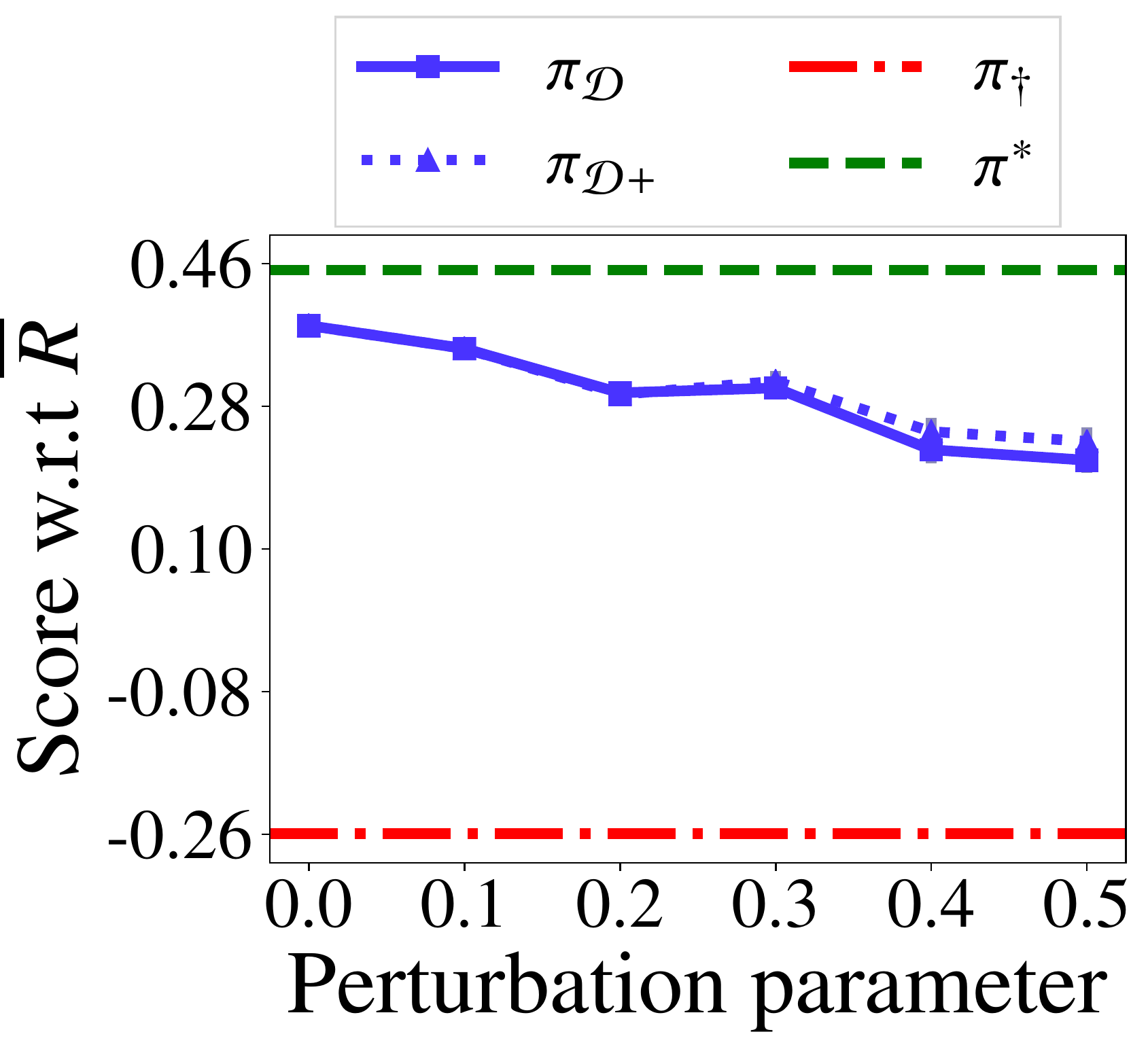}}
		\caption{Navigation, PreAttack}\label{fig.robustness.a}
	\end{subfigure}
	\begin{subfigure}{0.24\textwidth}
    	\centering
		\resizebox{0.98\linewidth}{!}{\includegraphics{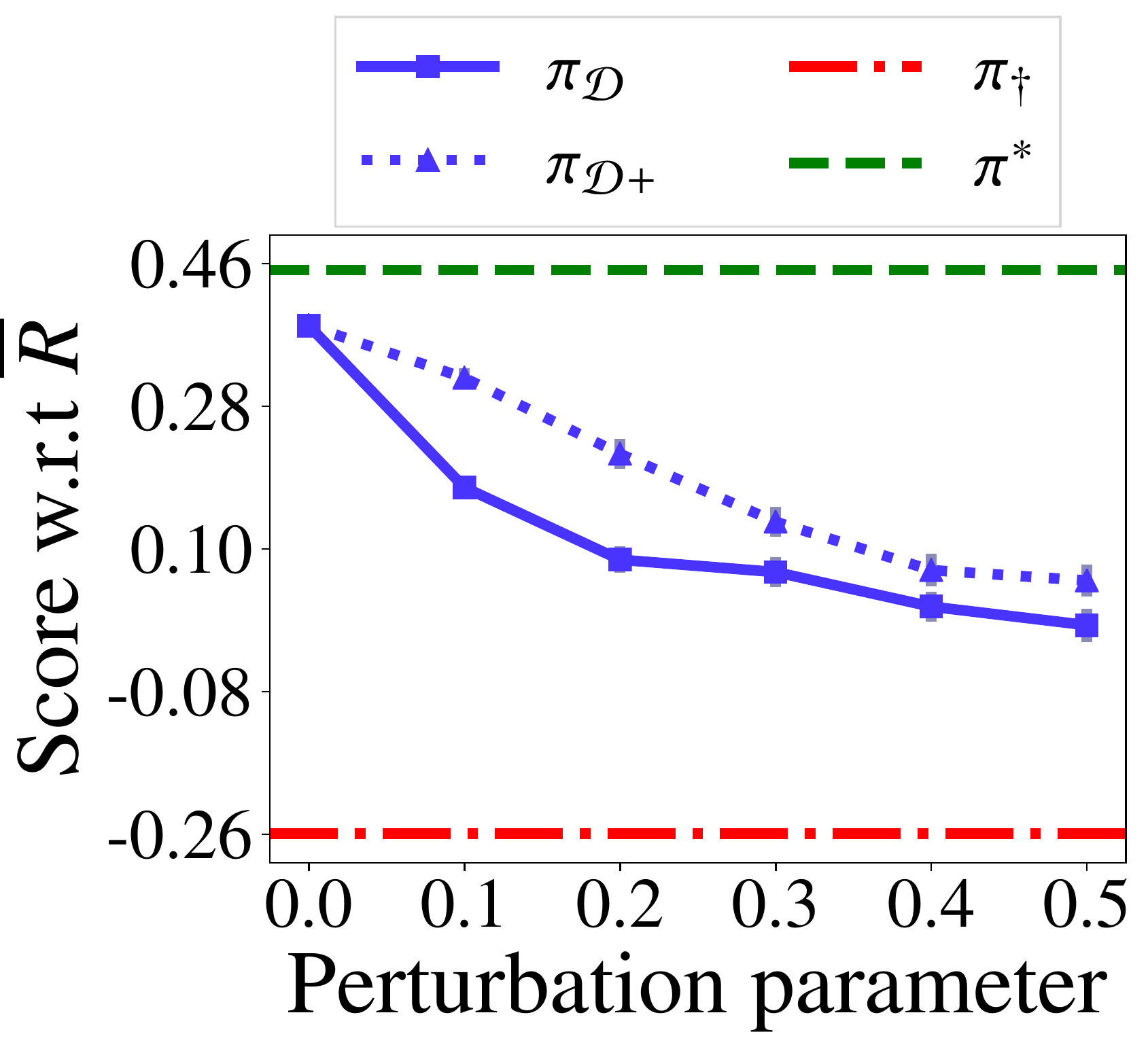}}
		\caption{Navigation, PostAttack}\label{fig.robustness.a}
	\end{subfigure}
	\begin{subfigure}{.24\textwidth}
    	\centering
		\resizebox{0.98\linewidth}{!}{\includegraphics{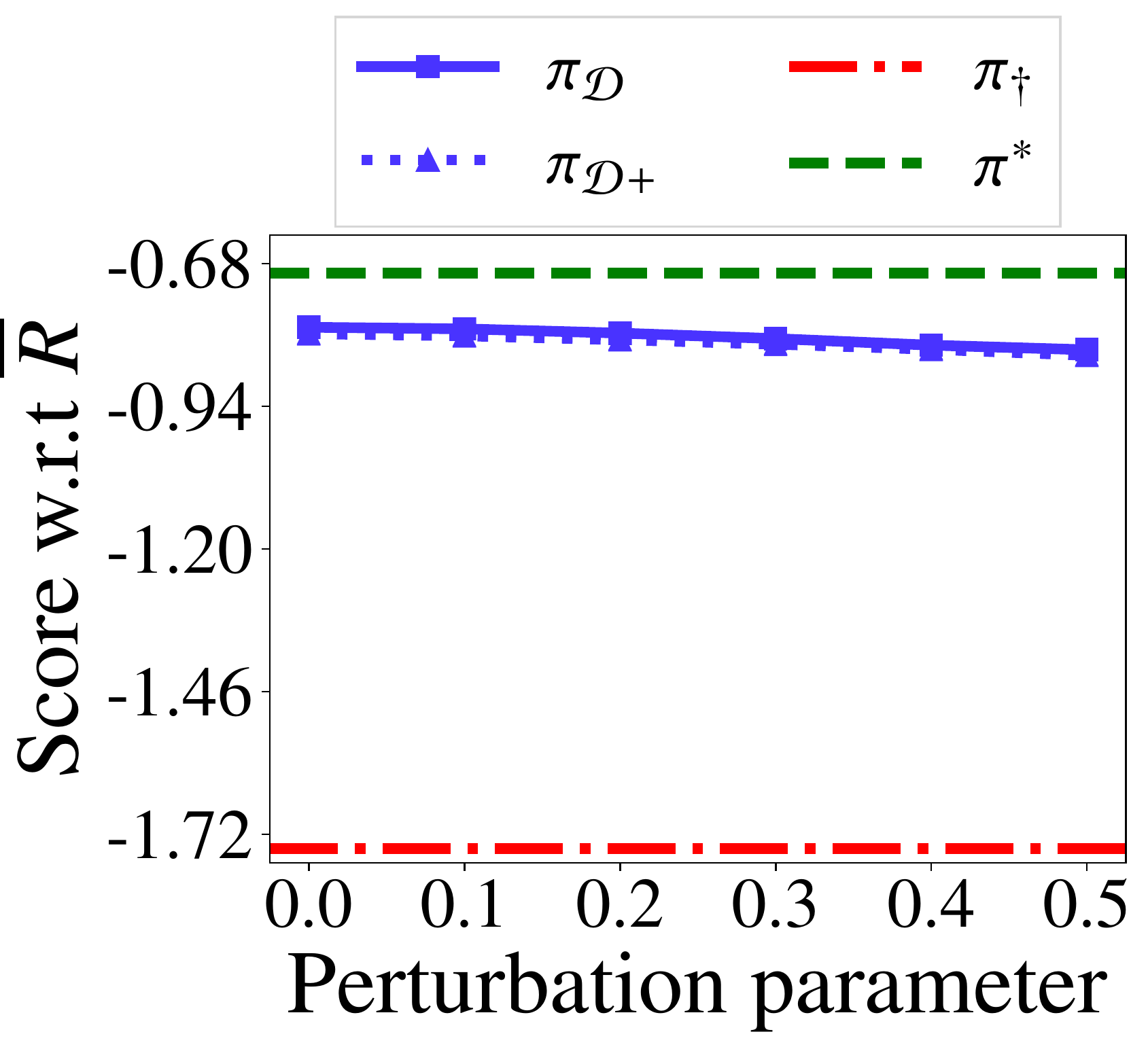}}
		\caption{Grid world, PreAttack}\label{fig.robustness.a}
	\end{subfigure}
	\begin{subfigure}{.24\textwidth}
    	\centering
		\resizebox{0.98\linewidth}{!}{\includegraphics{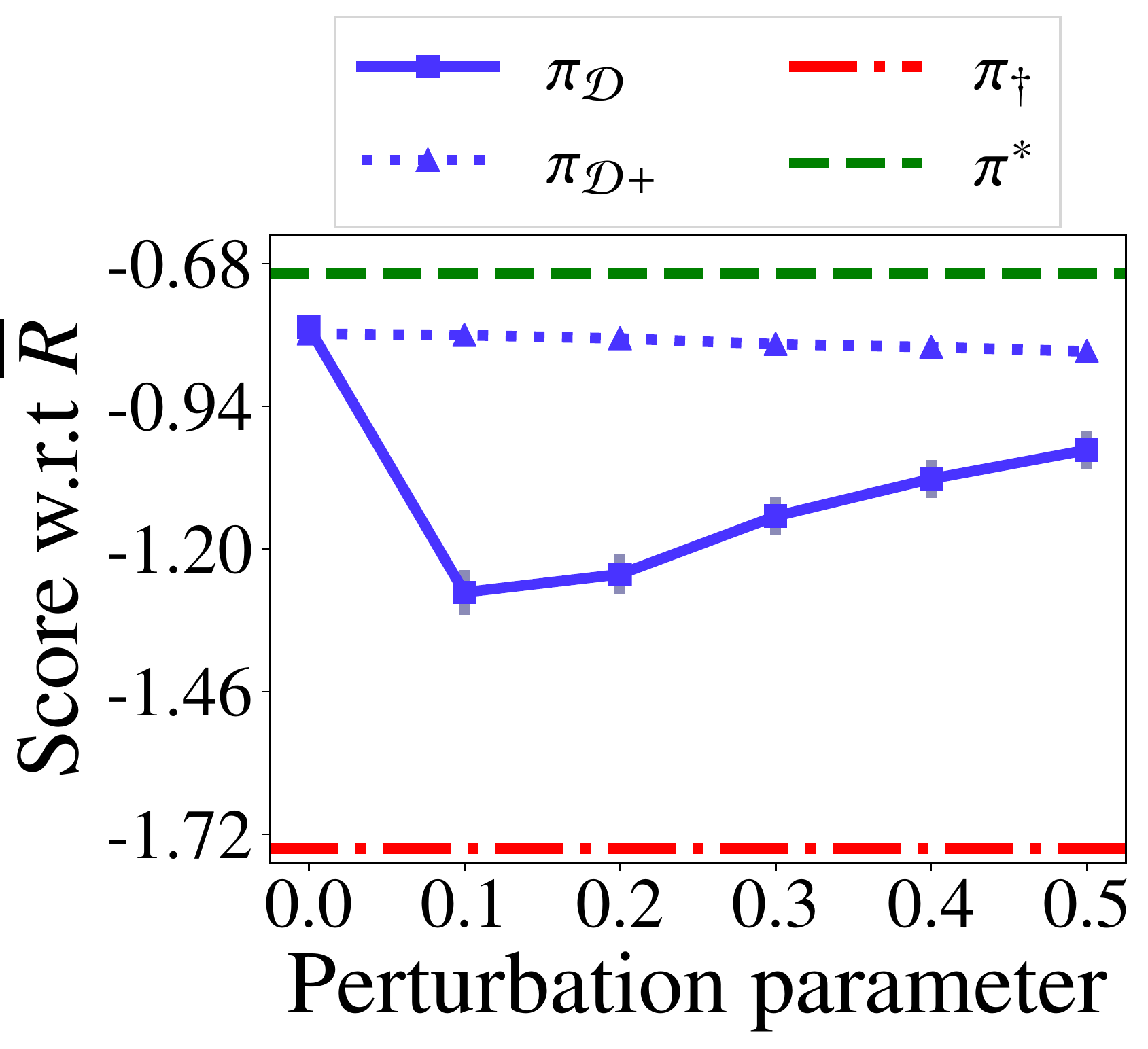}}
		\caption{Grid world, PostAttack}\label{fig.robustness.a}
	\end{subfigure}
    \caption{Robustness of the defense policy against random perturbation. Results are based on average of 100 runs for each data point.
    Error bars around the data points indicate standard error.
	}
	\vspace{-6mm}
	\label{fig.robustness}
\end{figure*}
\section{Concluding Discussions}\label{sec.conclusion}
In this paper, we introduced an optimization framework for designing defense strategies against reward poisoning attacks, in particular, poisoning attacks that change an agent's reward structure in order to steer the agent to adopt a target policy. We further analyzed the utility of using such defense strategies, providing characterization results that specify provable guarantees on their performance. Moving forward we see several interesting future research directions for extending these results.
~\\
\textbf{More refined analysis.} While our theoretical characterization shows
concrete benefits of using the optimization problems \eqref{prob.defense_a} and \eqref{prob.defense_b} as defense strategies, some of the bounds related to the notion of attack influence (i.e.,  Theorem \ref{thm.attack_influence_a}, Theorem \ref{cor.attack_influence_a}, and Theorem \ref{thm.overestimate_attack_param}) depend on the poisoned reward function. In the future work, it would be very interesting to establish bounds that do not have such dependency, or prove that this is not possible.
~\\
\textbf{Beyond the worst-case utility.} In this paper, we defined the defense objective
as the maximization of the agent's worst-case utility. 
While this is a sensible objective, there are other objectives that one could analyze.
For example, instead of focusing on the absolute performance, one can try to optimize performance relative to the target policy. Notice that this is a somewhat different, and possibly weaker goal, given that the target policy can have arbitrarily bad utility under 
$\overline{R}$.
~\\
\textbf{Informed prior.} We did not model prior knowledge that an agent might have about the attacker or the underlying reward function. In practice, we can expect that an agent has some information about, for example, the underlying true reward function. Incorporating such considerations calls for a Bayesian approach that could increase the effectiveness of the agent's defense by, for example, ruling out implausible candidates for $\overline{R}$ in the agent's inference of $\overline{R}$ given $\widehat R$. 
~\\
\textbf{Selecting $\epsD$ and non-oblivious attacks.}
The results in Section \ref{sec.defense_characterization_general_MDP.unknown_epsilon} indicate that choosing good $\epsD$ is important for having a functional defense. 
In practice, a selection procedure for $\epsD$ should take into account the cost that the attacker has for different choices of $\epsA$, as well as game-theoretic considerations: attacks might not be {\em oblivious} in that the strategy for selecting $\epsA$ might depend on the strategy for selecting $\epsD$. Namely, a direct consequence of Theorem \ref{sec.defense_characterization_general_MDP.unknown_epsilon} is that the attack optimization problem \eqref{prob.attack} can successfully achieve its goal if it sets $\epsA$ to large enough values. However, the cost of the attack also grows with $\epsA$, so the attack (if strategic) also needs to reason about $\epsD$ when selecting $\epsA$. We leave the full game-theoretic characterization of the parameter selection problem for the future work, as well as the inspection of other types of non-oblivious attacks (e.g., where the attack optimization problem has a different structure).
~\\
\textbf{Scaling up via function approximation}. 
While in our work we considered a tabular setting, large scale RL problems typically rely on function approximation. 
An interesting direction for future work would be to analyze defense strategies in this setting.

    \bibliographystyle{abbrv}

	\bibliography{main}

\begin{thebibliography}{10}

\bibitem{agrawal2018rewriting}
A.~Agrawal, R.~Verschueren, S.~Diamond, and S.~Boyd.
\newblock A rewriting system for convex optimization problems.
\newblock {\em Journal of Control and Decision}, 5(1):42--60, 2018.

\bibitem{amir2020prediction}
I.~Amir, I.~Attias, T.~Koren, R.~Livni, and Y.~Mansour.
\newblock Prediction with corrupted expert advice.
\newblock {\em CoRR}, abs/2002.10286, 2020.

\bibitem{bagnell2001solving}
J.~A. Bagnell, A.~Y. Ng, and J.~G. Schneider.
\newblock Solving uncertain markov decision processes.
\newblock Technical report, Carnegie Mellon University, 2001.

\bibitem{behzadan2017whatever}
V.~Behzadan and A.~Munir.
\newblock Whatever does not kill deep reinforcement learning, makes it
  stronger.
\newblock {\em CoRR}, abs/1712.09344, 2017.

\bibitem{DBLP:conf/icml/BiggioNL12}
B.~Biggio, B.~Nelson, and P.~Laskov.
\newblock Poisoning attacks against support vector machines.
\newblock In {\em ICML}, 2012.

\bibitem{biggio2018wild}
B.~Biggio and F.~Roli.
\newblock Wild patterns: Ten years after the rise of adversarial machine
  learning.
\newblock {\em Pattern Recognition}, 84:317--331, 2018.

\bibitem{bogunovic2020stochastic}
I.~Bogunovic, A.~Losalka, A.~Krause, and J.~Scarlett.
\newblock Stochastic linear bandits robust to adversarial attacks.
\newblock {\em CoRR}, abs/2007.03285, 2020.

\bibitem{charikar2017learning}
M.~Charikar, J.~Steinhardt, and G.~Valiant.
\newblock Learning from untrusted data.
\newblock In {\em STOC}, pages 47--60, 2017.

\bibitem{cretu2008casting}
G.~F. Cretu, A.~Stavrou, M.~E. Locasto, S.~J. Stolfo, and A.~D. Keromytis.
\newblock Casting out demons: Sanitizing training data for anomaly sensors.
\newblock In {\em IEEE Symposium on Security and Privacy}, pages 81--95. IEEE,
  2008.

\bibitem{diakonikolas2019sever}
I.~Diakonikolas, G.~Kamath, D.~Kane, J.~Li, J.~Steinhardt, and A.~Stewart.
\newblock Sever: A robust meta-algorithm for stochastic optimization.
\newblock In {\em ICML}, pages 1596--1606, 2019.

\bibitem{diamond2016cvxpy}
S.~Diamond and S.~Boyd.
\newblock Cvxpy: A python-embedded modeling language for convex optimization.
\newblock {\em The Journal of Machine Learning Research}, 17(1):2909--2913,
  2016.

\bibitem{dimitrakakis2017multi}
C.~Dimitrakakis, D.~C. Parkes, G.~Radanovic, and P.~Tylkin.
\newblock Multi-view decision processes: The helper-ai problem.
\newblock In {\em NeurIPS}, pages 5443--5452, 2017.

\bibitem{ethicsEU}
{European Commission}.
\newblock {Ethics Guidelines for Trustworthy Artificial Intelligence}.
\newblock URL:
  \url{https://ec.europa.eu/digital-single-market/en/news/ethics-guidelines-trustworthy-ai},
  2019.
\newblock [Online; accessed 15-January-2021].

\bibitem{ghosh2020towards}
A.~Ghosh, S.~Tschiatschek, H.~Mahdavi, and A.~Singla.
\newblock Towards deployment of robust cooperative ai agents: An algorithmic
  framework for learning adaptive policies.
\newblock In {\em AAMAS}, pages 447--455, 2020.

\bibitem{gupta2019better}
A.~Gupta, T.~Koren, and K.~Talwar.
\newblock Better algorithms for stochastic bandits with adversarial
  corruptions.
\newblock In {\em COLT}, pages 1562--1578, 2019.

\bibitem{hamon2020robustness}
R.~Hamon, H.~Junklewitz, and I.~Sanchez.
\newblock Robustness and explainability of artificial intelligence.
\newblock {\em Publications Office of the European Union}, 2020.

\bibitem{huang2011adversarial}
L.~Huang, A.~D. Joseph, B.~Nelson, B.~I. Rubinstein, and J.~D. Tygar.
\newblock Adversarial machine learning.
\newblock In {\em ACM workshop on Security and artificial intelligence}, pages
  43--58, 2011.

\bibitem{huang2017adversarial}
S.~Huang, N.~Papernot, I.~Goodfellow, Y.~Duan, and P.~Abbeel.
\newblock Adversarial attacks on neural network policies.
\newblock {\em CoRR}, abs/1702.02284, 2017.

\bibitem{DBLP:conf/gamesec/HuangZ19a}
Y.~Huang and Q.~Zhu.
\newblock Deceptive reinforcement learning under adversarial manipulations on
  cost signals.
\newblock In {\em GameSec}, pages 217--237, 2019.

\bibitem{iyengar2005robust}
G.~N. Iyengar.
\newblock Robust dynamic programming.
\newblock {\em Mathematics of Operations Research}, 30(2):257--280, 2005.

\bibitem{DBLP:conf/nips/Jun0MZ18}
K.~Jun, L.~Li, Y.~Ma, and X.~Zhu.
\newblock Adversarial attacks on stochastic bandits.
\newblock In {\em NeurIPS}, pages 3644--3653, 2018.

\bibitem{koh2017understanding}
P.~W. Koh and P.~Liang.
\newblock Understanding black-box predictions via influence functions.
\newblock In {\em ICML}, pages 1885--1894. PMLR, 2017.

\bibitem{DBLP:journals/corr/abs-1811-00741}
P.~W. Koh, J.~Steinhardt, and P.~Liang.
\newblock Stronger data poisoning attacks break data sanitization defenses.
\newblock {\em CoRR}, abs/1811.00741, 2018.

\bibitem{li2016data}
B.~Li, Y.~Wang, A.~Singh, and Y.~Vorobeychik.
\newblock Data poisoning attacks on factorization-based collaborative
  filtering.
\newblock In {\em NeurIPS}, pages 1885--1893, 2016.

\bibitem{lim2013reinforcement}
S.~H. Lim, H.~Xu, and S.~Mannor.
\newblock Reinforcement learning in robust markov decision processes.
\newblock In {\em NeurIPS}, pages 701--709, 2013.

\bibitem{DBLP:conf/ijcai/LinHLSLS17}
Y.~Lin, Z.~Hong, Y.~Liao, M.~Shih, M.~Liu, and M.~Sun.
\newblock Tactics of adversarial attack on deep reinforcement learning agents.
\newblock In {\em IJCAI}, pages 3756--3762, 2017.

\bibitem{DBLP:conf/icml/LiuS19a}
F.~Liu and N.~B. Shroff.
\newblock Data poisoning attacks on stochastic bandits.
\newblock In {\em ICML}, pages 4042--4050, 2019.

\bibitem{lykouris2018stochastic}
T.~Lykouris, V.~Mirrokni, and R.~Paes~Leme.
\newblock Stochastic bandits robust to adversarial corruptions.
\newblock In {\em STOC}, pages 114--122, 2018.

\bibitem{lykouris2019corruption}
T.~Lykouris, M.~Simchowitz, A.~Slivkins, and W.~Sun.
\newblock Corruption robust exploration in episodic reinforcement learning.
\newblock {\em CoRR}, abs/1911.08689, 2019.

\bibitem{DBLP:conf/gamesec/MaJ0018}
Y.~Ma, K.~Jun, L.~Li, and X.~Zhu.
\newblock Data poisoning attacks in contextual bandits.
\newblock In {\em GameSec}, pages 186--204, 2018.

\bibitem{ma2019policy}
Y.~Ma, X.~Zhang, W.~Sun, and J.~Zhu.
\newblock Policy poisoning in batch reinforcement learning and control.
\newblock In {\em NeurIPS}, pages 14543--14553, 2019.

\bibitem{mcmahan2003planning}
H.~B. McMahan, G.~J. Gordon, and A.~Blum.
\newblock Planning in the presence of cost functions controlled by an
  adversary.
\newblock In {\em ICML}, pages 536--543, 2003.

\bibitem{mei2015using}
S.~Mei and X.~Zhu.
\newblock Using machine teaching to identify optimal training-set attacks on
  machine learners.
\newblock In {\em AAAI}, pages 2871--2877, 2015.

\bibitem{moosavi2016deepfool}
S.-M. Moosavi-Dezfooli, A.~Fawzi, and P.~Frossard.
\newblock Deepfool: a simple and accurate method to fool deep neural networks.
\newblock In {\em CVPR}, pages 2574--2582, 2016.

\bibitem{nelson2008exploiting}
B.~Nelson, M.~Barreno, F.~J. Chi, A.~D. Joseph, B.~I. Rubinstein, U.~Saini,
  C.~A. Sutton, J.~D. Tygar, and K.~Xia.
\newblock Exploiting machine learning to subvert your spam filter.
\newblock {\em LEET}, 8:1--9, 2008.

\bibitem{nguyen2015deep}
A.~Nguyen, J.~Yosinski, and J.~Clune.
\newblock Deep neural networks are easily fooled: High confidence predictions
  for unrecognizable images.
\newblock In {\em CVPR}, pages 427--436, 2015.

\bibitem{nilim2005robust}
A.~Nilim and L.~El~Ghaoui.
\newblock Robust control of markov decision processes with uncertain transition
  matrices.
\newblock {\em Operations Research}, 53(5):780--798, 2005.

\bibitem{paudice2018detection}
A.~Paudice, L.~Mu{\~n}oz-Gonz{\'a}lez, A.~Gyorgy, and E.~C. Lupu.
\newblock Detection of adversarial training examples in poisoning attacks
  through anomaly detection.
\newblock {\em CoRR}, abs/1802.03041, 2018.

\bibitem{pinto2017robust}
L.~Pinto, J.~Davidson, R.~Sukthankar, and A.~Gupta.
\newblock Robust adversarial reinforcement learning.
\newblock In {\em ICML}, pages 2817--2826, 2017.

\bibitem{Puterman1994}
M.~L. Puterman.
\newblock {\em Markov Decision Processes: Discrete Stochastic Dynamic
  Programming}.
\newblock John Wiley \& Sons, Inc., 1994.

\bibitem{radanovic2019learning}
G.~Radanovic, R.~Devidze, D.~Parkes, and A.~Singla.
\newblock Learning to collaborate in markov decision processes.
\newblock In {\em ICML}, pages 5261--5270, 2019.

\bibitem{rakhsha2020policy-jmlr}
A.~Rakhsha, G.~Radanovic, R.~Devidze, X.~Zhu, and A.~Singla.
\newblock Policy teaching in reinforcement learning via environment poisoning
  attacks.
\newblock {\em CoRR}, abs/2011.10824, 2020.

\bibitem{rakhsha2020policy}
A.~Rakhsha, G.~Radanovic, R.~Devidze, X.~Zhu, and A.~Singla.
\newblock Policy teaching via environment poisoning: Training-time adversarial
  attacks against reinforcement learning.
\newblock In {\em ICML}, 2020.

\bibitem{regan2010robust}
K.~Regan and C.~Boutilier.
\newblock Robust policy computation in reward-uncertain mdps using nondominated
  policies.
\newblock In {\em AAAI}, volume~24, 2010.

\bibitem{schulman2015trust}
J.~Schulman, S.~Levine, P.~Abbeel, M.~Jordan, and P.~Moritz.
\newblock Trust region policy optimization.
\newblock In {\em ICML}, pages 1889--1897, 2015.

\bibitem{steinhardt2017certified}
J.~Steinhardt, P.~W. Koh, and P.~Liang.
\newblock Certified defenses for data poisoning attacks.
\newblock In {\em NeurIPS}, pages 3520--3532, 2017.

\bibitem{sun2020vulnerability}
Y.~Sun and F.~Huang.
\newblock Vulnerability-aware poisoning mechanism for online rl with unknown
  dynamics.
\newblock {\em CoRR}, abs/2009.00774, 2020.

\bibitem{sutton2018reinforcement}
R.~S. Sutton and A.~G. Barto.
\newblock {\em Reinforcement learning: An introduction}.
\newblock MIT press, 2018.

\bibitem{syed2008apprenticeship}
U.~Syed, M.~Bowling, and R.~E. Schapire.
\newblock Apprenticeship learning using linear programming.
\newblock In {\em ICML}, pages 1032--1039, 2008.

\bibitem{szegedy2014intriguing}
C.~Szegedy, W.~Zaremba, I.~Sutskever, J.~Bruna, D.~Erhan, I.~Goodfellow, and
  R.~Fergus.
\newblock Intriguing properties of neural networks.
\newblock In {\em ICLR}, 2014.

\bibitem{szepesvari1997asymptotic}
C.~Szepesv{\'a}ri.
\newblock The asymptotic convergence-rate of q-learning.
\newblock In {\em NeurIPS}, volume~10, pages 1064--1070, 1997.

\bibitem{tamar2014scaling}
A.~Tamar, S.~Mannor, and H.~Xu.
\newblock Scaling up robust mdps using function approximation.
\newblock In {\em ICML}, pages 181--189, 2014.

\bibitem{tretschk2018sequential}
E.~Tretschk, S.~J. Oh, and M.~Fritz.
\newblock Sequential attacks on agents for long-term adversarial goals.
\newblock {\em CoRR}, abs/1805.12487, 2018.

\bibitem{DBLP:conf/icml/XiaoBBFER15}
H.~Xiao, B.~Biggio, G.~Brown, G.~Fumera, C.~Eckert, and F.~Roli.
\newblock Is feature selection secure against training data poisoning?
\newblock In {\em ICML}, pages 1689--1698, 2015.

\bibitem{xiao2012adversarial}
H.~Xiao, H.~Xiao, and C.~Eckert.
\newblock Adversarial label flips attack on support vector machines.
\newblock In {\em ECAI}, pages 870--875, 2012.

\bibitem{zhang2020robust}
H.~Zhang, H.~Chen, C.~Xiao, B.~Li, D.~Boning, and C.-J. Hsieh.
\newblock Robust deep reinforcement learning against adversarial perturbations
  on observations.
\newblock {\em CoRR}, abs/2003.08938, 2020.

\bibitem{DBLP:conf/aaai/ZhangP08}
H.~Zhang and D.~C. Parkes.
\newblock Value-based policy teaching with active indirect elicitation.
\newblock In {\em AAAI}, 2008.

\bibitem{zhang2021robustICML}
X.~Zhang, Y.~Chen, X.~Zhu, and W.~Sun.
\newblock Robust policy gradient against strong data corruption.
\newblock {\em CoRR}, abs/2102.05800, 2021.

\bibitem{xuezhou2020adaptive}
X.~Zhang, Y.~Ma, A.~Singla, and X.~Zhu.
\newblock Adaptive reward-poisoning attacks against reinforcement learning.
\newblock In {\em ICML}, 2020.

\bibitem{zhang2018training}
X.~Zhang, X.~Zhu, and S.~Wright.
\newblock Training set debugging using trusted items.
\newblock In {\em AAAI}, volume~32, 2018.

\end{thebibliography}
	\iftoggle{longversion}{
		\clearpage
		\onecolumn
		\appendix 
		{\allowdisplaybreaks

\section{List of Appendices}\label{sec.appendix.table-of-contents}
In this section we provide a brief description of the content provided in the appendices of the paper.   
\begin{itemize}
\item Appendix \ref{sec.defense_characterization_specific_mdp} provides an intuition of our results using special MDPs in which the agent's actions do not affect the transition dynamics. The proofs of the results presented in this Appendix can be found in Appendix \ref{sec.appendix.additional_results_special_mdps}.
\item Appendix \ref{sec.appendix.experiments} provides additional details regarding the experiments. 
\item Appendix \ref{sec.appendix.background} contains some background on reward poisoning attacks, and a brief overview of the MDP properties
that are important for proving our formal results.  

\item Appendix \ref{sec.appendix.attack.characterization} contains characterization results for the attack optimization problem \eqref{prob.attack}.

\item Appendix \ref{sec.appendix.proofs_defense_characterization_general_MDP_known_epsilon} contains proofs of the formal results in Section \ref{sec.defense_characterization_general_MDP.known_epsilon}.
\begin{itemize}
    \item The proof of Lemma \ref{lm.tightset} is in Section \ref{sec.appendix.proof_lm_tightset}.
    \item The proof of Theorem \ref{thm.defense_optimization.known_epsilon} is in Section \ref{sec.appendix.proof_thm_defense_optimization.known_epsilon}.
    \item The proof of Theorem \ref{thm.attack_influence_a} is in Section \ref{sec.appendix.proof_thm_attack_influence_a}.
    \item The proof of Theorem \ref{cor.attack_influence_a} is in Section \ref{sec.appendix.proof_cor_attack_influence_a}.
    \item The proof of Theorem \ref{prop.influence_impossibility} is in Section \ref{sec.appendix.proof_prop_influence_impossibility}. 
\end{itemize}

\item Appendix \ref{sec.appendix.defense_characterization_general_MDP_uknown_epsilon} contains proofs of the formal results in Section \ref{sec.defense_characterization_general_MDP.unknown_epsilon}.
\begin{itemize}
    \item The proof of Theorem \ref{thm.overestimate_attack_param} is in Section \ref{sec.appendix.thm_prop_overestimate_attack_param}.
    \item The proof of Theorem \ref{thm.underestimate_attack_param} is in Section \ref{sec.appendix.proof_thm_underestimate_attack_param}.
\end{itemize}

\item Appendix \ref{sec.appendix.additional_results_special_mdps} contains
a formal treatment of
the results presented in
Appendix
\ref{sec.defense_characterization_specific_mdp}
\begin{itemize}
    \item Formal results characterizing the defense policy for special MDPs can be found in Section \ref{sec.appendix.special_mdp.defense_characterization}.
    \item Tighter bounds on attack influence for special MDPs can be found in Section \ref{sec.appendix.sepcial.attack.influence}.
    \item Proof of Theorem \ref{thm.bandit_impossibility} (See Appendix \ref{sec.defense_characterization_specific_mdp}) can be found in Section \ref{sec.appendix.thm.bandit_impossibility}.
\end{itemize}
\end{itemize}

\section{Intuition of Results using Special MDPs}\label{sec.defense_characterization_specific_mdp}
In this Appendix, we describe characterization results for special MDPs, in which the agent's actions do not affect the transitions, that is, we assume that
\vspace{-1mm}
\begin{align}\label{eq.condition.no_action_influence}
    P(s, a, s') = P(s, a', s') \quad \forall s, a, a', s'.
\vspace{-1mm}
\end{align}
Variants of the above condition have been studied in the literature (e.g.,  \cite{szepesvari1997asymptotic,dimitrakakis2017multi,sutton2018reinforcement,radanovic2019learning,ghosh2020towards}). Note that this assumption implies that any two policies $\pi$ and $\pi'$ have equal state occupancy measures, so we simplify the notation by denoting $\occstate = \occstate^{\pi} = \occstate^{\pi'}$.

While the results from the previous sections incorporate this special case, we study this setting because: i) the optimal solutions to the defense problem have a simple form, 
enabling us to provide intuitive explanations of our main results from the previous sections, ii) using this setting, we show a tightness result for Theorem \ref{thm.attack_influence_a}, Theorem \ref{cor.attack_influence_a}, and Theorem \ref{thm.overestimate_attack_param}. 

A more formal exposition of our results for this setting inlcuding the proofs can be found in Appendix \ref{sec.appendix.additional_results_special_mdps}.

\begin{figure*}[ht]
	\centering
	\begin{subfigure}{.31\textwidth}
		\resizebox{0.9\linewidth}{!}{	\begin{tikzpicture}
	\begin{axis}[
	axis lines = left,
	xlabel = \scalebox{1.8}{Actions},
	ylabel = {\scalebox{1.8}{$\overline{\text{R}}$}},
	xmin=-1, xmax=7,
	ymin=0, ymax=11,
	ytick={0, 2, 4, 6, 8, 10},
	yticklabels={
\scalebox{1.9}{0},	
\scalebox{1.9}{2},	
\scalebox{1.9}{4},	
\scalebox{1.9}{6},	
\scalebox{1.9}{8},	
\scalebox{1.9}{10},	
},
	xtick={0,1,2,3,4,5, 6},
	xticklabels={
\scalebox{2.1}{$a_1$},
\scalebox{2.1}{$a_2$},
\scalebox{2.1}{$a_3$},
\scalebox{2.1}{$a_4$},
\scalebox{2.1}{$a_5$},
\scalebox{2.1}{$a_6$},
\scalebox{2.1}{$a_7$},
	},
	]
	\addplot[
	only marks,
	color=blue,
	mark=,
	mark size=4pt
	]
coordinates {
	(0, 10.00) (1, 4.00) (2, 6.00) (3, 3.00) (4, 1.00) (5, 5.00) (6, 9.00) };
	\draw[dashed,-, line width=3pt, color=cyan] (axis cs:-1, 10) -- (axis cs:7, 10);
	\end{axis}
	\node[text width=0cm]  at (-0.2, -2.0) (t0) {\scalebox{2.1}{$\optpi$}};
	\node[text width=0cm] at (1.0, -2.1) (t0) {\scalebox{2.1}{$1$}};
	\node[text width=0cm] at (2.05, -2.1) (t0) {\scalebox{2.1}{$0$}};
	\node[text width=0cm] at (3.1, -2.1) (t0) {\scalebox{2.1}{$0$}};
	\node[text width=0cm] at (4.15, -2.1) (t0) {\scalebox{2.1}{$0$}};
	\node[text width=0cm] at (5.20, -2.1) (t0) {\scalebox{2.1}{$0$}};
	\node[text width=0cm] at (6.25, -2.1) (t0) {\scalebox{2.1}{$0$}};
	\node[text width=0cm] at (7.3, -2.1) (t0) {\scalebox{2.1}{$0$}};
	\end{tikzpicture}}
		\caption{$\overline{R}, \optpi$}
	\end{subfigure}
	\quad
	\begin{subfigure}{.31\textwidth}
		\resizebox{0.9\linewidth}{!}{	\begin{tikzpicture}
\begin{axis}[
axis lines = left,
xlabel = \scalebox{1.8}{Actions},
ylabel = {\scalebox{1.8}{$\widehat{\text{R}}$}},
xmin=-1, xmax=7,
ymin=0, ymax=11,
ytick={0, 2, 4, 6, 8, 10},
yticklabels={
	\scalebox{1.9}{0},	
	\scalebox{1.9}{2},	
	\scalebox{1.9}{4},	
	\scalebox{1.9}{6},	
	\scalebox{1.9}{8},	
	\scalebox{1.9}{10},	
},
xtick={0,1,2,3,4,5, 6},
xticklabels={
	\scalebox{2.1}{$a_1$},
	\scalebox{2.1}{$a_2$},
	\scalebox{2.1}{$a_3$},
	\scalebox{2.1}{$a_4$},
	\scalebox{2.1}{$a_5$},
	\scalebox{2.1}{$a_6$},
	\scalebox{2.1}{$a_7$},
},
]
\addplot[
only marks,
color=blue,
mark=,
mark size=4pt
]
coordinates {
	(0, 7.00) (1, 4.00) (2, 6.00) (3, 8.00) (4, 1.00) (5, 5.00) (6, 7.00) };

\draw[dashed,-, line width=3pt, color=cyan] (axis cs:-1, 8) -- (axis cs:7, 8);
\addplot[
only marks,
color=red,
mark=o,
mark size=4pt,
]
coordinates {
	(0, 10.00) (3, 3.00) (6, 9.00) };
\draw[dashed, -{Triangle[width=12,length=7]}, color=red, line width=4pt, opacity=0.9] (axis cs:0, 9.70) -- (axis cs:0, 7.30);
\draw[dashed, -{Triangle[width=12,length=7]}, color=red, line width=4pt, opacity=0.9] (axis cs:3, 3.30) -- (axis cs:3, 7.70);
\draw[dashed, -{Triangle[width=12,length=7]}, color=red, line width=4pt, opacity=0.9] (axis cs:6, 8.70) -- (axis cs:6, 7.30);

\end{axis}
\node[text width=0cm]  at (-0.2, -2.2) (t0) {\scalebox{2.1}{$\targetpi$}};
\node[text width=0cm] at (1.0, -2.1) (t0) {\scalebox{2.1}{$0$}};
\node[text width=0cm] at (2.05, -2.1) (t0) {\scalebox{2.1}{$0$}};
\node[text width=0cm] at (3.1, -2.1) (t0) {\scalebox{2.1}{$0$}};
\node[text width=0cm] at (4.15, -2.1) (t0) {\scalebox{2.1}{$1$}};
\node[text width=0cm] at (5.20, -2.1) (t0) {\scalebox{2.1}{$0$}};
\node[text width=0cm] at (6.25, -2.1) (t0) {\scalebox{2.1}{$0$}};
\node[text width=0cm] at (7.3, -2.1) (t0) {\scalebox{2.1}{$0$}};

\end{tikzpicture}}
		\caption{$\widehat{R}, \targetpi$}
	\end{subfigure}
	\quad 
	\begin{subfigure}{.31\textwidth}
		\resizebox{0.9\linewidth}{!}{	\begin{tikzpicture}
\begin{axis}[
axis lines = left,
xlabel = \scalebox{1.8}{Actions},
ylabel = {\scalebox{1.8}{$\widehat{\text{R}}$}},
xmin=-1, xmax=7,
ymin=0, ymax=11,
ytick={0, 2, 4, 6, 8, 10},
yticklabels={
	\scalebox{1.9}{0},	
	\scalebox{1.9}{2},	
	\scalebox{1.9}{4},	
	\scalebox{1.9}{6},	
	\scalebox{1.9}{8},	
	\scalebox{1.9}{10},	
},
xtick={0,1,2,3,4,5, 6},
xticklabels={
	\scalebox{2.1}{$a_1$},
	\scalebox{2.1}{$a_2$},
	\scalebox{2.1}{$a_3$},
	\scalebox{2.1}{$a_4$},
	\scalebox{2.1}{$a_5$},
	\scalebox{2.1}{$a_6$},
	\scalebox{2.1}{$a_7$},
},
]
\addplot[
only marks,
color=blue,
mark=,
mark size=4pt
]
coordinates {
	(0, 7.00) (1, 4.00) (2, 6.00) (3, 8.00) (4, 1.00) (5, 5.00) (6, 7.00) };
\draw[dashed,-, line width=3pt, color=cyan] (axis cs:-1, 7) -- (axis cs:7, 7);
\addplot[
only marks,
color=red,
mark=o,
mark size=4pt,
]
coordinates {
	(0, 10.00) (3, 3.00) (6, 9.00) };
\draw[dashed, -{Triangle[width=12,length=7]}, color=red, line width=4pt, opacity=0.9] (axis cs:0, 9.70) -- (axis cs:0, 7.30);
\draw[dashed, -{Triangle[width=12,length=7]}, color=red, line width=4pt, opacity=0.9] (axis cs:3, 3.30) -- (axis cs:3, 7.70);
\draw[dashed, -{Triangle[width=12,length=7]}, color=red, line width=4pt, opacity=0.9] (axis cs:6, 8.70) -- (axis cs:6, 7.30);

\end{axis}
\node[text width=0cm]  at (-0.2, -2.1) (t0) {\scalebox{2.1}{$\defensepi$}};
\node[text width=0cm] at (1.0, -2.1) (t0) {\scalebox{2.1}{$\frac 13$}};
\node[text width=0cm] at (2.05, -2.1) (t0) {\scalebox{2.1}{$0$}};
\node[text width=0cm] at (3.1, -2.1) (t0) {\scalebox{2.1}{$0$}};
\node[text width=0cm] at (4.15, -2.1) (t0) {\scalebox{2.1}{$\frac 13$}};
\node[text width=0cm] at (5.20, -2.1) (t0) {\scalebox{2.1}{$0$}};
\node[text width=0cm] at (6.25, -2.1) (t0) {\scalebox{2.1}{$0$}};
\node[text width=0cm] at (7.3, -2.1) (t0) {\scalebox{2.1}{$\frac 13$}};
\end{tikzpicture}}
		\caption{$\widehat{R}, \defensepi$}
		\label{fig.example.bandit.c}
	\end{subfigure}
	\caption{A 
	single-state environment 
	environment with 7 actions. 
	In each figure, the denoted policy is uniform
	over actions on or above the dashed line.
	\textbf{(a)} shows $\overline{R}$ and $\optpi$. Here, the optimal policy selects  action $a_1$.
	\textbf{(b)} shows $\widehat R$ and target policy $\targetpi$ with $\epsA=1$. Here, the target policy selects action $a_4$. 
	%
	%
	\textbf{(c)} shows $\widehat R$ and $\defensepi$ with $\epsD=2$.
	The defense strategy only sees poisoned rewards $\widehat R$, so it first calculates the optimal action and the set of all second best actions under $\widehat R$, in this case $\{a_1, a_7\}$, which then form the set $\tightsetstate=\{a_1,a_7\}$. To obtain defense policy $\defensepi$, we can solve the optimization problem
\eqref{prob.defense_optimization.known_epsilon_b}, which implies that $\defensepi$ should select an action uniformly at random from the set $\{\targetpi(s)\} \cup \tightsetstate =  \{a_1, a_4, a_7 \}$. 
}\label{fig.example.bandit}
\end{figure*}

\subsection{Optimal Defense Policy}\label{sec.defense_characterization_specific_mdp.optimal_policy} 

In this subsection, we provide the intuition behind defense policies for the unknown parameter setting with $\epsD \ge \epsA$ (Section \ref{sec.defense_characterization_general_MDP.unknown_epsilon.overestimating}). 
The key point about
 the assumption in Equation \eqref{eq.condition.no_action_influence} is that it allows us to consider each state separately in
 the defense optimization problems.  
 In particular, it can be shown that the optimization problem \eqref{prob.defense_optimization.known_epsilon} is equivalent to solving $|S|$ optimization problems of the form
\begin{align*}
    &\max_{\pi(\cdot|s) ~\in~ \mathcal P(A)} \vecdot{\pi(\cdot|s)}{\widehat R(s, \cdot)}
    \label{prob.defense_optimization.known_epsilon_b}
	\tag{P3b}\\
    &\quad\pi(a|s)\ge 
    \pi\big(\targetpi(s)~|~s\big)
    \quad \forall a \in \tightsetstate,
\end{align*}
where $\tightsetstate = \{ a :  \widehat{R}(s,a) - \widehat{R}(s, \targetpi(s)) = -\frac{\epsilon}{\occstate(s)}\}$.
If we instantiate Theorem
\ref{thm.overestimate_attack_param} for special MDPs by 
putting $\epsilon = \min\{\epsD, \widehat \epsilon \}$, 
the set
$\tightsetstate$ has
an intuitive description: it is
the set of all ``second-best'' actions (w.r.t $\widehat R$)
in state $s$ such that
their poisoned reward is greater than or equal to
$\widehat{R}(\targetpi(s))-\frac{\epsilon}{\occstate(s)}$.
 It turns out that the defense policy for state $s$ selects an action uniformly at random from the set $\tightsetstate \cup \{ \targetpi(s)\}$. In other words, the defense policy $\defensepi$ is given by: 
 \begin{align*}
     \defensepi(a|s) = \begin{cases}
     \frac{1}{|\tightsetstate| + 1} &\mbox{ if }  a \in \tightsetstate \cup \{ \targetpi(s)\}\\
     0 & \mbox{ otherwise }
     \end{cases}.
 \end{align*}
 To see why, note that the objective in  \eqref{prob.defense_optimization.known_epsilon_b} only improves as we put more probability on selecting $\targetpi(s)$ (since $\targetpi(s)$ is optimal under $\widehat R$). However, the constraints in \eqref{prob.defense_optimization.known_epsilon_b} require that the selection probability of any action in $\tightsetstate$ has to be at least as high as the selection probability of $\targetpi(s)$, which in turn give us the uniform at random selection rule. 
 Figure
\ref{fig.example.bandit}
illustrates attack and defense policies for special MDPs 
using a single-state MDP with action set $\{a_1,...a_7\}$.

\subsection{Attack Influence}\label{sec.defense_characterization_specific_mdp.attack_influence}

We can further inspect the attack influence. For this we will consider the known parameter setting. Following the arguments provided in Section \ref{sec.defense_characterization_general_MDP.known_epsilon.attack_influence}, one obtains that the factor which multiplies $\influence^{\targetpi}$ in Equation \eqref{eq.thm_attack_influence} is in this case equal to $\frac{1}{2}$. This can be seen from Theorem \ref{cor.attack_influence_a}, by using the fact that the assumption in Equation \eqref{eq.condition.no_action_influence} implies $\beta^{\occstate} = 0$. Interestingly, under
the assumption in Equation \eqref{eq.condition.no_action_influence}, the difference $\widehat \influence = \widehat \rho^{\targetpi} - \widehat \rho^{\defensepi}$ is at most $\epsA \cdot |S|$. 
More concretely, 
we can show that the attack influence of $\defensepi$ is bounded by
\begin{align*}
    \influence^{\defensepi} \le \max \left\{ |S| \cdot \epsA, \frac{1}{2} \cdot \influence^{\targetpi} + \frac{2\cdot|S| - 1}{2} \cdot \epsA \right\}.
\end{align*}
Notice that the bound now depends on the attack strategy only through $\influence^{\targetpi}$, whereas in Section \ref{sec.defense_characterization_general_MDP.known_epsilon} it also included the dependency on $\widehat \influence$. A natural question is whether the $\frac 12$ factor that multiplies $\influence^{\targetpi}$
could be improved by using alternative defense strategies. 
Our next result shows that this is not the case.
\begin{theorem}\label{thm.bandit_impossibility}
    Let $\delta > 0$ and $\epsA > 0$. 
    There exists a problem instance with poisoned reward function $\widehat{R}$ 
    and a target policy $\targetpi$ such that for all defense policies $\pi_{\widetilde{\mathcal{D}}} \in \Pi$ and constants $C$,
    we can find $\overline{R}$ that satisfies 
    $\widehat{R} = \attack(\overline{R}, \targetpi, \epsA)$ with the following lower bound on attack influence $\influence^{\pi_{\widetilde{\mathcal{D}}}}$:
     \begin{align*}
         \influence^{\pi_{\widetilde{\mathcal{D}}}} \ge
         \frac{1}{2+\delta}\cdot \influence^{\targetpi} + C.
     \end{align*}
\end{theorem}
In order to prove this results, we construct an MDP which satisfies the assumption in Equation \eqref{eq.condition.no_action_influence}, and we show that no defense policy $\pi_{\widetilde \defense}$
has attack influence which is much better 
than half of the attack influence of $\targetpi$ plus a constant. 
Note that this impossibility results applies to the general cases studied in Section \ref{sec.defense_characterization_general_MDP.known_epsilon} and Section \ref{sec.defense_characterization_general_MDP.unknown_epsilon}, and shows that the bounds in Theorem \ref{thm.attack_influence_a}, Theorem \ref{cor.attack_influence_a}, and Theorem \ref{thm.overestimate_attack_param}, are tight in special MDPs settings.

\section{Additional Details Regarding Experiments}\label{sec.appendix.experiments}

In this section we provide additional details regarding the experiments, focusing on the running times of the attack and defense optimization problems. Since the optimization problems \eqref{prob.attack}, \eqref{prob.defense_a} and \eqref{prob.defense_b} are convex, we use CVXPY to calculate their solutions. 

Following prior work~\cite{rakhsha2020policy-jmlr}, to test the running times, we use the chain environment from Figure \ref{fig.example}, but with different number of states (additional states  are added between $s_2$ and $s_3$, and the corresponding transitions and rewards are defined analogously to those for $s_2$). The attack and defense parameters are set to $\epsA=0.1$ and $\epsD=0.2$.
Table \ref{tab:runtime} shows the average running times (across $10$ runs) of the attack optimization problem \eqref{prob.attack.neighbor} and the defense optimization problem \eqref{prob.defense_b} for different sizes of the chain environment. 

It should be noted that the attack and defense optimization problems are similar in size,
both solve a problem with at most $|S|\cdot(|A| - 1)$ constraints on $\R^{|S|.|A|}$. However, solving the defense problem takes more time, partly because 
$\targetpi$, $\widehat{\epsilon}$ and $\tightset$ need to be identified before \eqref{prob.defense_optimization.known_epsilon} can be solved.

The machine used for obtaining these results is a Macbook Pro personal computer with 4 Gigabytes of memory and a 2.4 GHz Intel Core i5 processor.
\begin{table}[ht]
    \centering
    \renewcommand{\arraystretch}{1.5}
\begin{tabular}{c|c|c}
\diagbox{|S|}{Problem}& Attack& Defense\\\hline
4& $0.01\text{s} \pm 0.5\text{ms}$& $0.05\text{s} \pm 1.6\text{ms}$\\\hline
10& $0.01\text{s} \pm 0.2\text{ms}$& $0.09\text{s} \pm 1.5\text{ms}$\\\hline
20& $0.01\text{s} \pm 0.1\text{ms}$& $0.17\text{s} \pm 4.8\text{ms}$\\\hline
30& $0.02\text{s} \pm 2.0\text{ms}$& $0.27\text{s} \pm 9.7\text{ms}$\\\hline
50& $0.04\text{s} \pm 6.8\text{ms}$& $0.56\text{s} \pm 34.6\text{ms}$\\\hline
70& $0.07\text{s} \pm 3.0\text{ms}$& $1.02\text{s} \pm 69.7\text{ms}$\\\hline
100& $0.13\text{s} \pm 5.4\text{ms}$& $1.83\text{s} \pm 91.2\text{ms}$\end{tabular}

     \caption{Run time of the attack and defense optimization problems for the chain environment with varied number of states $|S|$. Reported numbers are average of 10 runs; standard error is shown with $\pm$.}\label{tab:runtime}
\end{table}


\section{Background and Additional MDP Properties}\label{sec.appendix.background}

In this section we briefly outline the background and MDP properties that we utilize in our proofs. 

\subsection{Reward Poisoning Attacks}
In this section, we provide some background on the cost-efficient reward poisoning attacks, focusing on the results from \cite{rakhsha2020policy-jmlr}.

The setting studied in \cite{rakhsha2020policy-jmlr} incorporates both the average and the discounted reward optimality criteria in a discrete-time Markov Decision Process (MDP), with finite state and action spaces. Our MDP setting is equivalent to their MDP setting under the discounted reward optimality criteria. This criteria can be specified by score $\score$. As defined in the main text, {\em score} $\score^\pi$ of policy $\pi$ is the total expected return scaled by factor $1-\gamma$:
\begin{align*}
    \score^{\pi} = \expct{ 
     (1-\gamma)\sum_{t=1}^{\infty} \gamma^{t-1} R(s_t, a_t) | \pi, \sigma},
\end{align*}
where the state $s_1$ is sampled from the initial state distribution $\sigma$, and subsequent states $s_t$ are obtained by executing policy $\pi$ in the MDP. Actions $a_t$ are sampled from policy $\pi$.

As explained in the main text, the following result is important for our analysis, since it allows us to simplify the optimization problem \eqref{prob.attack} into the optimization problem \eqref{prob.attack.neighbor}.
\begin{lemma}\label{lm.rakhsha}(Lemma 1 in \cite{rakhsha2020policy-jmlr})
    The score of a policy $\targetpi$ is at least $\epsA$ greater than all
    other
    deterministic policies if and only if
    its score is at least $\epsA$ greater than the score of any policy $\neighbor{\targetpi}{s}{a}$. In other words,
    \begin{align*}
        \Big(
        \forall \pi \in \Pi^{\text{det}}\backslash 
        \{ \pi^{\dagger} \}
        :
        \score^{\targetpi}\ge \score^{\pi}+\epsA
        \Big)
        \iff 
        \Big(
        \forall s, a\ne \targetpi(s):
        \score^{\targetpi}\ge \score^{\neighbor{\targetpi}{s}{a}}+\epsA
        \Big).
    \end{align*}
\end{lemma}
\begin{remark}\label{remark.only.neighbor}
    As explained in \cite{rakhsha2020policy-jmlr}, this lemma
    implies that the optimization problem 
    \eqref{prob.attack} is equivalent to \eqref{prob.attack.neighbor}.
    Furthermore,
    the optimization problem 
    is always feasible since
    any policy can be made
    optimal with sufficient perturbation
    of the reward
    function
    as formally shown by
    \cite{rakhsha2020policy-jmlr}
    and
    \cite{ma2019policy}.
\end{remark}
\label{lm.feasible}

\subsection{Overview of Important Quantities} 

Next, we provide an overview of standard MDP quantities and the quantities introduced in the main text that are important for our analysis. 

In addition to score $\score$,  
we consider state-action value function, or $Q$-value function, defined as 
\begin{gather*}
    Q^{\pi}(s,a) = 
    \expct{ \sum_{t=1}^{\infty}
    \gamma^{t-1}R(s_t, a_t)|s_1 =s, a_1=a, \pi
    }.
\end{gather*}
In other words, $Q^{\pi}(s,a)$ is the total expected return when 
the first state is $s$, the first action is $a$, 
while subsequent states $s_t$ and actions $a_t$ are obtained by executing policy $\pi$ in the MDP.

We consider two occupancy measures. 
By $\occupancy^{\pi}$ we denote the state-action occupancy measure in the Markov chain induced by policy $\pi$: 
\begin{align*}
   \occupancy^{\pi}(s, a) = \expct{(1-\gamma) \sum_{t=1}^{\infty} \gamma^{t-1} \ind{s_t = s, 
    a_t = a} | \pi, \sigma}.
\end{align*}
Given MDP $M$, the set of realizable occupancy measures under any (stochastic) policy $\pi \in \Pi$ is denoted by $\Occupancy$. 
Note that the following holds: 
\begin{align}\label{eq.score_occupancy_relation}
    \score^{\pi} = \vecdot{\occupancy^{\pi}}{R},
\end{align}
where $\vecdot{.}{.}$ in the above equation computes a dot product between two vectors of size $|S| \cdot |A|$ (i.e., two vectors in $\R^{|S|\cdot |A|}$).
 We also denote by $\occstate^{\pi}$ the state occupancy measure in the Markov chain induced policy $\pi \in \Pi$, i.e.:
\begin{align*}
    \occstate^{\pi}(s) = \expct{ (1-\gamma)\sum_{t=1}^{\infty} \gamma^{t-1}  \ind{s_t = s} | \pi, \sigma}.
\end{align*}
Note that
\begin{align*}
    \sum_{s,a}\occupancy^\pi(s,a) = \sum_{s} \occstate^\pi(s) = 1.
\end{align*}
State-action occupancy measure and state occupancy measure satisfy
\begin{align}\label{eq.occupancy_measures_relation_a}
    \occupancy^{\pi}(s,a) = \occstate^{\pi}(s) \cdot \pi(a|s),
\end{align}
which for deterministic $\pi$ is equivalent to
\begin{align}\label{eq.occupancy_measures_relation_b}
    \occupancy^{\pi}(s, a) = \ind{\pi(s) = a} \cdot \occstate^{\pi}(s).
\end{align}
Apart from the standard MDP quantities mentioned above, we also mention quantities introduced in the main text.  We denote by $\tightset$ state-action pairs $(s, a)$ for which the margin between $\widehat \score^{\targetpi}$ and $\widehat \score^{\neighbor{\targetpi}{s}{a}}$ is equal to $\epsilon$, i.e.:
\begin{align}
\tightset = \left \{ (s , a) : \widehat \score^{\neighbor{\targetpi}{s}{a}} - \widehat \score^{\targetpi} =  -\epsilon \right \},
\label{eq.tightset.definition}
\end{align}
which can be expressed through reward function $\widehat R$ using state-action occupancy measures $\occupancy$:
\begin{align*}
\tightset = \left \{ (s , a) : \vecdot{\occupancy^{\neighbor{\targetpi}{s}{a}} - \occupancy^{\targetpi}}{\widehat R} =  -\epsilon \right \}.
\end{align*}
Finally, quantity $\alignment(\pi)$ measures how well the occupancy measure of $\pi$ is aligned with $\occupancy^{\neighbor{\targetpi}{s}{a}}$ relative to $\occupancy^{\targetpi}$:
\begin{align}\label{eq.alignment}
   \alignment(\pi) = \vecdot{\occupancy^{\neighbor{\targetpi}{s}{a}}  - \occupancy^{\targetpi}}{\occupancy^{\pi}}.
\end{align}

\subsection{Relation Between Scores and $Q$ values}

Our analysis of Theorem \ref{cor.attack_influence_a} is based on relating score $\rho$ to occupancy measure $\occstate$.  
To do so, we utilize the following lemma, which is a well known result that relates state-action values ($Q$ values) to score $\score$.
\begin{lemma}\label{lm.score_diff_q_value}
    (Equation (2) in \cite{schulman2015trust})
    For any two deterministic policies $\pi, \pi'$ we have
    \begin{align}
     \score^{\pi'}-
        \score^{\pi}
        =\sum_{s\in S} \occstate^{\pi'}(s)
        \big(
        Q^{\pi}(s, \pi'(s))-
            Q^{\pi}(s, \pi(s)) 
        \big).
    \end{align}
\end{lemma}

\subsection{Occupancy Measures as Linear Constraints}\label{sec.appendix.background.linear.constraint.occupancy}

In this subsection, we introduce the Bellman flow linear constraints that characterize $\occupancy^{\pi}$ and $\occstate^{\pi}$. In order to characterize
$\occupancy^{\pi}$, we require the following constraints:
\begin{gather}
    \forall  s: \sum_{ a} \occupancy( s, a) = (1-\gamma)\sigma( s)
		+ \sum_{\tilde s,\tilde a} \gamma \cdot P(\tilde s, \tilde a, s) \cdot \occupancy(\tilde s,\tilde a).\label{eq.bellman.occupancy.1}\\
	\forall (s,a):
	\occupancy(s,a)\ge 0
	\label{eq.bellman.occupancy.2}.
\end{gather}
The importance of these constraints is reflected in the following lemma.
\begin{lemma}\label{lm.syed} (Theorem 2 in \cite{syed2008apprenticeship})
Let $\occupancy$ be a vector that 
satisfies the Bellman flow constraints
\eqref{eq.bellman.occupancy.1} and
\eqref{eq.bellman.occupancy.2}.
Define policy $\pi$
 as
 \begin{align}
     \pi(a|s) = \frac{\occupancy(s, a)}{\sum_{\tilde a}\occupancy(s, \tilde a)}
     \label{eq.syed.policy.def}.
 \end{align}
 Then
 $\occupancy$ is the state-action occupancy measure of $\pi$, in other words
 $\occupancy=\occupancy^\pi$.
 Conversely, 
 if $\pi\in \Pi$ is a policy with
 state-action
 occupancy measure $\occupancy$ (i.e, $\occupancy=\occupancy^{\pi}$) then
 $\occupancy$ satisfies the Bellman flow constraints 
 \eqref{eq.bellman.occupancy.1} and
 \eqref{eq.bellman.occupancy.2},
 as well as Equation
 \eqref{eq.syed.policy.def}.
\end{lemma}

By characterizing the condition
$\occupancy \in \Occupancy$
as linear constraints, namely 
\eqref{eq.bellman.occupancy.1}
and
\eqref{eq.bellman.occupancy.2}
,
the optimization problem
\eqref{prob.defense_optimization.known_epsilon}
becomes a linear program.
Furthermore given the one-to-one correspondence between policies and
occupancy measures, we can work with the latter instead of the former.

As for $\occstate^{\pi}$,
it is well-known (e.g., see \cite{rakhsha2020policy-jmlr})
 that
 a vector
 $\occstate$ is the state occupancy measure
 for policy $\pi$ (i.e., $\occstate=\occstate^{\pi}$), if and only if
 \begin{align}
     \occstate(s) = 
     (1-\gamma)\sigma(s) + 
     \gamma
     \sum_{\tilde s, \tilde a}
     \occstate(\tilde s)\pi(\tilde a|\tilde s)
     P(\tilde s, \tilde a, s)
     \label{eq.bellman.occstate}.
 \end{align}

\subsection{Additional MDP Properties}
\label{sec.appendix.additional_properties}

In this subsection we
state and prove a lemma that we need for
Theorem \ref{cor.attack_influence_a}, that is, Lemma \ref{lemma.fraction.bound}. This lemma compares the state occupancy measures of different policies: $\targetpi$, $\neighbor{\targetpi}{s}{a}$, and a policy $\pi$.  

Now, note that Lemma \ref{lm.score_diff_q_value}
provides an instructive way of comparing the scores 
of different policies. We will therefore try to utilize Lemma \ref{lm.score_diff_q_value} by considering specially designed reward functions (which are vectors of size $|S| \cdot |A|$, i.e., vectors in $\R^{|S|\cdot |A|}$). Note that these reward functions do not play any role in our attack and defense optimization problems; we only introduce  them for our proof technique in this section.
For example, if we set the reward function as
\begin{gather*}
    R(s, a) = \ind{s=s_0},
\end{gather*}
where $s_0$ is an arbitrary state,
then the score of a policy 
$\pi$ will be equal to $\score^{\pi} = \vecdot{\occupancy^{\pi}}{R} = \occstate^\pi(s_0)$.
Since changing the reward function
does not affect $\occstate$, 
this gives us a tool to relate
$\occstate$ to $\score$.

These insights are reflected in the following lemma.
\begin{lemma}\label{lemma.mu.sup.R}
	Let $\mathcal{R}$ be the set of all reward vectors $R$ such that
	\begin{gather}
		\forall(s,a): \sum_{s} |R(s, \targetpi(s))| \le 1 , \quad
		\forall s, a, \tilde a: 
		 R(s,a) =  R(s,\tilde a),
	\end{gather}
	then
	\begin{gather}
		||\occstate^{\pi}-\occstate^{\targetpi}||_{\infty} = 
		\sup_{R\in \mathcal{R}} \Big(
			\score^\pi( R) - \score^\targetpi(R)
		\Big),
	\end{gather}
	where
	$
	    \score^{\pi}(R)
	    =\vecdot{R}{\occupancy^{\pi}}
	    $
	    as in 
	    \eqref{eq.score_occupancy_relation}
	and we use the
	notation $\score^{\pi}(R)$
	to make the dependence on $R$
	explicit.
\end{lemma}
\begin{proof}
    Assume that
    $R\in \mathcal{R}$
    and
	denote by $r$
	the vector in $\R^{|S|}$ with entries $r(s)=R(s,\targetpi(s))$.
	Since $R(s,a)=R(s, \tilde a)$,
	\begin{align*}
	    \score^\pi(R)&=
	    \sum_{s,a}\occupancy^\pi(s,a)R(s,a) \\&=
	    \sum_{s}\occstate^\pi(s)\sum_{a}\pi(a|s)R(s,a)\\&=
	    \sum_{s}\occstate^\pi(s)\sum_{a}\pi(a|s)r(s)\\&=
	    \sum_{s}\occstate^\pi(s)r(s).
	\end{align*}
	Which implies
	\begin{align*}
	  \vecdot{r}{
	  \occstate^{\pi}
	  -\occstate^{\targetpi}
	  }
	  =\score^{\pi}(R)-\score^{\targetpi}(R).
	\end{align*}
	Since the constraint
	$\sum_{s}|R(s, \targetpi(s))|\le 1$ is equivalent to
	$||r||_1\le 1$,
	the claim follows directly from the fact that $||.||_\infty$ and $||.||_1$ are dual norms.
\end{proof}
\begin{remark}
\label{remark.sup.attainable}
	The $\sup$ in the above lemma can be changed to $\max$
	since the set $\mathcal{R}$
	is compact (it is clearly bounded
	and it is closed since it is the intersection of a closed ball with 
	closed subspaces)
	and the function $R \to \score^\pi(R)$
	is continuous.
\end{remark}
\begin{lemma}\label{lemma.fraction.bound}
	For all policies $\pi, \targetpi$, it holds that
\begin{gather*}
	\frac{
		||\occstate^{\pi}-\occstate^{\targetpi}||_{\infty}
	}{
		\max_{s,a}	||\occstate^{\neighbor{\targetpi}{s}{a}}-\occstate^{\targetpi}||_{\infty}
}\le 
\frac{1}{
\occstate_{\min}
},
\end{gather*}
where
\begin{gather*}
	\occstate_{\min} = \min_{s, \tilde{\pi}}\occstate^{\tilde \pi} (s).
\end{gather*}
\end{lemma}
\begin{proof}
	Given Remark
	\ref{remark.sup.attainable}, we can set the reward vector
	$R\in\mathcal{R}$ to be the vector such that
	\begin{gather*}
		||\occstate^{\pi}-\occstate^{\targetpi}||_{\infty} = 
		\Big(
		\score^\pi(R) - \score^\targetpi(R)
		\Big).
	\end{gather*}
	Now, note that given Lemma \ref{lm.score_diff_q_value} (with
	$(\pi', \pi)=(\pi, \targetpi)$), we have
	\begin{align*}
		\score^\pi(R) - \score^\targetpi(R) &= 
		\sum_{s\in S} \occstate^{\pi}(s)
		\big(
		Q^{\targetpi}(s, \pi(s)) - Q^{\targetpi}(s, \targetpi(s))
		\big) \\&\le 
		\max_{s}
		\Big(
		Q^{\targetpi}(s, \pi(s)) - Q^{\targetpi}(s, \targetpi(s))
		\Big).
	\end{align*}
	Furthermore, Lemma \ref{lm.score_diff_q_value} with
	$
	(\pi', \pi)=(\neighbor{\targetpi}{s}{\pi(s)}, \targetpi)$ implies 
	\begin{gather*}
	\score^{\neighbor{\targetpi}{s}{\pi(s)}}(R) - \score^{\targetpi}(R) = 
		\occstate^{\neighbor{\targetpi}{s}{\pi(s)}}(s)
		\big(
		Q^{\targetpi}(s, \pi(s)) - Q^{\targetpi}(s, \targetpi(s))
		\big).
	\end{gather*}
However 
 \begin{align*}
			\score^{\neighbor{\targetpi}{s}{\pi(s)}}(R) - \score^{\targetpi}(R) &\le 
			\sup_{\tilde R \in \mathcal{R}}
			\Big(
			\score^{\neighbor{\targetpi}{s}{\pi(s)}}(\tilde R) - \score^{\targetpi}( \tilde R)
			\Big)
			 \\&\overset{(i)}{\le}
			 \max_{s,a}	||\occstate^{\neighbor{\targetpi}{s}{a}}-\occstate^{\targetpi}||_{\infty},
	\end{align*}
	 where
	 we used 
	 Lemma \ref{lemma.mu.sup.R}
	 to obtain
	 $(i)$.
	 Since $\occstate^{\neighbor{\targetpi}{s}{\pi(s)}}(s) \ge \occstate_{\min}$,
	we conclude that
	\begin{align*}
	    ||\occstate^{\pi}-\occstate^{\targetpi}||_{\infty} &= 
		\Big(
		\score^\pi(R) - \score^\targetpi(R)
		\Big)
		\\&\le 
		\max_{s}
		\Big(
		Q^{\targetpi}(s, \pi(s)) - Q^{\targetpi}(s, \targetpi(s))
		\Big)
		\\&= 
		\max_{s}
		\Big(
		\frac{\score^{\neighbor{\targetpi}{s}{\pi(s)}}(R) - \score^{\targetpi}(R)}{\occstate^{\neighbor{\targetpi}{s}{\pi(s)}}(s)}
		\Big)
		\\&\le \frac{
		\max_{s,a}	||\occstate^{\neighbor{\targetpi}{s}{a}}-\occstate^{\targetpi}||_{\infty}
		}{\occstate_{\min}}.
	\end{align*}
	Rearranging the above proves the claim.
\end{proof}


\section{Attack Characterization Results}\label{sec.appendix.attack.characterization}

In this section we provide characterization results for the attack optimization problem \eqref{prob.attack}, which we use for proving the formal results presented in Section \ref{sec.defense_characterization_general_MDP.known_epsilon} and Section
\ref{sec.defense_characterization_general_MDP.unknown_epsilon}.
 In particular, the main result of this appendix is a set of Karush–Kuhn–Tucker (KKT) conditions that  characterize the solution to the optimization problem \eqref{prob.attack}. 

To compactly express the KKT characterization results, let us introduce state occupancy difference matrix $\occdiffmatrix\in \mathbb{R}^{|S|\cdot(|A|-1)\times |S|\cdot|A|}$ as a matrix with rows consisting of the vectors $\occupancy^{\neighbor{\targetpi}{s}{a}} - \occupancy^{\targetpi}$
for all neighboring policies
$\neighbor{\targetpi}{s}{a}$.
Additionally, for all $s, a \ne \targetpi(s)$, we use $\occdiffmatrix(s,a)$ to denote the transpose of the row of $\occdiffmatrix$ corresponding to $(s,a)$. Note that $\occdiffmatrix(s,a)$ is a column vector.
In this notation, given Remark \ref{remark.only.neighbor}
and Equation 
\eqref{eq.score_occupancy_relation}
, the optimization problem \eqref{prob.attack} is equivalent to 
\begin{align}
\label{prob.attack_compact}
\tag{P1"}
	&\min_{R} \frac 12 \norm{R-R'}_2^2\\
	\notag
&\quad \mbox{ s.t. } \quad \occdiffmatrix \cdot R \preccurlyeq -\epsA \cdot \mathbf 1, 
\end{align}
where $\mathbf 1$ is a $|S| \cdot (|A| - 1)$ vector whose each element equal to $1$, and $\preccurlyeq$ specifies that the left hand side is element-wise less than or equal to the right hand side.
Given this notation, the following lemma states the KKT conditions for a reward function $R$ (i.e., an $|S|\cdot |A|$ vector) to be an optimal solution to the optimization problem \eqref{prob.attack}.

\begin{lemma}{(KKT characterization)}\label{lm.kkt_conditions}
	$R$ is a solution to the optimization problem \eqref{prob.attack} if and only if
	there exists an $|S|\cdot |A|$ vector $\lambda$ such that
	\begin{align*}
		 (R-R') + \occdiffmatrix^T \cdot \lambda = \mathbf 0 & \quad\quad\text{\em stationarity},\\
		\occdiffmatrix \cdot R + \epsA \cdot \mathbf 1 \preccurlyeq \mathbf 0 & \quad\quad\text{\em primal feasibility},\\
		\lambda \succcurlyeq \mathbf 0 &\quad\quad \text{\em dual feasibility},\\
		\quad\quad\forall(s, a\ne \targetpi(s)): \lambda(s,a)\cdot (\occdiffmatrix(s,a)^T \cdot R+ \epsA) = \mathbf 0 & 
		\quad\quad\text{\em complementary slackness},
	\end{align*}
	where $\mathbf 0$ denotes an $|S| \cdot |A|$ vector whose each element equal to $0$, and likewise, $\mathbf 1$ denotes an $|S| \cdot |A|$ vector whose each element equal to $0$.
\end{lemma}
\begin{proof}
	Since \eqref{prob.attack} is always feasible (Remark \ref{remark.only.neighbor}) and all of the constrains are linear, strong duality holds. Now, the Lagrangian of the optimization problem is equal to
	\begin{align*}
		\mathcal{L} = \frac 12 \norm{R-R'}_2^2 + \lambda^T(\occdiffmatrix \cdot R + \epsA \cdot \mathbf{1}),
	\end{align*}
	and taking the gradient with respect to $R$ gives us
	\begin{align*}
		\nabla_{R}\mathcal{L} = (R-R') + \occdiffmatrix^T \cdot \lambda.
	\end{align*}
	The statement then follows by applying the standard KKT conditions.
\end{proof}
\begin{remark}\label{rem.unique.attack}
(Uniqueness)
    The solution to the optimization problem 
    \eqref{prob.attack}
    is unique since the objective
    $\frac{1}{2}
    ||R-R'||_2^2 
    $ is strongly convex.
\end{remark}


\section{Proofs of Section \ref{sec.defense_characterization_general_MDP.known_epsilon}}\label{sec.appendix.proofs_defense_characterization_general_MDP_known_epsilon}

This section of the appendix contains the proofs of the formal results presented in Section \ref{sec.defense_characterization_general_MDP.known_epsilon}.  

\subsection{Proof of Lemma \ref{lm.tightset}}\label{sec.appendix.proof_lm_tightset}

\textbf{Statement:} {\em Reward function $R$ satisfies $\widehat{R} = \attack(R, \targetpi, \epsA)$ if any only if
\begin{align*}
    R = \widehat{R} + \sum_{(s, a) \in \tightsettarget} \alpha_{s,a} \cdot \left (\occupancy^{\neighbor{\targetpi}{s}{a}} - \occupancy^{\targetpi} \right),
\end{align*}
for some $\alpha_{s, a} \ge 0$.}
\begin{proof}
    To prove the statement, we use Lemma \ref{lm.kkt_conditions}.
    The primal feasibility condition in the lemma always holds as
    $\widehat R \in \attack(\overline{R},
    \targetpi, \epsA)$. 
    Therefore $\widehat{R} \in \attack(R,
    \targetpi, \epsA)$ if and only if there exists $\lambda$ such that the other three conditions hold. Note that the complementary slackness condition is equivalent to
    \begin{gather*}
    \forall (s, a\ne \targetpi(s)): 
        \lambda(s,a)=0 \lor \occdiffmatrix(s,a)^T\cdot R + \epsA =0 \iff
        \forall (s,a) \notin \tightsettarget: \lambda(s,a)=0.
    \end{gather*}
    Therefore from dual feasibility, stationarity and complemantary slackness it follows that  $\widehat R \in \attack(R, \targetpi, \epsA)$ if and only if there exists $\lambda$ such that
    \begin{gather*}
    \lambda \succcurlyeq 0,\\
    R=\widehat{R} + \sum_{(s,a)} \lambda(s,a) \cdot \left (\occupancy^{\neighbor{\targetpi}{s}{a}} - \occupancy^{\targetpi} \right),\\
     \forall (s,a) \notin \tightsettarget: \lambda(s,a)=0.
    \end{gather*}
    The Lemma therefore follows by setting $\alpha_{s,a} =  \lambda(s,a)$ since setting $\lambda(s,a)=0$ for 
    all $(s,a) \notin \tightsettarget$ is equivalent to not summing over the terms corresponding to $(s,a) \notin \tightsettarget$ 
    in the stationarity condition.
\end{proof}
A direct consequence of this lemma is the following result.
\begin{corollary}\label{lm.epsilon.hat}
    Assume that
    $\widehat{R}=\attack(R, \targetpi, \epsA)$
    and
    $\widehat{R}\ne R$.
    It follows that
    \begin{align*}
        \widehat{\epsilon}=\epsA,
    \end{align*}
    where
    \begin{align*}
    \widehat{\epsilon}
    =
    \min_{s,a\ne \targetpi(s)}
    \left[  \hatscore^{\targetpi}- \hatscore^{\neighbor{\targetpi}{s}{a}} \right].
    \end{align*}
\end{corollary}
\begin{proof}
    Assume to the contrary that
    $\widehat{\epsilon}\ne \epsA$.
    Given the primal feasibility
    condition in Lemma \ref{lm.kkt_conditions}, 
    $\widehat{\epsilon}\ge \epsA$. Therefore
    $\widehat{\epsilon} > \epsA$. It follows that
    \begin{align*}
        \forall s, a \ne \targetpi(s):
        \hatscore^{\targetpi}- \hatscore^{\neighbor{\targetpi}{s}{a}} > \epsA \implies
        \tightsettarget=\emptyset.
    \end{align*}
    Given Lemma \ref{lm.tightset}, this implies
    that
    ${R}=\widehat{R}$, which contradicts the
    initial assumption ${R} \ne \widehat{R}$.
\end{proof}

\subsection{Proof of Theorem \ref{thm.defense_optimization.known_epsilon}}\label{sec.appendix.proof_thm_defense_optimization.known_epsilon}
Before proving the theorem we prove some results
that we need for the proof of this theorem, as well as for the results in later sections.
\begin{lemma}\label{lm.wost_case_utility}
     Consider policy $\pi$ with state-action occupancy measure $\occupancy^{\pi}$. Solution $\score^{\pi}_{\min}$ to the following optimization problem:
\begin{align*}
	\label{prob.wost_case_utility}
	\tag{P4}
	&\quad \min_{R} \score^{\pi} \quad \mbox{ s.t. } \quad  \widehat R = \attack(R, \targetpi, \epsA),
\end{align*}
satisfies:
\begin{align*}
    \score^{\pi}_{\min} = \begin{cases}
        \widehat \score^{\pi} &\mbox{ if } \quad \forall s, a \in \tightsettarget: \alignment(\pi) \ge 0  \\
        -\infty &\mbox{ otherwise }
    \end{cases}.
\end{align*} 
\end{lemma}

\begin{proof}
    We separately analyze the two cases: the case when $\alignment(\pi) \ge 0$ for all $(s,a) \in \tightsettarget$ holds, and the case when it does not.
    
    \textbf{Case 1:} If $\alignment(\pi) \ge 0$ for all $(s,a) \in \tightsettarget$, then by using Equation \eqref{eq.score_occupancy_relation} and Lemma \ref{lm.tightset} we obtain that 
	\begin{align*}
		\score^{\pi} - \widehat{\score}^{\pi} = 
		\vecdot{\occupancy^{\pi}}{R - \widehat R} = 
		\sum_{(s,a) \in \tightsettarget} \alpha_{s, a}\cdot \vecdot{\occupancy^{\pi}}{
	\occupancy^{\neighbor{\targetpi}{s}{a}} - \occupancy^{\targetpi}
	} \ge 0.
	\end{align*}
	Therefore, $\score^{\pi} \ge \widehat{\score}^{\pi}$. Furthermore, from Lemma \ref{lm.tightset}, we know that $R = \widehat R$ satisfies the constraint in the optimization problem \eqref{prob.wost_case_utility}, so the score of the optimal solution to \eqref{prob.wost_case_utility} 
	is $\score^{\pi}_{\min} = \widehat{\score}^{\pi}$.
	
	\textbf{Case 2:} Now, consider the case when $\alignment(\pi) < 0$ for a certain state-action pair $(s, a) \in \tightsettarget$. Let $\alpha_{s, a}$ be an arbitrary positive number. From Lemma \ref{lm.tightset}, we know that
\begin{align*}
    R = \widehat R + \alpha_{s, a} \cdot \vecdot{\occupancy^{\pi}}{
	\occupancy^{\neighbor{\targetpi}{s}{a}} - \occupancy^{\targetpi}}
\end{align*}
satisfies the constraint in the optimization problem \eqref{prob.wost_case_utility}, and hence is a solution to \eqref{prob.wost_case_utility}. Moreover, by using this solution together with Equation \eqref{eq.score_occupancy_relation}, we obtain
\begin{gather}
	\score^{\pi} - \widehat{\score}^{\pi} = 
\vecdot{\occupancy^{\pi}}{R - \widehat R} = 	\alpha_{s, a}\cdot \vecdot{\occupancy^{\pi}}{
	\occupancy^{\neighbor{\targetpi}{s}{a}} - \occupancy^{\targetpi}
	}= \alpha_{s, a} \cdot \alignment(\pi).
\end{gather}
Since $\alpha_{s, a}$ can be arbitrarily large and $\alignment(\pi) < 0$, while $\widehat{\score}^{\pi}$ is fixed, $\score^{\pi}$ can be arbitrarily small. Hence, the score of the optimal solution to \eqref{prob.wost_case_utility} is unbounded from below, i.e., $\score^{\pi}_{\min} = -\infty$.
\end{proof}

\begin{lemma}\label{lm.prob_feasibility}
    The optimization problem \eqref{prob.defense_optimization.known_epsilon} is feasible for all values of $\epsilon > 0$.
\end{lemma}\begin{proof}
   Consider a deterministic policy $\pi$ that never agrees with the attacker's policy, in other words
    \begin{gather*}
        \forall s\in S: \pi(s)\ne \targetpi(s).
    \end{gather*}
    Such a policy always exists as $|A| \ge 2$.
    We claim that $\occupancy=\occupancy^{\pi}$ is a feasible solution to the optimization problem. Clearly, $\occupancy^{\pi} \in \Occupancy$ by the definition of $\Occupancy$. Furthermore
    \begin{align*}
      \forall (s,a\ne \targetpi(s)):
       \vecdot{\occupancy^{\neighbor{\targetpi}{s}{a}}-\occupancy^{\targetpi}}{\occupancy^{\pi}} &= 
      \sum_{(\tilde s, \tilde a) }\occupancy^{\neighbor{\targetpi}{s}{a}} (\tilde s, \tilde a) \cdot \occupancy^{\pi}(\tilde s, \tilde a) - \sum_{(\tilde s, \tilde a)}\occupancy^{\targetpi} (\tilde s, \tilde a)\cdot\occupancy^{\pi}(\tilde s, \tilde a)
      \\&\overset{(i)}{=}
      \sum_{(\tilde s, \tilde a)}\occupancy^{\neighbor{\targetpi}{s}{a}} (\tilde s, \tilde a)\cdot \occupancy^{\pi}(\tilde s, \tilde a) \\&\overset{(ii)}{\ge} 0,
    \end{align*}
    where $(i)$ follows from the fact that 
    \begin{gather*}
    \forall s: \targetpi(s) \ne \pi(s) \implies \forall (\tilde s,\tilde a): \occupancy^{\pi}(s,a)\cdot
    \occupancy^{\targetpi}(\tilde s,\tilde a) = 0,
    \end{gather*}
    and $(ii)$ follows from the fact that 
    $\occupancy^{\pi}, \occupancy^{\neighbor{\targetpi}{s}{a}} \ge 0$
\end{proof}

We can now prove Theorem  \ref{thm.defense_optimization.known_epsilon}, that is the following statement.

\textbf{Statement:} {\em
Consider the following optimization problem parameterized by $\epsilon$:
    \begin{align}
    &\max_{\occupancy \in \Occupancy} \vecdot{\occupancy}{\widehat R}
	\tag{P3}\\
	\notag
	&\quad \mbox{ s.t. } \vecdot{\occupancy^{\neighbor{\targetpi}{s}{a}} - \occupancy^{\targetpi}}{\occupancy} \ge 0 \quad \forall s, a \in \tightset,
\end{align}
For $\epsilon = \epsA$, this optimization problem
is always feasible, and its optimal solution $\occupancy_{\max}$ specifies an optimal solution to optimization problem \eqref{prob.defense_a} with
\begin{align*}
    \defensepi(a|s) = \frac{\occupancy_{\max}(s, a)}{\sum_{a'} \occupancy_{\max}(s, a')}.
\end{align*}
The score of $\defensepi(a|s)$ is lower bounded by $\overline \score^{\defensepi} \ge \widehat \score^{\defensepi}$. Furthermore, $\alignment(\defensepi)$ is non-negative, i.e., $\alignment(\defensepi) \ge 0$ for all $(s,a) \in \tightsettarget$.
}

\begin{proof}
    
     The feasibility of the problem follows from Lemma \ref{lm.prob_feasibility}.
     Note that
     $\occupancy_{\max}$ always
     exists since
     \eqref{prob.defense_optimization.known_epsilon} is
     maximizing a continuous function 
     over a closed and bounded set. Concretely,
     the constraints
     $\vecdot{\occupancy^{\neighbor{\targetpi}{s}{a}}-
     \occupancy^{\targetpi}
     }{\occupancy}\ge 0$
     and Equations
     \eqref{eq.bellman.occupancy.1} and \eqref{eq.bellman.occupancy.2} each define closed sets,
     and since $||\occupancy||_1=1$, the set $\Occupancy$ is bounded.
     
     In order to see why
     $\occupancy_{\max}$
     specifies an optimal solution to
     \eqref{prob.defense_a}, 
     note that we can rewrite
     \eqref{prob.defense_a}
     as
     \begin{align*}
         \max_{\pi}
         \score^{\pi}_{\min},
     \end{align*}
     where
     $\score^{\pi}_{\min}$
     is the solution to the optimization problem 
     \eqref{prob.wost_case_utility}.
     Due to Lemma
     \ref{lm.wost_case_utility}, this could be rewritten as
     \begin{align*}
         &\max_{\pi} \hatscore^{\pi}\\
	&\quad \mbox{ s.t. } \alignment(\pi) \ge 0 \quad\forall(s,a) \in \tightsettarget.
	\end{align*}
	Namely, maximizing a function $f(x)$ subject to constraint $x \in \mathcal X$ (where $\mathcal X \ne \emptyset$) is equivalent to maximizing $\tilde f(x)$, where
     \begin{align*}
         \tilde f(x) =
         \begin{cases}
             f(x) \quad\text{if} \quad x \in \mathcal X\\
             -\infty \quad\text{o.w.}
    \end{cases}.
    \end{align*}
	Due to  
	\eqref{eq.score_occupancy_relation} and
	\eqref{eq.alignment}, 
	the constrained optimization problem above can be rewritten as
	\begin{align*}
	    &\max_{\pi} 
	    \vecdot{\occupancy^{\pi}}{\widehat{R}}
	    \\
	&\quad \mbox{ s.t. } 
	\vecdot{\occupancy^{\neighbor{\targetpi}{s}{a}}-\occupancy^{\targetpi}}{\occupancy^{\pi}}\quad\forall(s,a) \in \tightsettarget.
	\end{align*}
	Therefore,
	given Lemma 
	\ref{lm.syed}, 
	$\occupancy_{\max}$
	specifies a solution to
	\eqref{prob.defense_a} via
	\eqref{eq.defense_optimization.known_epsilon}.
    
    Finally, note that the constraints of the optimization problem \eqref{prob.defense_optimization.known_epsilon} ensure that a policy $\pi$ whose occupancy measure is equal to $\occupancy_{\max}$ will have $\alignment(\pi) \ge 0$ --- in other words, $\alignment(\defensepi)$ is non-negative
    for all $(s,a) \in \tightsettarget$.
    Due to Lemma \ref{lm.wost_case_utility}, we know that such policy $\pi$ will have the worst case utility equal to $\widehat \rho^{\pi}$. Therefore, $\overline{\rho}^{\pi} \ge \widehat \rho^{\pi}$.  
\end{proof}
\begin{remark}
Given Lemma \ref{lm.syed}, the constraint $\occupancy \in \Occupancy$ can equivalently be 
replaced with
constraints
\eqref{eq.bellman.occupancy.1} and
\eqref{eq.bellman.occupancy.2}, making the optimization problem \eqref{prob.defense_optimization.known_epsilon} a linear program.
\end{remark}
    
    
\subsection{Proof of Theorem \ref{thm.attack_influence_a}}
\label{sec.appendix.proof_thm_attack_influence_a}
In order to prove Theorem
\ref{thm.attack_influence_a}, we need the following lemma.
\begin{lemma}\label{lm.zeta_ge_0}
    Assume that the condition
    in Equation \eqref{eq.attack_influence_condition} holds.
    If $\zeta$ is defined
    as in Theorem \ref{thm.attack_influence_a}, then $\zeta \ge 0$.
\end{lemma}\begin{proof}
    If $\tightsettarget=\emptyset$, then the claim holds trivially since $\zeta = 0$. Otherwise, let 
    $(s,a)$ be an arbitrary member of $\tightsettarget$.
    Given the definition of
    $\zeta$ and
    Equation \eqref{eq.condition.no_action_influence},
    it suffices to show that there is a deterministic policy
    $\pi_{\zeta}$ such that
    \begin{align*}
        \alignment(\pi_{\zeta}) - \alignment(\defensepi) \ge 0
    \end{align*}
    Note however that if we set the rewards vector to
    $\occupancy^{\neighbor{\targetpi}{s}{a}}-\occupancy^{\targetpi}$\footnote{
    As was the case in the proofs in Section \ref{sec.appendix.additional_properties}, this reward vector is unrelated to our defense strategy and is solely for the purpose of our analysis.
    }, then
    for all policies $\pi$,
    \begin{align*}
        \alignment(\pi) = \score^{\pi}.
    \end{align*}
    Since there is always an optimal policy that is deterministic, there exists $\pi_{\zeta}\in \Pi^{\text{det}}$ such that
    \begin{align*}
        \score^{\pi_{\zeta}}\ge \score^{\defensepi} \implies
        \alignment(\pi_{\zeta})\ge \alignment(\defensepi)
    \end{align*}
    which proves the claim.
\end{proof}

We now prove Theorem 
\ref{thm.attack_influence_a}.

\textbf{Statement:} {\em
 Let $\defensepi$ be the defense policy obtained from the optimization problem \eqref{prob.defense_optimization.known_epsilon} and Equation \eqref{eq.defense_optimization.known_epsilon} with $\epsilon = \epsA$. Furthermore, let us assume that
 the condition in Equation \eqref{eq.attack_influence_condition} holds.
 Then the attack influence $\influence^{\defensepi}$ is bounded by
\begin{align*}
     \influence^{\defensepi} \le \max \{ \widehat \influence, \frac{\zeta}{1+\zeta} \cdot [\influence^{\targetpi} + \epsA]  + [\widehat \influence - \epsA] \},
\end{align*}
where $\widehat \influence = \widehat \rho^{\targetpi} - \widehat \rho^{\defensepi}$. Here, $\zeta = 0$ {\em if} $\tightsettarget = \emptyset$, and $\zeta = \max_{(s,a) \in \tightsettarget, \pi \in \Pi^{\text{det}}} \frac{\alignment(\pi)-\alignment(\defensepi)}{\alignment(\defensepi)-\alignment(\targetpi)}$ {\em if} $\tightsettarget \ne \emptyset$.
    }
    \begin{remark}\label{remark.zeta_infinity}
    Given Lemma 
    \ref{lm.zeta_ge_0}, 
    if $\tightsettarget\ne \emptyset$, 
    there exists a
    deterministic policy $\pi$
    such that
    the numerator
    in
    $\frac{\alignment(\pi) - \alignment(\defensepi)}{\alignment(\defensepi) - \alignment(\targetpi)}$
    is non-negative.
    Therefore
    if
    there exists a state-action pair
    $(s,a)\in \tightsettarget$
    such that
    $\alignment(\defensepi)=\alignment(\targetpi)$, we take $\zeta, \frac{1}{\zeta + 1}$ and $\frac{\zeta}{\zeta + 1}$ to be
    $+\infty, 0$, and $1$ 
    respectively.
    Note that $\frac{1}{\zeta + 1} + \frac{\zeta}{\zeta + 1} = 1$ holds 
    in this case as well.
    \end{remark}
    \begin{proof}[Proof of Theorem
\ref{thm.attack_influence_a}]
    
     Without loss of generality we assume that $\optpi$ is deterministic since there is always a deterministic optimal policy. If there are multiple such policies, we pick one arbitrarily.
     Note that the choice of the optimal policy has no effect on the theorem's statement since $\barscore^{\optpi}$ is the same for all optimal policies.
     We divide the proof into two cases based on whether or not $\optpi=\targetpi$.
     
    \textbf{Case 1:} Assume that $\optpi=\targetpi$. It suffices to prove that
    \begin{align*}
        &\influence^{\defensepi} \le
        \widehat{\influence}
        \\\iff&
        \influence^{\defensepi} \le \hatscore^{\targetpi} - \hatscore^{\defensepi} \\\iff&
        \barscore^{\targetpi} - \barscore^{\defensepi}
        \le \hatscore^{\targetpi} - \hatscore^{\defensepi}\\\iff&
    \vecdot{\overline{R} - \widehat{R}}{\occupancy^{\defensepi}-\occupancy^{\targetpi}}\ge 0,
    \end{align*}
    where we utilized Equation \eqref{eq.score_occupancy_relation}.
    Recall from Lemma \ref{lm.tightset} that
    \begin{align*}
    \overline{R} - \widehat{R} = 
    \sum_{(s,a)\in \tightsettarget} \alpha_{s,a}(\occupancy^{\neighbor{\targetpi}{s}{a} } - \occupancy^{\targetpi}).
    \end{align*}
    Therefore
    \begin{align*}
    \vecdot{\overline{R} - \widehat{R}}{\occupancy^{\defensepi}-\occupancy^{\targetpi}}&=
    \sum_{(s,a)\in \tightsettarget}
    \alpha_{s,a}\vecdot{\occupancy^{\neighbor{\targetpi}{s}{a} } - \occupancy^{\targetpi}}{\occupancy^{\defensepi}-\occupancy^{\targetpi}}
    \\&=
    \sum_{(s,a)\in \tightsettarget}
    \alpha_{s,a}(\alignment(\defensepi) - \alignment(\targetpi))
    \ge 0,
    \end{align*}
    where the inequality follows from the condition in Equation \eqref{eq.attack_influence_condition}, which is assumed to hold.
    
    \textbf{Case 2:} Assume that $\optpi \ne \targetpi$. Note that this implies
    $\widehat{R}\ne \overline{R}$,
    and therefore,
    given Lemma \ref{lm.tightset}, 
    $\tightsettarget \ne \emptyset$.
    
    {\em Part 1:} We first claim that
    \begin{align}
        (\barscore^{\optpi}-\widehat \score^{\optpi}) - (\barscore^{\defensepi}-\widehat \score^{\defensepi}) \le 
    	\frac{\zeta}{1+\zeta} \cdot \left(
    	(\barscore^{\optpi}-\widehat \score^{\optpi}) - (\barscore^{\targetpi}-\widehat \score^{\targetpi})
    	\right).
    	\label{eq.proof.attack_influence_a.1}
    \end{align}
    If there exists
    $(s,a)\in \tightsettarget$
    such that
    $\alignment(\defensepi)=\alignment(\targetpi)$, then
    due to Remark
    \ref{remark.zeta_infinity}, Equation \eqref{eq.proof.attack_influence_a.1} is equivalent to 
    \begin{align*}
    &(\barscore^{\optpi}-\widehat \score^{\optpi}) - (\barscore^{\defensepi}-\widehat \score^{\defensepi}) 
    \le 
    (\barscore^{\optpi}-\widehat \score^{\optpi}) - (\barscore^{\targetpi}-\widehat \score^{\targetpi}) 
    \\\iff& 
    (\barscore^{\targetpi}-\widehat \score^{\targetpi})\le  (\barscore^{\defensepi}-\widehat \score^{\defensepi}) 
    \\\iff & 
    \vecdot{\overline{R}-\widehat{R}}{
    \occupancy^{\defensepi}
    -
    \occupancy^{\targetpi}}
    \ge 0.
    \end{align*}
    Furthermore, given Lemma \ref{lm.tightset}, we obtain
    \begin{align*}
        &\overline{R}-\widehat{R}=
        \sum_{(s,a)\in\tightsettarget}\alpha_{s,a}
        (\occupancy^{\neighbor{\targetpi}{s}{a}}-\occupancy^{\targetpi}) \\\implies&
        \vecdot{\overline{R}-\widehat{R}}{
    \occupancy^{\defensepi}
    -
    \occupancy^{\targetpi}}
        =\sum_{(s,a)\in \tightsettarget}
        \alpha_{s,a}(\alignment(\defensepi) - \alignment(\targetpi)) \ge 0.
    \end{align*}
    Now, consider the case when $\alignment(\defensepi)> \alignment(\targetpi)$ for all $(s,a)\in \tightsettarget$. Since $\optpi$ is always deterministic,
    by the definition of $\zeta$, the following holds for all state-action pairs $(s, a) \in \tightsettarget$: 
    \begin{align*}
		\frac{\alignment(\optpi) -  \alignment(\defensepi)}{\alignment(\defensepi) - \alignment(\targetpi)} \le \zeta.
		\end{align*}
		Since $\alignment(\defensepi) > \alignment(\targetpi)$,
		the above inequality can equivalently be written as
		\begin{align*}
		\alignment(\optpi) -  \alignment(\defensepi) \le 
	\zeta \cdot (\alignment(\defensepi) - \alignment(\targetpi)).    
		\end{align*}
		Using the definition of $\alignment$, we obtain that for state-action pairs $(s, a) \in \tightsettarget$
		\begin{align*}
		\vecdot{\occupancy^{\neighbor{\targetpi}{s}{a}}  - \occupancy^{\targetpi}}{\occupancy^{\optpi} - \occupancy^{\defensepi}
- \zeta \cdot (\occupancy^{\defensepi} -\occupancy^{\targetpi})	
} \le 0.
	\end{align*}
	Due to Lemma \ref{lm.tightset}, we know that
	\begin{align*}
	    \overline{R} = \widehat{R} + \sum_{(s, a) \in \tightsettarget} \alpha_{s,a} \cdot \left (\occupancy^{\neighbor{\targetpi}{s}{a}} - \occupancy^{\targetpi} \right),
	\end{align*}
	for some $\alpha_{s, a} \ge 0$. Therefore
	\begin{align*}
	    &\vecdot{\overline{R} - \widehat{R}}{\occupancy^{\optpi} - \occupancy^{\defensepi}
			- \zeta \cdot (\occupancy^{\defensepi} -\occupancy^{\targetpi})	
		} \\=& \sum_{(s, a) \in \tightsettarget} \alpha_{s,a} \cdot \vecdot{\occupancy^{\neighbor{\targetpi}{s}{a}}  - \occupancy^{\targetpi}}{\occupancy^{\optpi} - \occupancy^{\defensepi}
- \zeta \cdot (\occupancy^{\defensepi} -\occupancy^{\targetpi})	
} \\\le& 0.
	\end{align*}
	Let us now rewrite the left hand side of the inequality:
	\begin{align*}
	&\vecdot{\overline{R} - \widehat{R}}{\occupancy^{\optpi} - \occupancy^{\defensepi}
			- \zeta \cdot (\occupancy^{\defensepi} -\occupancy^{\targetpi})	
		}\\=&
		\vecdot{\overline{R} - \widehat{R}}{(\occupancy^{\optpi} - \occupancy^{\defensepi}) \cdot (1+\zeta) 
			- \zeta \cdot (\occupancy^{\defensepi} -\occupancy^{\targetpi} + \occupancy^{\optpi} - \occupancy^{\defensepi})
		}\\
	=&\vecdot{\overline{R} - \widehat{R}}{(\occupancy^{\optpi} - \occupancy^{\defensepi}) \cdot (1+\zeta) 
		-\zeta \cdot ( \occupancy^{\optpi} -\occupancy^{\targetpi} )
	}\\
	=&(1+\zeta) \cdot \left ((\barscore^{\optpi}- \hatscore^{\optpi}) - ( \barscore^{\defensepi}- \hatscore^{\defensepi}) \right) - \zeta \cdot \left ( (\barscore^{\optpi}- \hatscore^{\optpi}) - (\barscore^{\targetpi}- \hatscore^{\targetpi}) \right),
	\end{align*}
	where we used Equation \eqref{eq.score_occupancy_relation}. Therefore
	\begin{align*}
	    (1+\zeta) \cdot \left ((\barscore^{\optpi}- \hatscore^{\optpi}) - ( \barscore^{\defensepi}- \hatscore^{\defensepi}) \right) - \zeta \cdot \left ( (\barscore^{\optpi}- \hatscore^{\optpi}) - (\barscore^{\targetpi}- \hatscore^{\targetpi}) \right) \le 0.
	\end{align*}
    Since $1+ \zeta > 0$, we can rearrange the terms in the above inequality to obtain Equation \eqref{eq.proof.attack_influence_a.1}.
    
{\em Part 2:} Let us now consider the implications of Equation \eqref{eq.proof.attack_influence_a.1}. By rearranging \eqref{eq.proof.attack_influence_a.1}, we obtain
	\begin{align*}
	(\barscore^{\optpi}- \barscore^{\defensepi}) + (\widehat \score^{\defensepi} - \widehat \score^{\optpi}) \le 
	\frac{\zeta}{1+\zeta} \cdot \left(
	(\barscore^{\optpi}-\barscore^{\targetpi}) + (\widehat \score^{\targetpi} - \widehat \score^{\optpi})
	\right).
	\end{align*}
	Now, by applying the definition of the attack influence $\influence^{\pi}$, this inequality can be written as 
	\begin{align*}
	\influence^{\defensepi} + (
	\widehat\score^{\defensepi} - \widehat \score^{\optpi}
	)\le \frac{\zeta}{1+\zeta}\influence^{\targetpi} + \frac{\zeta}{1+\zeta}(
		\widehat\score^{\targetpi} - \widehat\score^{\optpi}
	). \end{align*}
Since the attack is feasible and $\optpi$ is deterministic,
$\hatscore^{\targetpi} \ge \hatscore^{\optpi} + \epsA$.
Therefore
	\begin{align*}
	 \influence^{\defensepi} + (
	\widehat\score^{\defensepi} - \widehat \score^{\optpi}
	)
&\le
\frac{\zeta}{1+\zeta}\influence^{\targetpi} + \frac{\zeta}{1+\zeta}(
		\widehat\score^{\targetpi} - \widehat\score^{\optpi}
	)
	\\&=
\frac{\zeta}{1+\zeta} \cdot \influence^{\targetpi} + (
		\widehat\score^{\targetpi} - \widehat\score^{\optpi})
			-\frac{1}{\zeta + 1}(
		\widehat\score^{\targetpi} - \widehat\score^{\optpi})\\&\le
		\frac{\zeta}{1+\zeta} \cdot
		\influence^{\targetpi} + (
		\widehat\score^{\targetpi} - \widehat\score^{\optpi})
			-\frac{1}{\zeta + 1}\epsA
		.   
	\end{align*}
	By rearranging and setting $\widehat \influence = \widehat \score^{\targetpi} - \widehat\score^{\defensepi}$, we obtain  
	\begin{align*}
	\influence^{\defensepi} &\le \frac{\zeta}{1+\zeta}\cdot \influence^{\targetpi} +
	[\widehat \score^{\targetpi} - \widehat\score^{\defensepi}] -\frac 1{1+\zeta} \epsA
    \\&=
    \frac{\zeta}{1+\zeta} \cdot [\influence^{\targetpi} + \epsA]  + [\widehat \influence - \epsA].
	\end{align*}
	which completes the proof.
    \end{proof}
    

\subsection{Proof of Theorem \ref{cor.attack_influence_a}}\label{sec.appendix.proof_cor_attack_influence_a}

\textbf{Statement:} {\em
 Let $\occstate_{\min} = \min_{\pi, s} \occstate^{\pi}(s)$, 
 and assume that $\beta^{\mu} \le \occstate_{\min}^2$. Then, the condition in Equation \eqref{eq.attack_influence_condition} holds and the attack influence $\influence^{\defensepi}$ is bounded by 
 \begin{align*}
     \influence^{\defensepi} \le \max \left \{\widehat \influence, \frac{1 + 2 \cdot \frac{\beta^{\occstate}}{\occstate_{\min}^2}}{2+\frac{\beta^{\occstate}}{\occstate_{\min}^2}}\cdot [\influence^{\targetpi} +  \epsA] + [\widehat \influence - \epsA]\right \},
 \end{align*}
 where $\widehat \influence = \widehat \rho^{\targetpi} - \widehat \rho^{\defensepi}$.
 }
 \begin{proof}
     We prove the statement by first showing that the assumptions of the theorem ($\beta^{\mu} \le \occstate_{\min}^2$) imply those of Theorem \ref{thm.attack_influence_a} (the condition in Equation \eqref{eq.attack_influence_condition}), and then bounding the $\zeta$ term of Theorem \ref{thm.attack_influence_a}. 
     
     \textit{Part 1:} To see that the condition in Equation \eqref{eq.attack_influence_condition} follows from the condition $\beta^{\mu} \le \occstate_{\min}^2$, let us inspect the difference $\alignment(\defensepi) - \alignment(\targetpi)$ for state-action pairs $(s,a) \in \tightsettarget$. 
     The following holds for $\alignment(\targetpi)$:
	\begin{align*}
		\alignment(\targetpi) &=  \vecdot{\occupancy^{\neighbor{\targetpi}{s}{a}} - \occupancy^{\targetpi}}{\occupancy^{\targetpi}}\\
		&= \sum_{\tilde{s}, \tilde{a}}
\occupancy^{\targetpi}(\tilde s, \tilde a) \cdot 
\Big(
\occupancy^{\neighbor{\targetpi}{s}{a}}(\tilde s, \tilde a) -
\occupancy^{\targetpi}(\tilde s, \tilde a) 
\Big)\\
&\overset{(i)}{=} \sum_{\tilde{s}\ne s}
\occupancy^{\targetpi}(\tilde s, \targetpi(\tilde s)) \cdot 
\Big(
\occupancy^{\neighbor{\targetpi}{s}{a}}(\tilde s, \targetpi(\tilde s)) -
\occupancy^{\targetpi}(\tilde s, \targetpi(\tilde s)) 
\Big) - \occupancy^{\targetpi}(s,\targetpi(s))^2\\
&\overset{(ii)}{=}
 \sum_{\tilde{s}\ne s}
\occstate^{\targetpi}(\tilde s) \cdot
\Big(
\occstate^{\neighbor{\targetpi}{s}{a}}(\tilde s) -
\occstate^{\targetpi}(\tilde s) 
\Big) - \occstate^{\targetpi}(s)^2.
\end{align*}
To obtain $(i)$, we used the fact that $\occupancy^{\targetpi}(\tilde s, \tilde a) = 0$ if $\tilde a \ne \targetpi(\tilde s)$ and $\occupancy^{\neighbor{\targetpi}{s}{a}}(s, \targetpi(s))=0$ since $a\ne \targetpi(s)$. To obtain $(ii)$, we applied Equation \eqref{eq.occupancy_measures_relation_b}. We can further bound this term by
\begin{align*}
\alignment(\targetpi) &= \sum_{\tilde{s}\ne s}
\occstate^{\targetpi}(\tilde s) \cdot
\Big(
\occstate^{\neighbor{\targetpi}{s}{a}}(\tilde s) -
\occstate^{\targetpi}(\tilde s) 
\Big) - \occstate^{\targetpi}(s)^2\\
&\le \norm{\occstate^{\neighbor{\targetpi}{s}{a}} - \occstate^{\targetpi}}_\infty \cdot
\norm{\occstate^{\targetpi}}_1 -
\occstate^{\targetpi}(s)^2\overset{(i)}{\le} \beta^{\occstate} - \occstate^{\targetpi}(s)^2,
\end{align*}
	where $(i)$ follows from the definition of $\beta^{\occstate}$. 
	Therefore, we have that for state-action pairs $(s,a) \in \tightsettarget$
	\begin{align}\label{eq.cor.attack_influence_a.bound_zeta_denominator}
	    \alignment(\defensepi) - \alignment(\targetpi) \ge -\alignment(\targetpi) \ge \occstate^{\targetpi}(s)^2 - \beta^{\occstate} \ge 0, 
	\end{align}
	where the first inequality is due to Theorem \ref{thm.defense_optimization.known_epsilon}
	and the second inequality
	is due to the assumption
	$\beta^{\occstate} \le \occstate_{\min}^2$
	Hence, the condition in Equation \eqref{eq.attack_influence_condition} is satisfied. 
	
	\textit{Part 2:} Now, we proceed with by bounding the $\zeta$ term of Theorem \ref{thm.attack_influence_a}. 
	We assume that
	$\tightsettarget\ne \emptyset$
	as otherwise
	$\zeta=0$
	and the statement
	follows trivially
	from Theorem 
	\ref{thm.attack_influence_a}.
	We further assume that
	$\beta^{\mu} < \occstate_{\min}^2$.
	If this isn't the case, then
	the multiplicative factor behind
	$\influence^{\targetpi}$ in the theorem's statement would equal 1.
	Therefore the statement would directly follow from 
	Theorem
	\ref{thm.attack_influence_a},
	since $\frac{\zeta}{\zeta + 1}
	\le 1$. We therefore focus on the case where
	$\beta^{\mu} < \occstate_{\min}^2$.
	
	Recall that
	\begin{align*}
	    \zeta = \max_{(s,a) \in \tightsettarget, \pi \in \Pi^{\text{det}}} \frac{\alignment(\pi)-\alignment(\defensepi)}{\alignment(\defensepi)-\alignment(\targetpi)}.
	\end{align*}
	Equation \eqref{eq.cor.attack_influence_a.bound_zeta_denominator} bounds the denominator from below, so we only need to bound the nominator.
	
	Since $\alignment(\defensepi) \ge 0$ due to Theorem \ref{thm.defense_optimization.known_epsilon}, it follows that
	$\alignment(\pi) - \alignment(\defensepi) \le \alignment(\pi)$. Hence, it suffices to bound  $\alignment(\pi)$. 
	
	Now, using relation \eqref{eq.occupancy_measures_relation_b}, we have that for any deterministic $\pi$,
	\begin{align*}
		\alignment(\pi) &= 
\vecdot{\occupancy^{\neighbor{\targetpi}{s}{a}}
-\occupancy^{\targetpi}
}{\occupancy^{\pi}}  \\
&=\sum_{\tilde s, \tilde a} \occupancy^{\pi}(\tilde s, \tilde a) \cdot \left ( \occupancy^{\neighbor{\targetpi}{s}{a}}(\tilde s, \tilde a)
-\occupancy^{\targetpi}(\tilde s, \tilde a)
 \right )\\
 &=\sum_{\tilde a} \occupancy^{\pi} (s, \tilde a) \cdot \left ( \occupancy^{\neighbor{\targetpi}{s}{a}}(s, \tilde a)
-\occupancy^{\targetpi}(s, \tilde a)
 \right )  +
 \\&\quad\quad\quad\quad\quad\quad
 \sum_{\tilde s \ne s, \tilde a} \occupancy^{\pi}(\tilde s, \tilde a) \cdot \left ( \occupancy^{\neighbor{\targetpi}{s}{a}}(\tilde s, \tilde a)
-\occupancy^{\targetpi}(\tilde s, \tilde a)
 \right ),
 \end{align*}
Let us consider each of the terms separately. For the first term, we have
\begin{align*}
    \sum_{\tilde a} \occupancy^{\pi} (s, \tilde a) \cdot \left ( \occupancy^{\neighbor{\targetpi}{s}{a}}(s, \tilde a)
-\occupancy^{\targetpi}(s, \tilde a)
 \right )
 &\overset{(i)}{\le} \sum_{\tilde a} \occupancy^{\pi} (s, \tilde a) \cdot \occupancy^{\neighbor{\targetpi}{s}{a}}(s, \tilde a) \\
 &\overset{(ii)}{=} \occupancy^{\pi} (s, \pi(s)) \cdot \occupancy^{\neighbor{\targetpi}{s}{a}}(s, \pi(s))\\
 &\overset{(iii)}{\le} \occupancy^{\pi} (s, \pi(s)) \cdot \occupancy^{\neighbor{\targetpi}{s}{a}}(s, \neighbor{\targetpi}{s}{a}(s))\\
 &\overset{(iv)}{=}\occstate^{\pi} (s) \cdot \occstate^{\neighbor{\targetpi}{s}{a}}(s)\\
 &= \occstate^{\pi} (s) \cdot \left (\occstate^{\neighbor{\targetpi}{s}{a}}(s) - \occstate^{\targetpi}(s) \right ) + \occstate^{\pi} (s) \cdot \occstate^{\targetpi}(s)\\
 &\le \occstate^{\pi} (s) \cdot \left |\occstate^{\neighbor{\targetpi}{s}{a}}(s) - \occstate^{\targetpi}(s) \right | + \occstate^{\pi} (s) \cdot \occstate^{\targetpi}(s).
\end{align*}
Here, $(i)$ follows from $\occupancy^{\targetpi}(s, \tilde a) \ge 0$, $(ii)$ is due to the  
fact that $\occupancy^{\pi}(\tilde s, \tilde a) = 0$ if $\tilde a \ne \pi(\tilde s)$, $(iii)$ is due to the $\occupancy^{\neighbor{\targetpi}{s}{a}}(\tilde s, \tilde a) = 0$ if $\tilde a \ne \neighbor{\targetpi}{s}{a}(\tilde s)$, and $(iv)$ is due to Equation \eqref{eq.occupancy_measures_relation_b}. For the second term, we have
\begin{align*}
    &\sum_{\tilde s \ne s, \tilde a} \occupancy^{\pi}(\tilde s, \tilde a) \cdot \left ( \occupancy^{\neighbor{\targetpi}{s}{a}}(\tilde s, \tilde a)
-\occupancy^{\targetpi}(\tilde s, \tilde a)
 \right )
\\ \overset{(i)}{=}& \sum_{\tilde s \ne s} \occupancy^{\pi}(\tilde s, \targetpi(\tilde s)) \cdot \left ( \occupancy^{\neighbor{\targetpi}{s}{a}}(\tilde s, \targetpi(\tilde s))
-\occupancy^{\targetpi}(\tilde s, \targetpi(\tilde s))
 \right )\\
 \overset{(ii)}{=}& \sum_{\tilde s \ne s} \occupancy^{\pi}(\tilde s, \targetpi(\tilde s)) \cdot \left ( \occstate^{\neighbor{\targetpi}{s}{a}}(\tilde s)
-\occstate^{\targetpi}(\tilde s)
 \right )\\
 \overset{(iii)}{\le}& \sum_{\tilde s \ne s} \occupancy^{\pi}(\tilde s, \targetpi(\tilde s)) \cdot \left | \occstate^{\neighbor{\targetpi}{s}{a}}(\tilde s)
-\occstate^{\targetpi}(\tilde s)
 \right |\\
 \overset{(iv)}{\le}& \sum_{\tilde s \ne s} \occupancy^{\pi}(\tilde s, \pi(\tilde s)) \cdot \left | \occstate^{\neighbor{\targetpi}{s}{a}}(\tilde s)
-\occstate^{\targetpi}(\tilde s)
 \right |\\
 \overset{(v)}{=}& \sum_{\tilde s \ne s} \occstate^{\pi}(\tilde s) \cdot \left | \occstate^{\neighbor{\targetpi}{s}{a}}(\tilde s)
-\occstate^{\targetpi}(\tilde s)
 \right |,
\end{align*}
where $(i)$ is due to the fact that for $\tilde s \ne s$, $\occupancy^{\targetpi}(\tilde s, \tilde a ) = \occupancy^{\neighbor{\targetpi}{s}{a}}(\tilde s, \tilde a) = 0$ if $\tilde a \ne \targetpi(\tilde s)$ (note that $\targetpi(\tilde s) = \neighbor{\targetpi}{s}{a}(\tilde s)$ for $\tilde s \ne s$), $(ii)$ is due to Equation \eqref{eq.occupancy_measures_relation_b}, $(iii)$ is due to the 
fact that $\occupancy^\pi$
is non-negative,
$(iv)$ is due to the
fact that $\occupancy^{\pi}(\tilde s, \tilde a) = 0$ if $\tilde a \ne \pi(\tilde s)$, and $(v)$ is due to Equation \eqref{eq.occupancy_measures_relation_b}. Putting together the above, we obtain 
\begin{align*}
     \alignment(\pi) - \alignment(\defensepi) &\le \alignment(\pi) \\&\le \sum_{\tilde s} \occstate^{\pi}(\tilde s) \cdot \left | \occstate^{\neighbor{\targetpi}{s}{a}}(\tilde s)
-\occstate^{\targetpi}(\tilde s) \right| + \occstate^{\pi} (s) \cdot \occstate^{\targetpi}(s)\\
&\le \sum_{\tilde s} \occstate^{\pi}(\tilde s) \cdot \norm{\occstate^{\neighbor{\targetpi}{s}{a}}
-\occstate^{\targetpi} }_{\infty} + \occstate^{\pi} (s) \cdot \occstate^{\targetpi}(s)\\
&= \norm{\occstate^{\neighbor{\targetpi}{s}{a}}
-\occstate^{\targetpi} }_{\infty} + \occstate^{\pi} (s) \cdot \occstate^{\targetpi}(s)
\\&= \norm{\occstate^{\neighbor{\targetpi}{s}{a}}
-\occstate^{\targetpi} }_{\infty} + \left (\occstate^{\pi}  - \occstate^{\targetpi} \right ) (s) \cdot \occstate^{\targetpi}(s) +  \occstate^{\targetpi}(s)^2
\\&\le \norm{\occstate^{\neighbor{\targetpi}{s}{a}}
-\occstate^{\targetpi} }_{\infty} + \norm{\occstate^{\pi}  - \occstate^{\targetpi} }_{\infty} \cdot \occstate^{\targetpi}(s) +  \occstate^{\targetpi}(s)^2\\
&\overset{(i)}{\le} \beta^{\occstate} + \frac{\beta^{\occstate}}{\occstate_{\min}} \cdot \occstate^{\targetpi}(s) + \occstate^{\targetpi}(s)^2 \\&\le 2 \cdot \frac{\beta^{\occstate}}{\occstate_{\min}} \cdot \occstate^{\targetpi}(s) + \occstate^{\targetpi}(s)^2,
\end{align*}
where
for $(i)$
we applied the definition of $\beta^{\occstate}$ and Lemma \ref{lemma.fraction.bound}
and for the last inequality we have used the fact that
$\occstate^{\targetpi}(s) \ge \occstate_{\min}(s)$ by definition.

Therefore, combining the bounds on $\alignment(\pi) - \alignment(\defensepi)$ and $\alignment(\defensepi) - \alignment(\targetpi)$ (Equation \eqref{eq.cor.attack_influence_a.bound_zeta_denominator}), we obtain 
\begin{align*}
    \zeta \le \frac{2 \cdot \frac{\beta^{\occstate}}{\occstate_{\min}} \cdot \occstate^{\targetpi}(s) + \occstate^{\targetpi}(s)^2}{\occstate^{\targetpi}(s)^2 - \beta^{\occstate}} = \frac{1 + 2 \cdot \frac{\beta^{\occstate}}{\occstate_{\min} \cdot \occstate^{\targetpi}(s)} }{1- \frac{\beta^{\occstate}}{\occstate^{\targetpi}(s)^2}} \le \frac{1 + 2 \cdot \frac{\beta^{\occstate}}{\occstate_{\min}^2} }{1- \frac{\beta^{\occstate}}{\occstate_{\min}^2}}. 
\end{align*}
To obtain the last inequality, we have used the definition of $\occstate_{\min}$.
Since $x\to \frac{x}{1 + x}$ is an increasing function for $x > -1$, this further implies that 
\begin{align*}
    \frac{\zeta}{\zeta + 1}  \le \frac{\frac{1 + 2 \cdot \frac{\beta^{\occstate}}{\occstate_{\min}^2} }{1- \frac{\beta^{\occstate}}{\occstate_{\min}^2}}}{1 + \frac{1 + 2 \cdot \frac{\beta^{\occstate}}{\occstate_{\min}^2} }{1- \frac{\beta^{\occstate}}{\occstate_{\min}^2}}} = \frac{\frac{1 + 2 \cdot \frac{\beta^{\occstate}}{\occstate_{\min}^2} }{1- \frac{\beta^{\occstate}}{\occstate_{\min}^2}}}{\frac{1- \frac{\beta^{\occstate}}{\occstate_{\min}^2} + 1 + 2 \cdot \frac{\beta^{\occstate}}{\occstate_{\min}^2} }{1- \frac{\beta^{\occstate}}{\occstate_{\min}^2}}} \le \frac{1 + 2 \cdot \frac{\beta^{\occstate}}{\occstate_{\min}^2}}{2+\frac{\beta^{\occstate}}{\occstate_{\min}^2}},
\end{align*}
and hence, the claim of the theorem follows from Theorem \ref{thm.attack_influence_a}.
 \end{proof}

\subsection{Proof of Theorem \ref{prop.influence_impossibility}}\label{sec.appendix.proof_prop_influence_impossibility}
\textbf{Statement: }{\em 
Fix the poisoned reward function $\widehat{R}$, and assume that there exists state-action pair $(s, a) \in \tightsettarget$ such that $\alignment(\defensepi) < \alignment(\targetpi)$. Then for any $\delta > 0$, there exists a reward function $\overline{R}$ such that $\widehat R = \attack(\overline R, \targetpi,\epsA)$ and $\influence^{\defensepi} \ge  \influence^{\targetpi} + \delta$.
}

\begin{proof}
	Consider a state-action pair $(s,a) \in \tightsettarget$ such that
	\begin{align*}
		 \alignment(\defensepi) < \alignment(\targetpi).
	\end{align*}
	Note that the statement of the theorem assumes that there is at least one such pair. Due to Lemma \ref{lm.tightset}, reward function $R$ defined as 
	\begin{align*}
	R = \widehat{R} + 
	\alpha \cdot (\occupancy^{\neighbor{\targetpi}{s}{a}}
	 - \occupancy^{\targetpi}),
	\end{align*}
	where $\alpha > 0$ is an arbitrary positive number, 
	satisfies $\widehat{R} = \attack(R, \targetpi, \epsA)$. Therefore, such $R$ is a plausible candidate for $\overline{R}$, so let us consider the case when $\overline{R} = R$. 
	
	Now, using the definition of the attack influence and Equation \eqref{eq.score_occupancy_relation}, we obtain
	\begin{align*}
		\influence^\defensepi - \influence^\targetpi &=
		(\overline \score^{\optpi} - \overline \score^{\defensepi}) - (\overline \score^{\optpi} - \overline \score^{\targetpi}) = \overline \score^{\targetpi} -  \overline \score^{\defensepi} \\&\overset{(i)}{\ge}
		\overline \score^{\targetpi} - \overline \score^{\defensepi} - 	[\widehat\score^{\targetpi} - \widehat\score^{\defensepi}] = (\overline \score^{\targetpi} - \widehat\score^{\targetpi}) - 
		(\overline \score^{\defensepi} - \widehat\score^{\defensepi}),
		\end{align*}
		where $(i)$ is due to the optimality of $\targetpi$ under the reward function $\widehat R$. Since $\overline{R} = R$, and due to Equation \eqref{eq.score_occupancy_relation}, we have that
		\begin{align*}
		    \overline \score^{\targetpi} - \widehat\score^{\targetpi} &= \vecdot{\occupancy^{\targetpi}}{\overline{R}} - \vecdot{\occupancy^{\targetpi}}{\widehat{R}} \\&= \vecdot{\overline{R} - \widehat{R}}{\occupancy^{\targetpi}} = \vecdot{R - \widehat{R}}{\occupancy^{\targetpi}} \\&= \vecdot{\alpha \cdot (\occupancy^{\neighbor{\targetpi}{s}{a}}
	 - \occupancy^{\targetpi})}{\occupancy^{\targetpi}}, 
		\end{align*}
		and similarly
		\begin{align*}
		    \overline \score^{\defensepi} - \widehat\score^{\defensepi} &= \vecdot{\occupancy^\defensepi}{\overline{R}} - \vecdot{\occupancy^\defensepi}{\widehat{R}} \\&= \vecdot{\overline{R} - \widehat{R}}{\occupancy^\defensepi} = \vecdot{R - \widehat{R}}{\occupancy^\defensepi} \\&= \vecdot{\alpha \cdot (\occupancy^{\neighbor{\targetpi}{s}{a}}
	 - \occupancy^{\targetpi})}{\occupancy^\defensepi}.
		\end{align*}
		Therefore, putting everything together, we obtain 
		\begin{align*}
		\influence^\defensepi - \influence^\targetpi \ge (\barscore^{\targetpi} - \widehat\score^{\targetpi}) - 
		(\barscore^{\defensepi} - \widehat\score^{\defensepi}) &= \alpha \cdot  \vecdot{\occupancy^{\neighbor{\targetpi}{s}{a}}
			- \occupancy^{\targetpi}}{\occupancy^{\targetpi} - \occupancy^{\defensepi}} \\&=
		 \alpha \cdot ( \alignment(\targetpi) - \alignment(\defensepi) ).
	\end{align*}
	Since $\alpha$ is any positive number, by setting it to $\alpha = \frac{\delta}{\alignment(\targetpi) - \alignment(\defensepi)}$, we obtain $\influence^\defensepi - \influence^\targetpi \ge \delta$,
	which proves the claim. 
\end{proof}


\section{Proofs of Section \ref{sec.defense_characterization_general_MDP.unknown_epsilon}}\label{sec.appendix.defense_characterization_general_MDP_uknown_epsilon}

\subsection{Proof of Theorem \ref{thm.overestimate_attack_param}}\label{sec.appendix.thm_prop_overestimate_attack_param}
The proof of the theorem is similar to the proof of
Theorem
\ref{thm.defense_optimization.known_epsilon} and builds on two lemmas which we introduce in this section.
\begin{lemma}\label{lm.tightset.unknown}
Set $\epsilon=\min\{\epsD, \widehat \epsilon\}$, where
\begin{align*}
    \widehat \epsilon = \min_{s, a\ne \targetpi(s)} \left [ \widehat \score^{\targetpi}-\widehat \score^{\neighbor{\targetpi}{s}{a}} \right]. 
\end{align*}
    Reward function $R$ satisfies $\widehat{R} = \attack(R, \targetpi, \tilde \epsilon)$ for some $\tilde \epsilon \in (0, \epsD]$ if any only if
\begin{align*}
    R = \widehat{R} + \sum_{(s, a) \in \tightset} \alpha_{s,a} \cdot \left (\occupancy^{\neighbor{\targetpi}{s}{a}} - \occupancy^{\targetpi} \right),
\end{align*}
for some $\alpha_{s, a} \ge 0$.
\end{lemma}\begin{proof}
    We divide the proof into two parts, respectively proving the sufficiency and the necessity of the condition.
    
    \textit{Part 1 (Necessity):} Assume that $\widehat{R} = \attack(R, \targetpi, \tilde \epsilon)$ for some $\tilde \epsilon \in (0, \epsD]$. 
		From the stationariry and dual feasibility
		conditions in Lemma \ref{lm.kkt_conditions}, we deduce
		\begin{align}
		\exists\lambda \succcurlyeq 0:
			R = \widehat{R} + \sum_{s,a\ne \targetpi(s)}  \lambda(s,a) \cdot (
		\occupancy^{\neighbor{\targetpi}{s}{a}} - \occupancy^{\targetpi}
		)
		\label{eq.proof.lm.tightset.unkonwn.1}.
		\end{align}
		We claim that $\lambda(s,a)=0$ for all $(s,a) \notin \tightset$. Note that this would imply the
		lemma's statement by setting $
\alpha_{s,a}= \lambda(s,a)$,
since the terms corresponding to $(s,a)\notin \tightset$ could be skipped in the summation of \eqref{eq.proof.lm.tightset.unkonwn.1}.

		To see why the claim holds, assume that $\lambda(s,a) \ne 0$ for some $(s,a)$
		where $a\ne\targetpi(s)$
		.
		From complementary slackness, we know that
		$\occdiffmatrix(s,a)^T \cdot R + \tilde \epsilon = 0$, which implies that
\begin{align}
\widehat{\epsilon} = \min_{\tilde s, \tilde a \ne \targetpi(\tilde s)} (-\occdiffmatrix(\tilde s,\tilde a)^T \cdot R) \le 
	-\occdiffmatrix(s,a)^T \cdot R  = \tilde \epsilon
	\label{eq.proof.lm.tightset.unkonwn.2}.
\end{align}
However, $\tilde \epsilon \le \widehat{\epsilon}$ holds by primal feasibility. 
Therefore, all the inequalities are equalities, which implies
$
\tilde	\epsilon=\widehat{\epsilon}
$.
Since $\tilde \epsilon\le \epsD$,
we conclude that $\tilde \epsilon=\min\{\epsD, \widehat{\epsilon}\}=\epsilon$.
Since
all of the inequalities in 
\eqref{eq.proof.lm.tightset.unkonwn.2} are indeed
equalities, we conclude
\begin{align*}
		-\occdiffmatrix(s,a)^T \cdot R =\epsilon \implies
		(s,a) \in \tightset,
\end{align*}
which proves the claim.

\textit{Part 2 (Sufficiency):}
Assume that
\begin{align*}
    R = \widehat{R} + \sum_{(s, a) \in \tightset} \alpha_{s,a} \cdot \left (\occupancy^{\neighbor{\targetpi}{s}{a}} - \occupancy^{\targetpi} \right),
\end{align*}
for some $\alpha_{s, a} \ge 0$.
Set $\tilde \epsilon=\epsilon$ and note that
$\tilde \epsilon \le \epsD$ by definition. 
Set 
\begin{align*}
    \lambda(s,a)
=\begin{cases}
\alpha_{s,a} \quad\text{if} \quad (s,a) \in \tightset\\
0 \quad\text{o.w.}
\end{cases}.
\end{align*}
We now verify all the conditions of Lemma \ref{lm.kkt_conditions}
hold. 
Stationarity and dual feasibility hold because
$R = \widehat{R}+ \sum_{s, a} \lambda(s,a) \cdot \occdiffmatrix(s,a)$ and $\lambda \succcurlyeq 0$.
Primal feasibility holds because $\tilde \epsilon=\min\{\epsD, \widehat{\epsilon}\}\le \widehat{\epsilon}$.
Finally, complementary slackness holds because
\begin{align*}
\lambda(s,a) \ne 0 \implies
(s,a) \in \tightset \implies
\occdiffmatrix(s,a)^TR + \epsilon = 0.
\end{align*}
\end{proof}
\begin{lemma}\label{lm.wost_case_utility.unkown}
    Let $\score^{\pi}_{min}$ be the solution to
	the following optimization problem
	\begin{align}
	\label{prob.wost_case_utility.unkown}
	\tag{P5}
		\min_{R} \score^{\pi}\quad s.t. \quad
		\widehat{R} = \attack(R, \targetpi, \tilde\epsilon)\land 
		0< \tilde \epsilon \le \epsD.
	\end{align}
	Then
	\begin{align*}
		\score^{\pi}_{min} = \begin{cases}
			\hatscore^{\pi} \quad\text{if } \quad\forall (s,a) \in \tightset: \alignment(\pi)\ge 0\\
			-\infty \quad \text{o.w.}
		\end{cases},
	\end{align*}
	where $\epsilon=\min\{\epsD, \widehat \epsilon\}$,
	and 
	\begin{align*}
    \widehat \epsilon = \min_{s, a\ne \targetpi(s)} \left [ \widehat \score^{\targetpi}-\widehat \score^{\neighbor{\targetpi}{s}{a}} \right]. 
\end{align*}
\end{lemma}
\begin{proof}
The proof is similar to the proof of Lemma \ref{lm.wost_case_utility}.
    We separately analyze the two cases: the case when $\alignment(\pi) \ge 0$ for all $(s,a) \in \tightset$ holds, and the case when it does not.
    
    \textbf{Case 1:} If $\alignment(\pi) \ge 0$ for all $(s,a) \in \tightset$, then by using Equation \eqref{eq.score_occupancy_relation} and Lemma \ref{lm.tightset.unknown} we obtain that 
	\begin{align*}
		\score^{\pi} - \widehat{\score}^{\pi} = 
		\vecdot{\occupancy^{\pi}}{R - \widehat R} = 
		\sum_{(s,a) \in\tightset} \alpha_{s, a}\cdot \vecdot{\occupancy^{\pi}}{
	\occupancy^{\neighbor{\targetpi}{s}{a}} - \occupancy^{\targetpi}
	} \ge 0.
	\end{align*}
	Therefore, $\score^{\pi} \ge \widehat{\score}^{\pi}$. Furthermore, from Lemma \ref{lm.tightset.unknown}, we know that $R = \widehat R$ satisfies the constraint in the optimization problem \eqref{prob.wost_case_utility.unkown}, so the score of the optimal solution to \eqref{prob.wost_case_utility.unkown} is $\score^{\pi}_{\min} = \widehat{\score}^{\pi}$.
	
	\textbf{Case 2:} Now, consider the case when $\alignment(\pi) < 0$ for a certain state-action pair $(s, a) \in \tightset$. Let $\alpha_{s, a}$ be an arbitrary positive number. From Lemma \ref{lm.tightset.unknown}, we know that
\begin{align*}
    R = \widehat R + \alpha_{s, a} \cdot \vecdot{\occupancy^{\pi}}{
	\occupancy^{\neighbor{\targetpi}{s}{a}} - \occupancy^{\targetpi}}
\end{align*}
satisfies the constraint in the optimization problem \eqref{prob.wost_case_utility.unkown}, and hence is a solution to \eqref{prob.wost_case_utility.unkown}. Moreover, by using this solution together with Equation \eqref{eq.score_occupancy_relation}, we obtain
\begin{gather*}
	\score^{\pi} - \widehat{\score}^{\pi} = 
\vecdot{\occupancy^{\pi}}{R - \widehat R} = 	\alpha_{s, a}\cdot \vecdot{\occupancy^{\pi}}{
	\occupancy^{\neighbor{\targetpi}{s}{a}} - \occupancy^{\targetpi}
	}= \alpha_{s, a} \cdot \alignment.
\end{gather*}
Since $\alpha_{s, a}$ can be arbitrarily large and $\alignment < 0$, while $\widehat{\score}^{\pi}$ is fixed, $\score^{\pi}$ can be arbitrarily small. Hence, the score of the optimal solution to \eqref{prob.wost_case_utility.unkown} is unbounded from below, i.e., $\score^{\pi}_{\min} = -\infty$.
\end{proof}

We are now ready to prove Theorem \ref{thm.overestimate_attack_param}.

\textbf{Statement:} { \em
Assume that $\epsD \ge \epsA$, and define $\widehat \epsilon = \min_{s, a\ne \targetpi(s)} \left [ \widehat \score^{\targetpi}-\widehat \score^{\neighbor{\targetpi}{s}{a}} \right]$. 
Then, the optimization problem \eqref{prob.defense_optimization.known_epsilon} with $\epsilon = \min \{ \epsD, \widehat \epsilon \}$ is feasible and its optimal solution $\occupancy_{\max}$ identifies an optimal policy $\defensepi$ for the optimization problem \eqref{prob.defense_b} via Equation
\eqref{eq.defense_optimization.known_epsilon}.
This policy $\defensepi$ satisfies $\overline \score^{\defensepi} \ge \widehat \score^{\defensepi}$. Furthermore, if the condition in Equation \eqref{eq.attack_influence_condition} holds, 
the attack influence of policy $\defensepi$ is bounded as in Equation \eqref{eq.thm_attack_influence}.
}
\begin{proof}
    The proof is divide into two parts, respectively proving the first and the second claim in the theorem statement.
    
    \textit{Part 1 (Solution to \eqref{prob.defense_b}):} We prove that the optimization problem  \eqref{prob.defense_optimization.known_epsilon} is feasible, its optimal solution $\occupancy_{\max}$ identifies an optimal solution to \eqref{prob.defense_b} via Equation \eqref{eq.defense_optimization.known_epsilon}, and satisfies
    $\barscore^{\defensepi} \ge \hatscore^{\defensepi}$.
    
    The feasibility of the problem follows from Lemma \ref{lm.prob_feasibility}.
     Note that
     $\occupancy_{\max}$ always
     exists since
     \eqref{prob.defense_optimization.known_epsilon} is
     maximizing a continuous function 
     over a closed and bounded set. Concretely,
     the constraints
     $\vecdot{\occupancy^{\neighbor{\targetpi}{s}{a}}-
     \occupancy^{\targetpi}
     }{\occupancy}\ge 0$
     and Equations
     \eqref{eq.bellman.occupancy.1} and \eqref{eq.bellman.occupancy.2} each define closed sets
     and since $||\occupancy||_1=1$, the set $\Occupancy$ is bounded.
    
      In order to see why
     $\occupancy_{\max}$
     specifies an optimal solution to
     \eqref{prob.defense_b}, 
     note that we can rewrite
     \eqref{prob.defense_b}
     as
     \begin{align*}
         \max_{\pi}
         \score^{\pi}_{\min},
     \end{align*}
     where
     $\score^{\pi}_{\min}$
     is the solution to the optimization problem 
     \eqref{prob.wost_case_utility.unkown}.
     Due to Lemma
     \ref{lm.wost_case_utility.unkown}, this could be rewritten as
     \begin{align*}
         &\max_{\pi} \hatscore^{\pi}\\
	&\quad \mbox{ s.t. } \alignment(\pi) \ge 0 \quad\forall(s,a) \in \tightset,
	\end{align*}
	where
	$\epsilon = \min\{\epsD, \widehat{\epsilon}\}$.
	Namely, maximizing a function $f(x)$ subject to constraint $x \in \mathcal X$ (where $\mathcal X \ne \emptyset$) is equivalent to maximizing $\tilde f(x)$, where
     \begin{align*}
         \tilde f(x) =
         \begin{cases}
             f(x) \quad\text{if} \quad x \in \mathcal X\\
             -\infty \quad\text{o.w.}
    \end{cases}.
    \end{align*}
	Due to 
	\eqref{eq.score_occupancy_relation} and
	\eqref{eq.alignment}, 
	the constrained optimization problem above can be rewritten as
	\begin{align*}
	    &\max_{\pi} 
	    \vecdot{\occupancy^{\pi}}{\widehat{R}}
	    \\
	&\quad \mbox{ s.t. } 
	\vecdot{\occupancy^{\neighbor{\targetpi}{s}{a}}-\occupancy^{\targetpi}}{\occupancy^{\pi}}\quad\forall(s,a) \in \tightset.
	\end{align*}
	Therefore,
	given Lemma 
	\ref{lm.syed}, 
	$\occupancy_{\max}$
	specifies a solution to
	\eqref{prob.defense_b} via
	\eqref{eq.defense_optimization.known_epsilon}.
    Finally, given Lemma
    \ref{lm.wost_case_utility.unkown}, $\occupancy^{\defensepi}$
    satisfies the constraints of \eqref{prob.wost_case_utility.unkown} and therefore $\hatscore^{\defensepi}$ is a lower bound on
    $\barscore^{\defensepi}$.
    
   \textit{Part 2 (Attack Influence):} We prove that 
   if the condition in Equation \eqref{eq.attack_influence_condition} holds,
   then Equation \eqref{eq.thm_attack_influence} holds with $\epsilon$.
   We divide the proof into two cases.
   
   \textbf{Case 1:}
   Assume that
    $\widehat{R} = \overline{R}$. In this case Equation \eqref{eq.thm_attack_influence} holds trivially as $\targetpi$ is optimal and 
    $\influence^{\defensepi} = \hatscore^{\targetpi} - \hatscore^{\defensepi}=\widehat{\influence}$.
    
    \textbf{Case 2:} 
    Assume that $\widehat{R}\ne \overline{R}$.
    Corollary 
    \ref{lm.epsilon.hat} implies that
    $\widehat{\epsilon}=\epsA$, which further implies
    $\epsilon=\min\{\widehat{\epsilon}, \epsD\}=\epsA$.
    Now,
    given Part 1 of this theorem,
    the optimization problem
     \eqref{prob.defense_b}
     is equivalent
     to
     optimization problem
     \eqref{prob.defense_optimization.known_epsilon}
     with $\epsilon=\epsA$.
     Therefore the statement
     follows from 
     Theorem
     \ref{thm.attack_influence_a}.
\end{proof}

\subsection{Proof of Theorem \ref{thm.underestimate_attack_param}}\label{sec.appendix.proof_thm_underestimate_attack_param}
\textbf{Statement:} {\em If $\epsA > \epsD$, then
$\targetpi$ is the unique solution of
the optimization problem \eqref{prob.defense_b}.
Therefore
$\defensepi=\targetpi$ and
$\influence^{\targetpi} = \influence^{\defensepi}$.}
\begin{proof}
    As in Theorem \ref{thm.overestimate_attack_param}, set $\epsilon=\min\{\widehat{\epsilon}, \epsD\}$
    where 
	\begin{align*}
    \widehat \epsilon = \min_{s, a\ne \targetpi(s)} \left [ \widehat \score^{\targetpi}-\widehat \score^{\neighbor{\targetpi}{s}{a}} \right]. 
\end{align*}
    From the feasibility of the attack, we have that
\begin{align*}
    \forall s, a\ne \targetpi(s):
    \hatscore^{\targetpi}
     - \hatscore^{\neighbor{\targetpi}{s}{a}}
    \ge \epsA > \epsD \ge \epsilon \implies
    \tightset=\emptyset.
\end{align*}
    Therefore, given Lemma \ref{lm.tightset.unknown}, the constraint in the optimization problem \eqref{prob.defense_b} is satisfied only for $R = \widehat R$. This reduces the optimization problem \eqref{prob.defense_b} to $\max_{\pi} \widehat \score^{\pi}$, which has a unique optimal solution: $\targetpi$. This, in turn, implies that $\defensepi = \targetpi$ and $\influence^{\targetpi} = \influence^{\defensepi}$.
\end{proof}

\section{Proofs of Appendix \ref{sec.defense_characterization_specific_mdp}}\label{sec.appendix.additional_results_special_mdps}

In this section, we provide a more formal treatment of the results in Appendix \ref{sec.defense_characterization_specific_mdp}, formally stating and proving these results.

\subsection{Characterization of the defense policy}\label{sec.appendix.special_mdp.defense_characterization}

In this section we provide
a formal treatment of the results outlined
in Appendix \ref{sec.defense_characterization_specific_mdp}, related to the defense policy $\defensepi$.

\begin{proposition}\label{lm.bandit.defense.character}
Assume that condition \eqref{eq.condition.no_action_influence} holds.
Set $\widehat{\epsilon}$ as
\begin{align*}
    \widehat \epsilon = \min_{s, a\ne \targetpi(s)} \left [ \widehat \score^{\targetpi}-\widehat \score^{\neighbor{\targetpi}{s}{a}} \right]. 
\end{align*}
Consider the following policy
\begin{align}
\defensepi(a|s) = \frac{\ind{a \in \tightsetstate
\cup \{ \targetpi(s) \}
}
}{|\tightsetstate| + 1}.
\label{eq.lm.bandit.characterization.policy}
\end{align}
Equation \eqref{eq.lm.bandit.characterization.policy} characterizes the solution to the optimization problems
\eqref{prob.defense_a}
and \eqref{prob.defense_b} with parameters
$\epsilon=\epsA$ and
$\epsilon=\min\{\widehat{\epsilon}, \epsD\}$
respectively.
Furthermore, in both cases 
$\barscore^{\defensepi}=\hatscore^{\defensepi}$
\end{proposition}
\begin{proof}
    Given Theorem \ref{thm.defense_optimization.known_epsilon}, Theorem \ref{thm.overestimate_attack_param}, and Lemma \ref{lm.syed},
    it suffices to show that if $\epsilon \le \widehat{\epsilon}$, the solution to the optimization problem
    \begin{align*}
    &\max_{\occupancy \in \Occupancy} \vecdot{\occupancy}{\widehat R}
	\tag{P4}\\
	\notag
	&\quad \mbox{ s.t. } \vecdot{\occupancy^{\neighbor{\targetpi}{s}{a}} - \occupancy^{\targetpi}}{\occupancy} \ge 0 \quad \forall s, a \in \tightset,
\end{align*}
corresponds to the occupancy measure of policy
$\defensepi$
defined by Equation
\eqref{eq.lm.bandit.characterization.policy}.
Namely, the optimization problems
\eqref{prob.defense_a} and
\eqref{prob.defense_b} correspond to
the optimization problem
\eqref{prob.defense_optimization.known_epsilon}
with parameters
$\epsilon=\epsA$ and
$\epsilon=\min\{\widehat{\epsilon}, \epsD\}$
respectively. 
Since $\widehat{\epsilon}\le \epsA$, the primal
feasibility condition in Lemma \ref{lm.kkt_conditions} implies that the solution to the above optimization problem characterizes both cases (\eqref{prob.defense_a} and
\eqref{prob.defense_b}).

Now, due to Lemma \ref{lm.syed}, we have
\begin{align*}
    \occupancy \in \Occupancy \iff 
    \occupancy \succcurlyeq 0 \land \forall s: 
    \sum_{a}\occupancy(s,a) = 
    (1-\gamma)\sigma(s) + 
    \gamma\sum_{\tilde s, \tilde a}
    P(\tilde s, \tilde a, s)\occupancy(\tilde s, \tilde a).
\end{align*}
Since $P(\tilde s, \tilde a, s)$ is independent of
$\tilde a$, 
the second condition is equivalent to
\begin{align*}
\forall s: 
    \sum_{a}\occupancy(s,a) =
     (1-\gamma)\sigma(s) + 
     \gamma \sum_{\tilde s}\Big(P(\tilde s, \targetpi(\tilde s), s)(\sum_{\tilde a}\occupancy(\tilde s, \tilde a))
     \Big),
\end{align*}
which,  due to \eqref{eq.bellman.occstate},
is equivalent to
\begin{align*}
  \sum_{a}\occupancy(s,a) = \occstate(s).
\end{align*}

 Furthermore, given the independence of the transition distributions from policies, we have the following 
\begin{align}
   \big(
   \occupancy^{\neighbor{\targetpi}{s}{a}}-\occupancy^{\targetpi}
   \big)(\tilde s, \tilde a)
   =
   \begin{cases}
       \occstate(s)\quad
       \text{if} \quad
       (\tilde s, \tilde a)=(s,a)\\
       -\occstate(s)\quad
       \text{if} \quad
       (\tilde s, \tilde a) = 
       (s, \targetpi(s))\\
       0 \quad\text{o.w}
   \end{cases}.
   \label{eq.occ_diff_characterization_special_mdp}
\end{align}
Therefore, the constraint
$\vecdot{\occupancy^{\neighbor{\targetpi}{s}{a}} - \occupancy^{\targetpi}}{\occupancy} \ge 0$
is equivalent to
$\occupancy(s,a)\ge
\occupancy(s, \targetpi(s))
$.
Furthermore, note that
\begin{align*}
    (s,a)\in \tightset &\iff 
    \hatscore^{\neighbor{\targetpi}{s}{a}}-\hatscore^{\targetpi}=-\epsA \\&\iff
    \vecdot{\occupancy^{\neighbor{\targetpi}{s}{a}}-\occupancy^{\targetpi}}{\widehat{R}} \le -\epsA
    \\&\iff
    \widehat{R}(s,a)-\widehat{R}(s,\targetpi(s))=-\frac{\epsA}{\occstate(s)} 
    \\&\iff a \in \tightsetstate.
\end{align*}
Putting it all together,
the optimization problem
\eqref{prob.defense_optimization.known_epsilon} is equivalent to
\begin{align*}
    &\max_{\occupancy}
    \vecdot{\widehat R}{\occupancy}\\&
    \text{s.t.} \quad
    \occupancy(s,\targetpi(s)) \le \occupancy(s,a)
    \quad\forall s, a\in \tightsetstate\\&
    \sum_{a}\occupancy(s,a) = \occstate(s)
    \quad\forall s \in S
    \\&
    \occupancy(s,a)\ge 0\quad \forall (s,a).
\end{align*}
    
    Note that the maximization is now over all vectors $\occupancy\in \mathbb{R}^{|S|.|A|}$ as
    the constraint $\occupancy\in \Occupancy$ has been made explicit.
    Furthermore,
    given Lemma \ref{lm.syed}
    and
    Equation
    \eqref{eq.occupancy_measures_relation_a},
    any vector
    $\occupancy$
    satisfying
    the last two constraints (the bellman constraints)
    corresponds to a policy
    $\pi$ through
    \begin{align*}
        \pi(a|s) = \frac{\occupancy(s,a)}{\occstate(s)}.
    \end{align*}
    In other words, probability of choosing $a$ in state $s$ is proportional to $\occupancy(s,a)$.
    
    Now, let us analyze the solution to this optimization problem which we will denote by
    $\occupancy_{\max}$.
    This solution
    $\occupancy_{\max}$ exists,
    since
    the optimization problem
    is maximizing a continuous function
    on a closed and bounded set.
    
    We first claim that if
    $a \notin \tightsetstate \cup \{\targetpi(s)\}$,
    then
    $\occupancy_{\max}(s,a)=0$.
    If this is not the case, then $\occupancy_{\max}$ is not optimal.
    Concretely, consider the following vector
    $\occupancy$
    \begin{align*}
        \occupancy(\tilde s, \tilde a)=
        \begin{cases}
        \occupancy_{\max}(\tilde s, \tilde a) + \frac{1}{|\tightsetstate| + 1}\occupancy_{\max}(s, a)
        \quad\text{if}\quad \tilde s =s\land
        \tilde a \in \tightsetstate \cup\{\targetpi(s)\}\\
        0 \quad\text{if} \quad
        \tilde s = s \land \tilde a = a\\
          \occupancy_{\max}(\tilde s, \tilde a) \quad \text{o.w.}
        \end{cases}.
    \end{align*}
    In other words, we
    uniformly spread
    the probability
    of choosing
    action $a$
    in state $s$
    over
    the set
    $ \tightsetstate \cup 
    \{\targetpi(s)\}
    $.
    The vector $\occupancy$
    still satisfies the constraints:
    if
    $\tilde a\in \tightsetstate$,
    $\psi(s, \targetpi(s))-\psi(s, \tilde a)=
    \psi_{\max}(s, \targetpi(s))-\psi_{\max}(s, \tilde a)
    $
    and the objective has strictly
    improved because
 \begin{align*}
    \hatscore^{\neighbor{\targetpi}{s}{a}}-\hatscore^{\targetpi}\le -\widehat{\epsilon}\le -\epsilon \implies
    \widehat{R}(s, a)
    \le \widehat{R}(s, \targetpi(s))
    -\frac{\epsilon}{\occstate(s)}.
    \end{align*}
    Since $a\notin\tightsetstate$, 
    the inequality is strict and therefore
    \begin{align*}
        \forall \tilde a \in \tightsetstate \cup \{\targetpi(s)\}:
        \widehat{R}(s,\tilde a) > 
        \widehat{R}(s, a).
    \end{align*}
    This means that
    $\occupancy$
    was not optimal, contradicting
    the initial assumption.
    
    Now note that if
    $\occupancy_{\max}(s,a)>\occupancy_{\max}(s, \targetpi(s))$ for some
    $a\in \tightsetstate$, then again
    $\occupancy_{\max}$ isn't optimal as we could replace it with
    \begin{align*}
        \occupancy(\tilde s, \tilde a)=
        \begin{cases}
        \occupancy_{\max}(\tilde s, \tilde a) + \cfrac{\occupancy_{\max}(s, a)-
        \occupancy_{\max}(s, \targetpi(s))}{|\tightsetstate| + 1}
        \enskip\text{if}\enskip \tilde s =s\land
        \tilde a \in \tightsetstate \cup\{\targetpi(s)\}\backslash \{a\}\\
        \occupancy_{\max}(\tilde s, \tilde a) - \cfrac{|\tightsetstate|(\occupancy_{\max}(s, a)-
        \occupancy_{\max}(s, \targetpi(s)))}{|\tightsetstate| + 1}
        \enskip\text{if}\enskip
        \tilde s = s \land \tilde a = a\\
          \occupancy_{\max}(\tilde s, \tilde a) \quad \text{o.w.}
        \end{cases}.
    \end{align*}
    Intuitively, 
    since the
    action $a$
    was being chosen
    with strictly higher probability
    than action
    $\targetpi(s)$,
    we have uniformly spread this excess probability among the set $\tightsetstate \cup \{\targetpi(s)\}$.
    This vector would still be feasible as
    $\occupancy(s,a)=\occupancy(s, \targetpi(s))$ and would be strictly better in terms of utility as
    $\widehat{R}(s, \targetpi(s))> \widehat{R}(s, a)$.
    This contradicts our initial assumption and therefore
    $\occupancy_{\max}(s,a) 
    =\occupancy_{\max}(s, \targetpi(s))$
    for all $a\in \tightsetstate$.
    
    Since the occupancy measure $\occupancy_{\max}$
    satisfies $\occupancy_{\max}(s,a)=0$
    for all $a\notin \tightsetstate \cup \{\targetpi(s)\}$ and
    $\occupancy_{\max}(s,a)=\occupancy(s, \targetpi(s))$ for all $a\in \tightsetstate$, we conclude that
    it is the occupancy measure for the policy
    $\defensepi$ as defined in Equation
    \eqref{eq.lm.bandit.characterization.policy}.
    
    In order to prove $\barscore^{\defensepi}=\hatscore^{\defensepi}$, first note that for $(s,a)\in \tightset$ 
    \begin{align*}
        \vecdot{
        \occupancy^{\neighbor{\targetpi}{s}{a}}-
        \occupancy^{\targetpi}
        }{
        \occupancy^{\defensepi}
        }&=
        \occstate(s)(\occupancy^{\defensepi}(s, a) - \occupancy^{\defensepi}(s, \targetpi(s))
        \\&= 
        \occstate(s)^2(\defensepi(a|s) - \defensepi(\targetpi(s)|s)) = 0, 
    \end{align*}
    where we used Equation \eqref{eq.occ_diff_characterization_special_mdp} and Equation
    \eqref{eq.lm.bandit.characterization.policy}.
    Therefore
    \begin{align*}
        \barscore^{\defensepi}-
        \hatscore^{\defensepi}&=
        \vecdot{\overline{R}-\widehat{R}}{
        \occupancy^{\defensepi}
        }\\&\overset{(i)}{=}
        \sum_{(s,a)\in \tightset}
        \alpha_{s,a}
        \vecdot{
        \occupancy^{\neighbor{\targetpi}{s}{a}}-
        \occupancy^{\targetpi}
        }{
        \occupancy^{\defensepi}
        }\\&=
        \sum_{(s,a)\in \tightset}
        \alpha_{s,a}\cdot0
        \\&=0,
    \end{align*}
where $(i)$ follows from
Lemma \ref{lm.tightset}
in the known parameter case and
Lemma 
\ref{lm.tightset.unknown}
in the unknown parameter case.
\end{proof}

\subsection{Bounds on attack influence}\label{sec.appendix.sepcial.attack.influence}
 In this section, we provide  bounds on the attack influence of our defense in Special MDPs. We start with the bound we already stated in Appendix \ref{sec.defense_characterization_specific_mdp}, i.e., the bound for the known parameter case.   

 \begin{proposition}\label{lm.bandit.bound.known.general.state}
 (Bound for the known parameter case)
    Let $\defensepi$
    be the solution to the optimization problem 
    \eqref{prob.defense_a}. Then, the attack influence $\influence^{\defensepi}$ is bounded by
     \begin{align*}
         \influence^{\defensepi}\le 
         \max\{|S|\cdot \epsA, 
         \frac 12 \cdot\influence^{\targetpi}
         +
         \frac{2|S|-1}{2}\cdot\epsA
         \}.
     \end{align*}
 \end{proposition}
 \begin{proof}
   By instantiating the result of Theorem \ref{cor.attack_influence_a} with $\beta^{\mu}=0$, which holds in this setting, we obtain that the attack influence $\influence^{\defensepi}$ is bounded by
   \begin{align*}
       \influence^{\defensepi} \le \max \left \{\widehat \influence, \frac{1}{2}\cdot [\influence^{\targetpi} +  \epsA] + [\widehat \influence - \epsA]\right \}.
   \end{align*}
   Therefore, to prove the claim, it suffices to show that
\begin{align}
\widehat{\Delta}
\le |S| \cdot \epsA.
\label{eq.second.term.bound}
\end{align}
In order to see why Equation \eqref{eq.second.term.bound} holds, first note that 
\begin{align*}
\widehat{\Delta}
&=\hatscore^{\targetpi}
-\hatscore^{\defensepi}
\\&\overset{(i)}{=}
\sum_{s,a}\widehat{R}(s,a)
(
\occupancy^{\targetpi}(s,a)
-
\occupancy^{\defensepi}(s,a)
)
\\&\overset{(ii)}{=}
\sum_{s,a}
\widehat{R}(s,a)
\occstate(s)(\targetpi(a|s)-
\defensepi(a|s)),
\end{align*}
where
$(i)$
follows from 
\eqref{eq.score_occupancy_relation} and
$(ii)$ follows
from Equation
\eqref{eq.occupancy_measures_relation_a}
and the fact that
$\occstate$ is independent of policy.
By rearranging the last expression we obtain
\begin{align}
\widehat{\influence}=
    \sum_{s}\occstate(s)
\Big(
\sum_{a} \widehat{R}(s,a)(\targetpi(a|s)-\defensepi(a|s))
\Big).
\label{eq.bandit.defense.character.second.term}
\end{align}

We now analyze the term $\Big(
\sum_{a} \widehat{R}(s,a)(\targetpi(a|s)-\defensepi(a|s))
\Big)$
for an arbitrary $s$.
Note that
\begin{align*}
&\forall
a \in \tightsetstatetarget: \quad
\widehat{R}(s,a)
= 
\widehat{R}(s, \targetpi(s))
-\frac{\epsA}{\occstate(s)}
\\ \implies&
\forall a \in \tightsetstatetarget
\cup \{\targetpi(s)\}
: \quad
\widehat{R}(s,a)
\ge 
\widehat{R}(s, \targetpi(s))
-\frac{\epsA}{\occstate(s)}
.
\end{align*}
Therefore,
given Proposition \ref{lm.bandit.defense.character}, we have
\begin{align*}
\sum_{a} \widehat{R}(s,a)(\targetpi(a|s)-\defensepi(a|s))
&=
\sum_{a} \widehat{R}(s,a)\targetpi(a|s)-
\sum_{a} \widehat{R}(s,a)\defensepi(a|s)
\\&\le 
\widehat{R}(s, \targetpi(s))
-\left (\widehat{R}(s, \targetpi(s))
-\frac{\epsA}{\occstate(s)}
\right )
\\&=
\frac{\epsA}{\occstate(s)}.
\end{align*}
Finally, using Equation \eqref{eq.bandit.defense.character.second.term}, we obtain
\begin{align*}
\widehat{\influence}&=
    \sum_{s}\occstate(s)
\Big(
\sum_{a} \widehat{R}(s,a)(\targetpi(a|s)-\defensepi(a|s))
\Big)\\&\le
\sum_{s}\occstate(s)\left (\frac{\epsA}{\occstate(s)} \right )\le \sum_{s}\epsA=|S| \cdot \epsA,
\end{align*}
which proves the claim.
\end{proof}
 For the unknown parameter case, we obtain the following bound on the attack influence. 
 \begin{proposition}
 (Bound for the unknown parameter case)
    Let $\defensepi$
    be the solution to the optimization problem 
    \eqref{prob.defense_b} with $\epsD \ge \epsA$. Then, the attack influence $\influence^{\defensepi}$ is bounded by
     \begin{align*}
         \influence^{\defensepi}\le 
         \max\{|S|\cdot \epsD, 
         \frac 12 \cdot\influence^{\targetpi}
         +
         \frac{2|S|-1}{2}\cdot\epsA
         \}.
     \end{align*}
 \end{proposition}
 \begin{proof}
    
    Set
    $\epsilon=\min\{\widehat{\epsilon}, \epsD\}$
    We divide the proof into two cases
    
    \textbf{Case 1:}
    Assume that
    $\overline{R}=\widehat{R}$.
    Fix state $s$.
    Due to Proposition
    \ref{lm.bandit.defense.character},
    $\defensepi$ satisfies 
    Equation \eqref{eq.lm.bandit.characterization.policy}.
    Therefore for
    all actions $a$ such that
    $\defensepi(a|s) > 0$, 
    
    \begin{align*}
        \widehat{R}(s,a)
    \ge \widehat{R}(s, \targetpi(s))
    -\frac{\epsilon}{\occstate(s)},
    \end{align*}
    which implies
    \begin{align*}
        \influence^{\defensepi}=
        \hatscore^{\targetpi}-
        \hatscore^{\defensepi}
        &=\sum_{s}\occstate(s)(\widehat{R}(s, \targetpi(s)) - \sum_{a}
        \defensepi(a|s)\widehat{R}(s,a))
        \\&\le 
        \sum_{s}\occstate(s)\cdot
        \frac{\epsilon}{\occstate(s)}
        =\sum_{s}\epsilon 
        \\&\le \sum_{s}\epsD
        =|S|\cdot \epsD.
    \end{align*}
    
    \textbf{Case 2:}
    Assume that
    $\overline{R}\ne \widehat{R}$.
    From Corollary \ref{lm.epsilon.hat}, we conclude that
    $\widehat{\epsilon}=\epsA$, and therefore,
    $\min\{\widehat{\epsilon}, \epsD\}=\epsA$. 
    Due to Proposition
    \ref{lm.bandit.defense.character}, 
    this means that the optimization problems
    \eqref{prob.defense_a} and
    \eqref{prob.defense_b} are equivalent, and hence, 
    and the claim follows from Proposition 
    \ref{lm.bandit.bound.known.general.state} 
    since $\epsA \le \epsD$.
 \end{proof}
 

\subsection{Proof of Theorem \ref{thm.bandit_impossibility}}
\label{sec.appendix.thm.bandit_impossibility}
\textbf{Statement:}
 {\em
     Let $\delta > 0$ and $\epsA > 0$. 
    There exists a problem instance with poisoned reward function $\widehat{R}$ 
    and a target policy $\targetpi$ such that for all defense policies $\pi_{\widetilde{\mathcal{D}}} \in \Pi$ and constants $C$,
    we can find $\overline{R}$ that satisfies 
    $\widehat{R} = \attack(\overline{R}, \targetpi, \epsA)$ with the following lower bound on attack influence $\influence^{\pi_{\widetilde{\mathcal{D}}}}$:
     \begin{align*}
         \influence^{\pi_{\widetilde{\mathcal{D}}}} \ge
         \frac{1}{2+\delta}\cdot \influence^{\targetpi} + C.
     \end{align*}
 }
\begin{proof}
	Take $k$ to be an arbitrary positive integer such that
	$k\ge 2 + \frac 4\delta$.
We consider a single state MDP with state set $S=\{s_1\}$ and action set $A=\{a_1,...a_{k+1}\}$. Since
there is only one state, 
we abuse our notation and denote $R(s_1, a_i)$ by $R(a_i)$.
Moreover, we denote the deterministic policy that selects action $a_i$ with probability $1$ by $\pi_i$, and in general,
we use $\pi(a_i)$
to denote $\pi(a_i|s_1)$
since there is only one state.

Set
$\targetpi(s_1)=a_{k+1}$, and 
consider the following values for the
reward function $\widehat{R}$:
\begin{gather*}
	\widehat{R}(a_i) = \begin{cases}
		\epsA \quad\text{if}\quad i = k+1\\
		0 \quad\text{o.w}
	\end{cases}
\end{gather*}
Since the poisoned reward of the target action is at least $\epsA$ higher than any other action, the reward function
$\widehat{R}$ is plausible, i.e., it satisfies the constraints of the optimization problem \eqref{prob.attack}.

Let  $\pi_{\widetilde{\mathcal{D}}}$ be an arbitrary policy as stated in the theorem. Note that $\sum_{i=1}^{k}	\pi_{\widetilde{\mathcal{D}}}(a_{ i}) \le 1$, which implies
\begin{gather*}
\sum_{i=1}^{k}	\pi_{\widetilde{\mathcal{D}}}(a_{ i})  \le 1\implies
\exists l \in \{1,...k\}:
	\pi_{\widetilde{\mathcal{D}}}(a_{l}) \le \frac{1}{k},
\end{gather*}
i.e., there exists action $a_l \ne a_{k+1}$ whose selection probability under the defense strategy $\pi_{\widetilde{\mathcal{D}}}$ is smaller than or equal to $\frac{1}{k}$.
Now, set $\eta\ge \max\{
2\epsA,
\frac{4+2\delta}{\delta}\cdot C\}$ to be an arbitrary positive number, and
consider the following reward function $\overline R$: 
\begin{align*}
	\overline{R}(a_i) = \begin{cases}
		\eta \quad\text{if}\quad i=l\\
		-\eta +\epsA \quad\text{if} \quad i =k+1\\
		0 \quad o.w.
	\end{cases},
\end{align*}
for which the optimal policy is $\optpi = a_{l}$. We first need to verify that $\overline R$ is a plausible true reward function for poisoned $\widehat R$. Recall from 
\eqref{eq.tightset.definition} that
\begin{align*}
    \tightsettarget = \left \{ (s , a) : 
    a\ne \targetpi(s_i)
    \land
    \widehat \score^{\neighbor{\targetpi}{s}{a}} - \widehat \score^{\targetpi} =  -\epsA \right \}=
    \left \{
    (s_1, a_i): 1\le i \le k
    \right \},
\end{align*}
which implies that 
$(s_1, a_l)$
is in $\tightsettarget$. Using Equation \eqref{eq.occupancy_measures_relation_a}, we also know that the state-action occupancy measure of the policy that selects action $a_i$ with probability $1$, i.e., $\pi_i$, satisfies
\begin{align*}
    \occupancy^{\pi_i}(s_1, a_j)=
    \pi_i(a_j|s_1)\occstate(s_1)=
    \ind{i=j}.
\end{align*}
Also, note that $\overline{R} = \widehat{R} + \eta(\occupancy^{\pi_{l}} - \occupancy^{\targetpi})$. Therefore, from  Lemma \ref{lm.tightset} (by setting $\alpha_{s_1, a_l} = \eta$ and $\alpha_{s_1, \tilde{a}} = 0$ for all $(s_1, \tilde a \ne a_l) \in \tightsettarget$), it follows that
\begin{align*}
	\overline{R} = \widehat{R} + \eta(\occupancy^{\pi_{l}} - \occupancy^{\targetpi}) \implies
	\attack(\overline{R}, \targetpi, \epsA) = \widehat{R}.
\end{align*}
In other words, $\overline R$ is indeed a plausible true reward function for poisoned $\widehat R$.  

To establish the bound in the statement, let us investigate the attack influences of $\targetpi$ and $\pi_{\widetilde{\mathcal{D}}}$, i.e., $\influence^{\targetpi}$ and $\influence^{\pi_{\widetilde{\mathcal{D}}}}$, and compare them. 
The attack influence of $\targetpi$ is upper bounded by
\begin{align*}
	\influence^{\targetpi} =
	\barscore^{\optpi}
	-\barscore^{\targetpi}=
	2 \cdot \eta - \epsA\le 2\eta .
\end{align*}
Furthermore
\begin{align*}
	\barscore^{\pi_{\widetilde{\mathcal{D}}}} =\sum_{1\le i \le k+1}
	\pi_{\widetilde{\mathcal{D}}}(a_i) \cdot \overline{R}(a_i) = 
	\eta \cdot \pi_{\widetilde{\mathcal{D}}}(a_l) - \eta \cdot \pi_{\widetilde{\mathcal{D}}}(a_{k+1}) + \pi_{\widetilde{\mathcal{D}}}(a_{k+1}) \cdot \epsA \le \frac{1}{k} \cdot \eta,
\end{align*}
where the last inequality
follows from the fact that
$\eta \ge \epsA$
and
$\pi_{\widetilde{\mathcal{D}}}(a_l)\le \frac 1k$.
Therefore, the attack influence of $\pi_{\widetilde{\mathcal{D}}}$ is lower bounded by
\begin{align*}
	\influence^{\pi_{\widetilde{\mathcal{D}}}}
	=\barscore^{\optpi}
	-\barscore^{\pi_{\widetilde{\mathcal{D}}}}
	\ge
	(1-\frac 1k)\cdot \eta
	.
\end{align*}
We need to show that 
\begin{align*}
\influence^{\pi_{\widetilde{\mathcal{D}}}}
\ge \frac 1{2+\delta} \cdot 
\influence^{\targetpi} + C
.
\end{align*}
Given the bounds on the attack influences of $\targetpi$ and $\pi_{\widetilde{\mathcal{D}}}$, it actually suffices to show that 
\begin{align*}
    (1-\frac 1k)\cdot \eta\ge \frac{2\eta}{2+\delta} + C,
\end{align*}
since this would imply
\begin{align*}
\influence^{\pi_{\widetilde{\mathcal{D}}}}
	\ge
	(1-\frac 1k)\cdot \eta
	\ge
	\frac{2\eta}{2+\delta} + C\ge
	\frac 1{2+\delta} \cdot 
\influence^{\targetpi} + C
.
\end{align*}
We have that
\begin{align*}
	(1-\frac 1k)\cdot \eta\ge \frac{2\eta}{2+\delta} + C &\iff
	(2+\delta)\cdot \eta - \frac{2+\delta}{k}\cdot \eta 
	\ge 2\cdot \eta + C\cdot (2+\delta)
	\\&\iff 
	\eta\cdot (\delta - \frac{2+\delta}{k})\ge C\cdot (2+\delta)
	.
\end{align*}
Now, notice that we chose $k$ to be sufficiently large, i.e., $k\ge 2 + \frac 4\delta$. This implies
\begin{align*}
	k\ge 2 + \frac 4\delta=\frac{\delta + 2}{\frac{\delta}{2}} \implies 
	\frac \delta 2 \ge \frac{2+\delta}{k} \implies
	\delta - \frac{\delta  + 2}{k}\ge \frac \delta 2\implies
	\eta\cdot (\delta - \frac{2+\delta}{k})\ge
\eta \cdot\frac{\delta}{2}.
\end{align*}
Since we chose $\eta$ that satisfies $\eta\ge \frac{4+2\delta}{\delta}\cdot C$, we further obtain 
\begin{align*}
    \eta\cdot (\delta - \frac{2+\delta}{k})\ge
\eta \cdot\frac{\delta}{2} \ge \frac{4+2\delta}{\delta}\cdot C \cdot\frac{\delta}{2} = C\cdot (2+\delta). 
\end{align*}
Therefore
\begin{align*}
    (1-\frac 1k)\cdot \eta\ge \frac{2\eta}{2+\delta} + C
\end{align*}
holds, which proves the claim, i.e.:
\begin{align*}
\influence^{\pi_{\widetilde{\mathcal{D}}}}
\ge \frac 1{2+\delta} \cdot 
\influence^{\targetpi} + C
.
\end{align*}
\end{proof}

		}
	}
	{}
\end{document}